\documentclass[sigconf]{acmart} 
\AtBeginDocument{%
  }

\usepackage{booktabs}      
\usepackage{tabularx}      
\usepackage{xspace} 

\usepackage{graphicx}
\usepackage{textcomp}
\usepackage{xcolor}
\usepackage[T1]{fontenc}

\usepackage{enumerate}
\usepackage{adjustbox}
\usepackage{hyperref}
 \usepackage{url}
\usepackage[noend]{algpseudocode}
\usepackage{amsmath} 
\usepackage{amsthm}
\usepackage{tabularx}
\usepackage{tikz}
\usepackage{multirow}
\usepackage{threeparttable}
\usetikzlibrary{calc, positioning, arrows}
\usepackage[most]{tcolorbox}
\usepackage{subcaption}
\usepackage{array}

\usepackage{algorithmicx}
\usepackage[ruled,vlined]{algorithm2e}
\usepackage{float}
\usepackage{pgf-umlsd}
\usepackage{autobreak}

\makeatletter
\def\old@comma{,}
\catcode`\,=13
\def,{%
  \ifmmode%
    \old@comma\discretionary{}{}{}%
  \else%
    \old@comma%
  \fi%
}
\makeatother

\setcopyright{none}
\renewcommand\footnotetextcopyrightpermission[1]{}
\newcommand{\saga}{SAGA\xspace}
\newcommand{\sagaplus}{MAGIQ\xspace}
\newcommand{\provider}{\texttt{Provider}\xspace}

\newcommand{\masterw}{{\em orchestrator}\xspace}
\newcommand{\tmsg}{{\em task-msg}\xspace}
\newcommand{\Tmsg}{{\em Task-msg}\xspace}

\newcommand{\remove}[1]{}

\newcommand{\ases}{{$A$-session}}
\newcommand{\cses}{{$C$-session}}
\newcommand{\caler}{{$A_I$}}
\newcommand{\caled}{{$A_R$}}
\newcommand{\master}{{$A_I^*$}}

\newcommand{\Nmaster}
{$Q^*_{tot}$}
\newcommand{\Tmaster}{$\Delta^*_{tot}$}

\newcommand{\Fases}{$\mathcal{F}_{\mathsf{Ases}}$} 
\newcommand{\Fcses}{$\mathcal{F}_{\mathsf{Cses}}$} 

\begin{document}


\title{\sagaplus: A Post-Quantum Multi-Agentic AI Governance System with Provable Security} 




\author{Sepideh Avizheh}
\affiliation{%
  \institution{University of Calgary}
  \city{Calgary}
  \state{Alberta}
  \country{Canada}
}

\author{Tushin Mallick}
\affiliation{%
  \institution{Northeastern University}
  \city{Boston}
  \state{MA}
  \country{USA}
}

\author{Alina Oprea}
\affiliation{%
  \institution{Northeastern University}
  \city{Boston}
  \state{MA}
  \country{USA}
}

\author{Cristina Nita-Rotaru}
\affiliation{%
  \institution{Northeastern University}
  \city{Boston}
  \state{MA}
  \country{USA}
}

\author{Reihaneh Safavi -Naini}
\affiliation{%
  \institution{University of Calgary}
  \city{Calgary}
  \state{Alberta}
  \country{Canada}
}

\renewcommand{\shortauthors}{}


\begin{abstract} 
Our computing ecosystem is being transformed by two emerging paradigms, the increased deployment of agentic AI systems and the advancements in quantum computing. With respect to agentic AI systems, one of the most critical problems is creating secure governing architectures that ensure agents follow their owner’s communication and interaction policies, and can be held accountable for the messages they exchange with other agents. With respect to quantum computing, existing systems must be retrofitted and new cryptographic-based mechanisms must be designed to ensure long-term security and quantum-resistance. In fact, NIST recommends that standard public-key cryptographic algorithms, including RSA, Diffie–Hellman (DH), and elliptic-curve constructions (ECC), be deprecated starting in 2030 and disallowed after 2035.  

In this paper we create \sagaplus, a framework for policy definition and enforcement of multi-agentic AI systems using novel highly efficient quantum-resistant cryptographic protocols with proved security guarantees. \sagaplus (i) allows users to define rich communication and access control policy budgets for agent-to-agent sessions and tasks,
with global budgets for one-to-many agent sessions; (ii) enforces such policies using post-quantum cryptographic primitives;
 (iii) supports session-based enforcement of policies for agent-to-agent and one-to-many agent sessions; and (iv) provides accountability of agents to their users in the form of message attribution. We formally model and prove correctness and security of  the system using the Universal Composability (UC) framework. We evaluate the computation and communication overhead of our framework and compare it with the state of the art agentic AI framework \saga. \sagaplus is the first step towards post-quantum secure solutions 
for agentic AI systems.

\end{abstract}

\maketitle

\section{Introduction}

Our society is being transformed by the emergence of agentic AI systems. Agents are
now deployed across domains such as finance, healthcare, customer support, and software engineering,  to autonomously plan and execute complex, multi-step tasks with minimal human intervention.
One of the most critical problems in agentic AI deployment \cite{shavit2023practices} is creating secure governing architectures that ensure defining
unique identification  of AI agents, authenticating the agents in their interactions,  ensuring agents follow their owner’s communication and interaction policies, and can be held accountable for
their interactions with other agents.

Previous work on agentic governance has focused on defining 
requirements for agent identities~\cite{chan_ids_2024}, 
capability taxonomies~\cite{agntcy_agentdirectory_2025}, 
introducing attribution~\cite{chan_infrastructure_2025}, authorization with delegation capabilities~\cite{south_authenticated_2025}, or indexing~\cite{nanda_index_2025}. These proposals remain theoretical without implementation, empirical evaluation, or provable guarantees.
Protocols like Agent2Agent (A2A) ~\cite{a2a_2025} and Model Context Protocol~(MCP)~\cite{mcp_registry_2025} promote agent interoperability
but neither provides policy enforcement mechanisms or runtime mediation of agent interactions, 
leaving them vulnerable to documented attacks~\cite{a2a_attacks_2025,survey_mcpsattack_2025}. The only 
practical system focused on ensuring agents follow policies created by their users is \saga~\cite{saga_ndss2026}.  \saga enforces access control policies based on an access control token generated by the agent allowing access for the agent requesting access, with the help of a centralized \provider managing a registry of users and agents. \saga's policies 
are for agent-to-agent interactions, and consist of user-defined authorized communication budgets of incoming requests per agent.
While \saga  provides formal (ProVerif) proofs for the basic security properties of the communicated messages, it does not provide a security model for policy definition and enforcement. In addition, 
\saga protocols use classical public key cryptographic primitives.

Recent advancements in quantum computing coupled with the Shor's~\cite{shor_quantum_1994} and Grover's~\cite{Grover1996} algorithms
raise serious concerns about the use of public key cryptographic primitives in systems and weakening of symmetric-primitives. 
Draft NIST IR 8547~\cite{moody2024nistir8547} (Transition to Post-Quantum Cryptography Standards) recommends deprecating quantum-vulnerable algorithms (RSA, ECDSA, EdDSA, DH, ECDH) by 2030 and fully retiring them by 2035.
To help with the transition, NIST released several PQC standards---FIPS~203~\cite{NIST-pqc-standard-fips203}, FIPS~204~\cite{NIST-pqc-standard-fips204}, and FIPS~205~\cite{NIST-pqc-standard-fips205} in August 2024. Two more drafts are under development for a digital signature based on FALCON~\cite{fouque2020falcon}, and a key encapsulation algorithm based on 
the Hamming Quasi-Cyclic (HQC), both expected to have drafts in 2026.

Given the 
growing deployment and reliance 
on agentic AI systems, and rapid development of quantum technologies  driven by global investment in the field, 
it is essential
to design post-quantum secure governance systems that remain secure in the presence of an adversary with access to a quantum computer. 

One solution to creating a post-quantum secure governance system is to replace \saga's underlying cryptographic primitives with post-quantum counterparts. However, this approach would incur significant communication and computation overhead and, more importantly, would neither address the limitations in policy definition nor provide formal security guarantees for policy enforcement.

\textbf{Our work.}
In this paper we introduce 
\sagaplus, a framework for policy definition and enforcement in multi-agentic AI systems using novel highly efficient quantum-resistant cryptographic protocols with proved security guarantees. 
The goal of \sagaplus is to define and enforce policies that are meaningful with respect to agent-semantics, i.e. tasks they need to solve. 
While agents can communicate securely over PQ-TLS, these transport level messages
are decoupled from task semantics and too low-level to capture agent-level interactions. 

We introduce
two application-level communication abstractions which we refer to as \emph{A-session} and \emph{C-session}. An \ases \ captures 
one-to-one communication between two agents, whereas a \cses\ captures
one-to-many 
communication between a coordinating agent and a set of other agents. We define the unit of communication  on these sessions as \tmsg. A \tmsg consists of a structured message that combines intent, data (payload), context, and agent identity, so the receiving agent can understand, verify, and act on the request.  
These two types of sessions are used to define and enforce policies.

MAGIQ cryptographically enforces four types of policies for an agent: 
\textcircled{1} the access control list with respect to agents receiving and initiating contact,
\textcircled{2} the (maximum) number of {\em A-sessions} that an agent can initiate, or respond to, 
\textcircled{3} the (maximum) number of \tmsg an agent can send in an {\em A-session}, taking the role of initiating or receiving party, and 
\textcircled{4} the duration of time an agent can be active for a task. These policies prevent access from unauthorized agents, mitigate attacks from malicious agents, and provide protections against different forms of denial-of-service attacks.
In addition, because each \tmsg is individually counted and cryptographically linked to a one-time signature/token, our policies provide attribution with respect to each interaction between agents. 

Cryptographic policy enforcement in an \ases\ relies on two mechanisms.
A digitally signed session-authorization token 
enforces the user-defined policy on authorized contactable agents and the allowed number of \ases. Two hash chains, one constructed by the 
agent that initiates the session, \caler\, and signed by its owner, 
and one constructed by the agent that responds to the request,
\caled, enforce the agents' respective \tmsg\ policies in the \ases. 
Policy enforcement for a \cses \ relies on the policy enforcement of its \ases, together with an efficiently computed membership token  for the session.
To summarize our contributions:
\begin{itemize}
\item We propose \sagaplus, a system that (i) allows users to define rich communication and access control policy budgets, (ii) enforces such policies using cryptographic primitives that are post-quantum resistant, 
and (iii) provides accountability of agents in the form of \tmsg attribution to their users.
\item We provide a UC-based formal model for policy-defined and policy-enforced secure agent communication, and prove that the \sagaplus\ protocols securely realize it.
\item We evaluate the overhead of \sagaplus and show that it is practical when compared to the state-of-the art  solution \saga that uses classical cryptography. 
\end{itemize}

The rest of the paper is organized as follows. Section \ref{sec:bk} provides background on governance for agentic AI systems and post-quantum primitives. Section \ref{sec:attackmodel} describes the attacker model we assume. Section \ref{sec:sagaplus} describes the design of the system and Section \ref{sec:model} its security. Section \ref{sec:results} presents evaluation results. Section \ref{sec:relwork} presents related work and Section \ref{sec:conclusion} concludes our work.

\section{Background and Problem Statement}
\label{sec:bk}

\subsection{Governance for AI Agentic Systems}
\begin{figure} [t]
    \centering
    \includegraphics[width=\columnwidth]{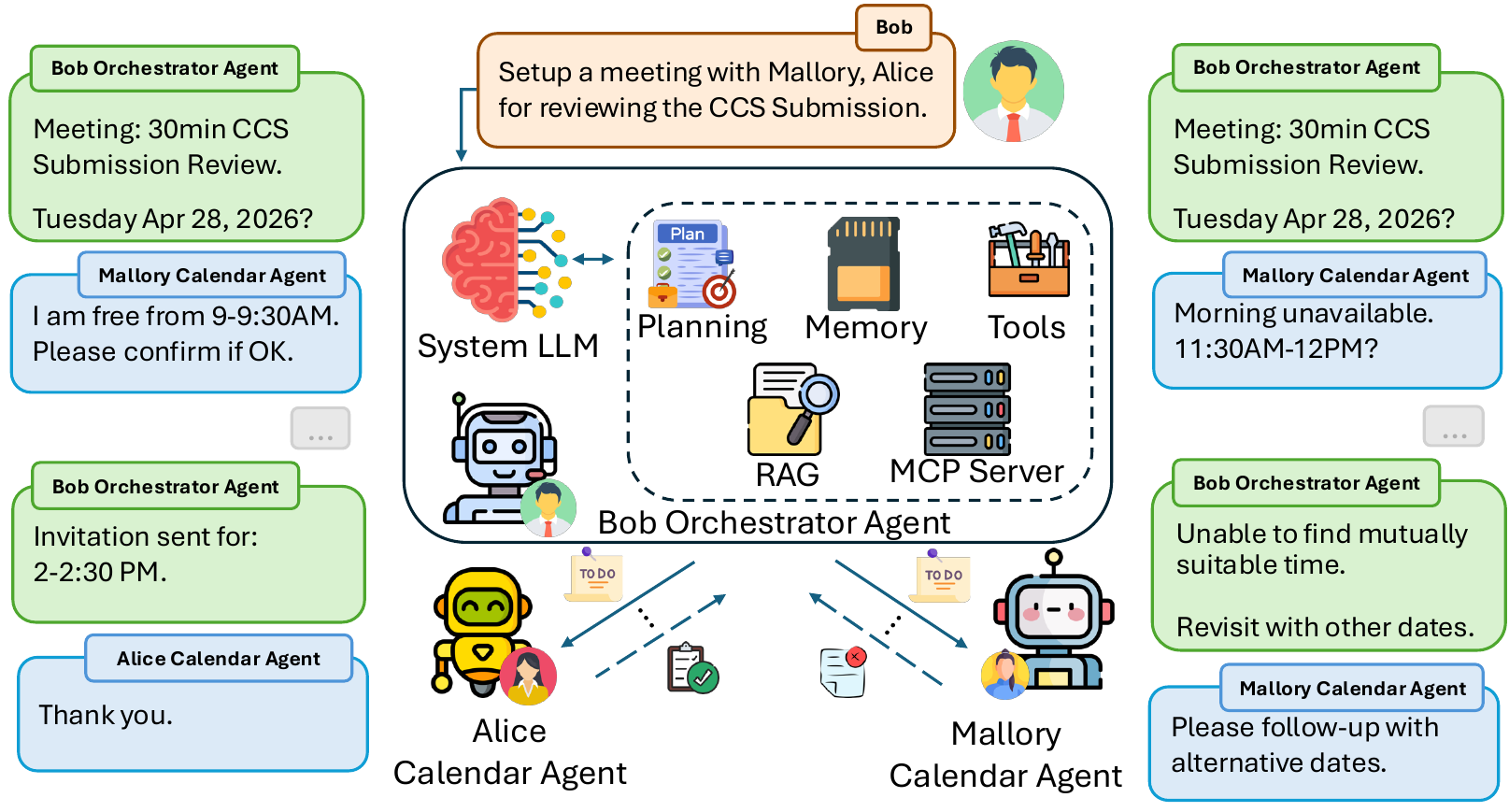}
    \caption{Example of a multi-agent coordination.}
    \label{fig:multi_agent_interaction}
    \Description{To avoid warning.}
\end{figure}

An agentic AI system is a system composed of one or more autonomous agents that perceive, reason, and act—individually or collaboratively—to achieve tasks with minimal human intervention (see Figure \ref{fig:multi_agent_interaction}). 
To accomplish their tasks, agents leverage multiple capabilities: tools,  external information accessed through web browsing, and coordination with other AI agents. Multi-agent frameworks (AutoGen~\cite{autogen}, MetaGPT~\cite{metagpt}) provide support for  multiple agents to coordinate  with the help of an orchestrator that determines what capabilities are needed for a task and then partitions the task for each individual agent that has the corresponding capability.

A critical aspect of an agentic AI system is how agents discover each other. This requires identities for agents and some form of agent registry. Such registry and discovery can be centralized, like in \saga\cite{saga_ndss2026}, or distributed like in AGNTCY \cite{agntcy_agentdirectory_2025}. The enhanced functionality of AI agents broadens the attack surface of agentic systems, introducing multiple security risks such as the use of malicious tools, access to untrusted data, and interaction with adversarial agents. While solutions against malicious tools and untrusted data have been proposed~\cite{ace_ndss2026,CaMeL}, very few solutions provide resilience against malicious agents~\cite{saga_ndss2026}. 


One of the primary concerns with autonomous agents is that they may go rogue and act against their users' interests. Several works \cite{shavit2023practices} therefore argue for continuous user control and oversight throughout the agent lifecycle. This implies that verifiable agent identities~\cite{chan_ids_2024} alone are not sufficient — these identities must be cryptographically bound to their users, and users must retain  authority to control their agents' behavior through flexible policies.

Another important question is what should the users control? Users should control who can find and access their agents (as in \saga), to mitigate certain forms of denial of service and limit the number of interactions or duration of interaction between agents. This provides protection against abuses and support for agent revocation if needed, but is not sufficient. Agent interactions should be attributable to the agents and the users who own them. Consider the scenario in which a malicious agent asks a benign agent to execute an action that is part of  writing some attack code. The request from the malicious agent must be attributable to it, such that the benign agent proves its innocence and the malicious interaction (task request) is bound to the malicious user and its agent that can be legally responsible for the action.  Secure communication protocols like TLS are not sufficient to provide these services as they focus on transport-level semantics and not task-level semantics. Thus, there is a need for policies that can capture agent-level semantics and that are under the control of the user owning the agents.


\subsection{Post-Quantum Primitives}

Post quantum primitives are cryptographic constructions that maintain their security even in the presence of an attacker who has access to a quantum computer. Below we describe the post-quantum primitives we use in this work.

\noindent\textbf{Cryptographic hash functions}. 
A hash function is a deterministic map $H: \{0,1\}^*\rightarrow \{0,1\}^{\ell(\lambda)}$,
where the domain consists of bitstrings of \emph{arbitrary length}, and the range is the set of $\ell(\lambda)$ bitstrings whose $\lambda $ is the security parameter. In our construction, depending on the component, the hash function must satisfy one of the standard properties, {\em pre-image resistance}, {\em second pre-image resistance,} or {\em collision resistance}. An example is SHA-256. 

\noindent\textbf{Hash chain.} 
Hash chains were first used in secure micropayment systems~\cite{rivest1996payword}.
A hash chain is constructed by iteratively applying a cryptographic hash
function $H$ to an initial value $s_0$, obtaining $
    s_i = H(s_{i-1}) \quad \text{for } 1 \leq i \leq n .
$
The terminal value $s_n$ serves as a commitment to the chain. A payment unit is
spent by revealing a preimage of the current committed value.

\noindent\textbf{HMAC.} HMAC is a {\em message authentication code} that relies on a keyed hash function and provides message integrity and authenticity. HMAC is constructed as: $H(k\oplus opad, (H(k \oplus ipad), m))$, where $m$ is the message, $k$ is the key, and opad and ipad are outer padding and inner paddings used to construct two different keys from $k$ for the outer and inner hash, that have a large hamming distance from each other. HMAC-SHA256 uses SHA-256 as the hash function.

\noindent\textbf{Pseudorandom function (PRF).}  A PRF is a family of efficiently computable functions
$F : K \times D \to R$, where $K=\{0,1\}^{\lambda}$ and
$D,R$ are efficiently samplable domains depending on the security parameter
$\lambda$, such that $F_k(\cdot)$ for random $k \leftarrow K$ is
computationally indistinguishable from a random function from $D$ to $R$.

\noindent\textbf{Merkle tree.} 
A Merkle tree is a binary tree built using a \textit{collision-resistant}
hash function over leaves $D=\{s_1,\ldots,s_n\}$. Each internal node is the
hash of the concatenation of its two children, and the root is denoted by
$Mroot$.
{\em To prove that $s_i$ is a leaf of the tree }rooted at $Mroot$, a
\textit{Merkle proof} is given, consisting of the sibling hash values along
the path from $H(s_i)$ to $Mroot$. The proof size is logarithmic in the
number of leaves.

\noindent\textbf{Post-quantum signatures (PQ signature).} PQ-signature schemes remain secure against adversaries with quantum computing capabilities. Hash-based signatures are PQ signatures whose security relies only on properties of hash functions.  

A hash-based signature scheme is a tuple
$HS=(HS.Gen,HS.Sign,HS.Verify)$, where
$HS.Gen(1^\lambda)$ outputs $(sk,pk)$,
$HS.Sign_{sk}(m)$ outputs a signature $\sigma$ for
$m\in\mathcal{M}$, and
$HS.Verify_{pk}(m,\sigma)$ outputs $1$ iff $\sigma$ is valid on $m$
under $pk$. Security is defined via existential unforgeability under
chosen-message attacks (EUF-CMA), as in classical signature schemes.

XMSS~\cite{buchmann2011xmss} is a \emph{stateful hash-based signature}
scheme in which the signer tracks used one-time signing keys and ensures each key is used at most once. XMSS is secure if the underlying hash functions satisfies standard properties such as second-preimage resistance and pseudorandomness, and is standardized in RFC~8391 \cite{rfc8391}.

\noindent\textbf{PQ-TLS.}  
TLS is a standardized secure communication
protocol that allows applications to communicate over a network while providing
authentication, confidentiality, and integrity for transmitted messages
\cite{gajek2008universally}. PQ-TLS replaces or augments
the core public-key primitives used in TLS, such as key exchange mechanisms and
signature schemes, with post-quantum secure ones, e.g., ML-KEM for key
establishment and ML-DSA for authentication, to provide security against
adversaries with quantum capabilities~\cite{ietf-draft-uta-pqc-app-00}. 

\subsection{Problem Statement}
We consider an agentic AI system consisting of a \textit{trusted~\provider } (P) and \textit{Users} ($U$) who own \textit{Agents} ($A$) and instruct them to solve tasks that  require interaction with other agents. 
Agents are assumed to have trusted direct access to an LLM engine.
We assume that the agents can collaborate in one-to-one and 
one-to-many settings to solve a task. We refer to the agent 
initiating the communication as \caler, and to the agent responding, as \caled. In the case of the one-to-many  interactions, we refer to the initiating agent as an \masterw. 

Our goal is to design a system that provides verifiable agent identities bound to their users, enforces user-defined  policies for agent control, and ensures attributability of agent interactions. For an agent $A$, a policy specifies which agents it may contact for service, which agents may request service from it, the allowed number of such contacts, and \emph{budgets} governing the communication.

Specifically, a user’s agent $A$ receives a task from its owner together with an interaction policy and associated budgets. Acting as \caler, the agent communicates with other agents to perform the task. The system guarantees that such communication takes place only within authorized, mutually authenticated interactions that respect the assigned budgets of the participating agents. The exchanged messages are uniquely attributable to the sender. 

Identity verification, policy enforcement, and attributability are achieved using cryptographic mechanisms.  
As the deadline to retire traditional cryptographic primitives is approaching, our solution takes into account attackers with quantum capabilities. All security guarantees are required to be formally provable: all guarantees in our framework are established cryptographically, except those involving time, which rely on an external trusted network clock.

\section{Assumptions and Attacker Model}
\label{sec:attackmodel}
In this section we describe our assumptions and attacker model.

\subsection{Assumptions} 
\textbf{Cryptographic assumptions.}
Entities have public and private keys that are used for authentication, and \tmsg attribution. Users and the \provider have their public key certified by a trusted {\em certificate authority (CA),} while agents' public keys are provided and  certified by their owner.
All cryptographic primitives provably satisfy their stated security properties. Entities communicate using mutually authenticated channels that are implemented using post-quantum secure TLS (PQ-TLS). The CA correctly authenticates entities (including distinguishing humans from bots) and binds public keys to them. Honest entities protect their secret keys and local state. Entities have secure randomness when required. Entities access a trusted, sufficiently synchronized network clock; time-based enforcement relies on it. User-side policy checks and
signing operations are performed by user-authorized software, protected by
trusted hardware.

\noindent\textbf{Trust assumptions.}
$CA$ and {\em Users} are trusted. \provider\ is {\em honest-but-curious} and remains available for the protocol steps assigned to it.  It follows the protocol  but can infer information about users
and agents' interactions. While we minimize this information when possible, we do no explicitly model and analyze the leakage to the \provider. {\em Agents} may be corrupted before communicating with other agents, in which case they may behave arbitrarily.

\subsection{Attacker Model} 
We consider that all processing and communication in the system are classical. The adversary, however, is quantum polynomial-time (QPT) and has access to a quantum computer. Its goal is to violate the security guarantees of the system. Specifically, we  consider an  adversary with the following capabilities:
\begin{itemize}
\item 
It fully controls the network: it can eavesdrop on, intercept, replay, delay, drop, and inject messages, except as prevented by the security of the employed cryptographic primitives.
\item 
It is {\em monolithic}: it coordinates all corrupted agents, which may share secret keys and internal states.
\item
It may corrupt agents statically. 
Upon corruption, it obtains the corrupted agent's secret keys and internal state, and corrupted agents may thereafter deviate arbitrarily from the protocol. 
\item
It may clone or replicate a corrupted agent on the same or on different devices. Replicated instances may share the secret keys and internal state of the original agent. If such replicas can appear as distinct identities, this enables Sybil-style attacks.
\end{itemize}

Network-layer attacks such as DoS are out of scope.\\

\section{\sagaplus~Design}
\label{sec:sagaplus}

In this section we describe \sagaplus. First we provide an overview of the system then we describe the system in details.

\subsection{Overview} 

\sagaplus~is an agentic AI governance system which allows agents to collaborate with each other in a one-to-one or one-to-many  settings to  perform a task. Users define agents’ policies in terms of budgets that govern different aspects of their interactions with other agents. 
There are three key aspects of \sagaplus's system design: \textcircled{1} the communication abstractions, \textcircled{2} the policies that users define over these abstractions, and \textcircled{3} the cryptographic mechanisms used to enforce them. 
An underlying
design principle, including the choice of cryptographic primitives, is
efficiency, so that security does not become a bottleneck. 

\textbf{(1) Communication abstractions.}
The goal of \sagaplus is to define and enforce rich policies that reflect agent-level semantics, capturing constraints on tasks that they are
expected to perform. While agents communicate over PQ-TLS, these transport level messages are too low-level and decoupled from task semantics. We introduce
two application-level communication abstractions which we refer to as \emph{A-session} and \emph{C-session}. An \ases \ models 
one-to-one communication between two agents, whereas a \cses\ models
one-to-many communication between an \masterw 
agent and a set of other agents. These two types of sessions are used to define and enforce policies. We define the unit of communication  on these sessions as \tmsg to distinguish it from transport-level messages. A \tmsg consists of a structured message that combines intent, data (payload), context, and header information, including cryptographic constructs, enabling  
the receiving agent 
to understand, verify, and act on the request. An \emph{A-session} may last over one or more PQ-TLS sessions. A \emph{C-session} is a {\em composed session}, consisting of multiple \ases\ composed sequentially or concurrently, and treated as a single policy-enforced unit.
%


{\bf (2) Policies.} 
The number of \tmsg\ and the session duration are natural user-defined control
parameters: they bound the cost incurred by agents and provide a measure of
protection against denial-of-service attacks, both for initiating and responding
agents. We therefore allow users to define policies over these parameters. 
For one-to-many setting, the policy controls interaction and duration across multiple orchestrated sessions. Specifically, 
users define policies that govern {\em the number of} \tmsg sent by each agent (\tmsg budget), and the {\em time duration of their activation } in the session  (time budget), for an {\em A-session} or {\em C-session}. 
An \ases\ that is initiated by an \caler\ may terminate due to exhaustion of the budgets of either of the two agents, or explicit termination by the \caler. In an one-to-many setting, the {\em \masterw} initiates a
one-to-many \cses \ that has an overall associated \tmsg and time budget.

\sagaplus enforces four types of policies for an agent $A$:
\begin{enumerate}
\item \emph{Agent Authorized Contact List ($ACL_A$)}
consists of two lists: an \emph{Authorized \caled\  list}, determining the  agents that can be contacted by $A$, and an \emph{Authorized \caler\ list}, which identifies the agents for which $A$ can be a responder.

\item  {\em \ases \ Participation Budget} specifies the (maximum) number of \ases\ that an agent can initiate, or respond to. 

\item {\em \tmsg Budget of an agent $A$, denoted by $Q_A$, } determines the (maximum) number of \tmsg that an agent can send in an \ases, taking  the role of \caler\ or \caled.

\item {\em Time budget}. An agent $A$ is activated by the user or a \tmsg.
The time budget $\Delta_A$ of an agent $A$ 
specifies the time duration that $A$ can remain activated. The budget is enforced by (uncorrupted) agents.
\end{enumerate}
 
  {\em Effective policy of an \ases: }
 Both  \caler\ and \caled \ in an \ases \ can be corrupted.  
 An \ases\ policy is defined and enforced {\em if at least one of the two agents is honest.} There are two cases:
 \begin{itemize}
\item {\em \caler \ and \caled\ are honest}:  the effective message budget for the \ases \ is  $Q_{I,R}= min \{Q_I, Q_R\}$, where $Q_I$ and $Q_R$ denote message budgets of $A_I$ and $A_R$, respectively. The {\em effective time budget of the session} is $\Delta_{I,R}=min\{\Delta_I, \Delta_R\}$, where $\Delta_I$ and $\Delta_R$ are the time budget of initiator and the responding agent respectively.

\item {\em \caler \ or \caled\ is honest:} the enforced message budget is $Q_{I,R}= min \{Q_I, Q_R\}$.  The effective time budget of the \ases\ is upperbounded by the time budget of the honest agent, that is, $\Delta_{I,R}=\Delta_h$, where $\Delta_h$ is the time budget of the honest agent $h \in \{I,R\}$.
 \end{itemize}
  The enforced policy of a \cses\ is obtained by aggregating the policies of its constituent \ases \, subject to the restriction that the total message budget across all component \ases \ is bounded by the message budget of the \cses, and the time budgets of all component \ases \ lie within the time budget of the \cses.

%

{\bf (3) Policy enforcement.} 
The $ACL_A$ policy is enforced by the \provider during agent discovery, i.e. when an agent contacts the \provider to obtain information and credentials about another agent. The time budget is enforced by the receiving agent by keeping a clock and comparing current time with timestamps attached to the communication. The number of sessions is enforced by counting the ongoing sessions. 
Finally, \tmsg policies of the \caler\ and the \caled\ are enforced using digitally signed \emph{hash chains} \cite{rivest1996payword}, which enable secure cryptographic counting of \tmsg within an \ases.
Session transcript integrity, ensuring the integrity of the messages and their ordering in an \ases\ that may span multiple PQ-TLS sessions, is achieved using HMACs that effectively chain the communicated messages.

{\em Using hash chains for efficient \Tmsg counting.}  
Inspired by the 
use of hash chains for secure micro-payment systems \cite{rivest1996payword},
we use hash chains to enforce \tmsg \ 
count budgets, as outlined below. Starting from a random seed $s_0$, an agent  computes a chain $s_i = H(s_{i-1})$ for $1 \leq i \leq n$, and publishes the terminal value $s_n$ as a commitment to the chain ( the terminal value will be signed by the  user software). Each message consumes one chain element, by revealing the pre-image of the current committed value. 
Correctness is verified by checking that applying $H$ to the revealed value yields the previously committed chain value.
Each agent in an \ases\ constructs a hash chain corresponding to its own \tmsg budget, allowing the other agent to verify cryptographically that it remains within that budget.
There are subtleties in using hash chains for secure \tmsg counting: a hash chain must be authorized by the user and bound to a single \ases, so that a corrupted \caler\ cannot reuse the chain in other \ases. 
(See Appendix \ref{Appendix:policyEnforce} for details of specializing a hash chain to a specific \ases.)

{\em C-session policy enforcement.} 
A \cses\ is initiated by an \master\ who receives a task, with a total message budget \Nmaster\ and time budget \Tmaster\ for \cses.
The budgets can be distributed among multiple \caled\ in their associated \ases.
For simplicity, we focus on {\em static work plans} where the \master\ plans the task and identifies the set of \caled\ that must be contacted before starting the session.
The construction, however, can be extended to the dynamic case, where the next \caled \ agent is  determined based on the previous \ases\ (See Appendix \ref{Appendix:policyEnforce} for details).
\master\ initiates an \ases\ with each planned \caled, assigning to each a \tmsg budget and a time budget, such that the total \tmsg budget and the time budgets across all \ases \ remain within the specified bounds, the total message budget \Nmaster \ and total time budget \Tmaster, respectively. 

A direct approach to message counting token generation for \cses~is to divide the \Nmaster \ among the $m$ planned sessions and execute the \ases\ token generation for each. This  requires $m$ signature computation. 
We use a Merkle tree to aggregate the terminal values of  $m$ hash-chains $s_n^1, s_n^2\cdots s_n^m$, into a single value, $Mroot$, that will be  appended by the chain length and digitally signed by the user. Here the hash chain length $n$, satisfies \Nmaster$=m\times n$. 
The initial token of the $j$th \ases \ (with $A_j$ as the responder) of the \cses \ will include this aggregate token, the terminal value of the j$th$ hash-chain and a {\em Merkle proof} that the terminal value is an aggregated leaf in the Merkle tree, whose root is signed.

\noindent\textbf{\sagaplus protocols.}
\sagaplus consists of the following protocols: user and agent registration, and inter-agent communication which in turn, consists of agent discovery and \ases \ (or \cses) protocols. Below we describe each of the protocols.
(For notations see Table \ref{tab:notation} in Appendix \ref{App:notations}.)

\subsection{User Registration}
We illustrate user registration in Figure~\ref{fig:UserReg}. Users first register with the \provider and subsequently
register their agents.
Each user generates a public–private key pair for a hash-based signature
scheme and obtains a certificate for the public key from a Certificate
Authority (CA). We assume that the \provider verifies the user’s real-world
identity using an external identity service, such as OpenID Connect.

\begin{enumerate}

    \item \textbf{User account setup.} The user selects a public identifier 
    $\mathit{uid}_U$ (e.g., \texttt{alice@domain.com}) and a secret password 
    $\mathit{Pwd}$ for authentication with the provider.

    \item \textbf{User key generation.} The user generates two key pairs: an 
    identity key pair $(sk_U, pk_U)$ used to sign agent information, and a 
    PQ-TLS key pair $(sk^{\mathit{tls}}_U, pk^{\mathit{tls}}_U)$. The user 
    then obtains from the $\mathcal{CA}$ the certificates:
    \[
        \mathit{Cert}^{\mathit{ID}}_U = \mathsf{GenCert}_{sk_{\mathcal{CA}}}(\mathit{uid}_U, pk_U)
    \]
    \[
        \mathit{Cert}^{\mathit{tls}}_U = \mathsf{GenCert}_{sk_{\mathcal{CA}}}(\mathit{uid}_U, pk^{\mathit{tls}}_U)
    \]

    \item \textbf{Connection establishment.} The user retrieves from 
    $\mathcal{CA}$ the provider's public keys $(pk^{\mathit{tls}}_{\mathit{TA}},\, 
    pk_{\mathit{TA}})$ along with their certificates 
    $(\mathit{Cert}^{\mathit{ID}}_{\mathit{TA}},\, 
    \mathit{Cert}^{\mathit{tls}}_{\mathit{TA}})$, and verifies them. 
    Here $pk^{\mathit{tls}}_{\mathit{TA}}$ is the provider's PQ-TLS key and 
    $pk_{\mathit{TA}}$ is its identity key used to sign user and agent 
    information. A PQ-TLS session is then established between the user and 
    provider using $pk^{\mathit{tls}}_{\mathit{TA}}$ and 
    $pk^{\mathit{tls}}_U$.

    \item \textbf{Sending user information.} Over the established session, the 
    user submits $(\mathit{uid}_U, \mathit{Pwd})$ and 
    $\mathit{Cert}^{\mathit{ID}}_U$ to the provider.

    \item \textbf{User identity verification.} The provider verifies the 
    user's identity via an external service $\mathcal{S}$ (e.g., OpenID 
    Connect). Upon successful verification, if no account exists for 
    $\mathit{uid}_U$, the provider finalizes the registration.

    \item \textbf{Account storage and confirmation.} The provider records the 
    user's entry in its registry:
    \[
        D_U[\mathit{uid}_U] \leftarrow \langle\, H(\mathit{Pwd}),\ 
        \mathit{Cert}^{\mathit{ID}}_U\, \rangle
    \]
    and sends a confirmation to the user.

\end{enumerate}
After the user registration is completed successfully, the user can proceed to register its agents with the \provider.

\remove{

\begin{figure}
    \centering
    \includegraphics[width=0.6\columnwidth]{Figures/UserReg.jpg}
    \caption{User registration flow}
    \label{fig:UserReg}
\end{figure}
}

\begin{figure}
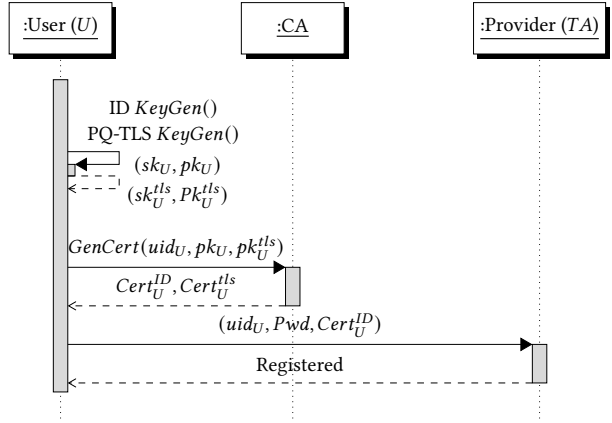

\centering
\Description[User Registration]{User Registration with the provider.}
\resizebox{\columnwidth}{!}{
\begin{sequencediagram}
    \newthread{U}{:User ($U$)}
    \newinst[2]{CA}{:CA}
    \newinst[2]{TA}{:Provider ($TA$)}
    \postlevel
    \begin{call}{U}{\shortstack{ID $KeyGen()$\\ PQ-TLS $KeyGen()$}}{U}{\shortstack{$(sk_U,pk_U)$\\ $(sk^{tls}_U, Pk^{tls}_U)$}}
    \end{call}
    \postlevel
    \begin{call}{U}{$GenCert(uid_U,pk_U, pk^{tls}_U)$}{CA}{$Cert^{ID}_U, Cert^{tls}_U$}
    \end{call}
    \begin{call}{U}{$(uid_U,Pwd,Cert^{ID}_U)$}{TA}{Registered}
    \end{call}
\end{sequencediagram}
}
\caption{User registration.}
\label{fig:UserReg}
\end{figure}

\subsection{Agent Registration}
The agent registration ensures that each agent is cryptographically bound to its user and a specific user's device (see Figure \ref{fig:AgentReg}).

\begin{enumerate}

    \item \textbf{Generating agent information.} The user selects an identifier 
    $\mathit{name}_A$ for the agent, forming a unique agent ID in combination 
    with their username:
    \[
        \mathit{aid}_A = \mathit{uid}_U : \mathit{name}_A
    \]
    The user specifies the agent's device name $\mathit{device}_A$, IP address 
    $\mathit{IP}_A$, and port $\mathit{port}_A$, forming the agent's endpoint 
    descriptor:
    \[
        \mathit{ED}_A = \langle\, \mathit{device}_A,\ \mathit{IP}_A,\ 
        \mathit{port}_A\, \rangle
    \]

    \item \textbf{Generating cryptographic keys.} The user generates the 
    following keys for the agent:
    \begin{itemize}
        \item \emph{PQ-TLS credentials} $(sk^{\mathit{tls}}_A, 
        pk^{\mathit{tls}}_A)$ for establishing secure channels with other 
        agents, with a CA-signed certificate:
        \[
            \mathit{Cert}^{\mathit{tls}}_A = 
            \mathsf{GenCert}_{sk_{\mathcal{CA}}}(\mathit{aid}_A,\, 
            pk^{\mathit{tls}}_A)
        \]
        \item \emph{An identity key pair} $(sk_A, pk_A)$ for authentication 
        and authorization, signed by the user:
        \[
            \sigma^{U}_{\mathit{ID}} = 
            \mathsf{HS.Sign}_{sk_U}(\mathit{aid}_A,\, pk_A)
        \]
    \end{itemize}
    The user additionally signs the agent's endpoint and key material, 
    including the provider's public keys to bind the agent to the specified 
    provider:
    \[
        \sigma^{U}_{A} = \mathsf{HS.Sign}_{sk_U}\!\left(\mathit{aid}_A,\; 
        \mathit{ED}_A,\; pk^{\mathit{tls}}_A,\; pk^{\mathit{tls}}_{\mathit{TA}},\; 
        pk_{\mathit{TA}}\right)
    \]

    \item \textbf{Specifying the contact policy.} The user specifies the 
    agent's contact policies $\mathit{CP}_A$, $\mathit{RCP}_A$, and 
    $\mathit{ICP}_A$.

    \item \textbf{User authentication to provider.} The user establishes a 
    PQ-TLS connection with the provider and authenticates by submitting 
    $\langle \mathit{uid}_U, \mathit{Pwd} \rangle$. The provider verifies the 
    credentials and proceeds if successful.

    \item \textbf{Registration submission.} The user submits to the provider:
    \begin{itemize}
        \item agent information $(\mathit{aid}_A,\, \mathit{ED}_A,\, 
        \mathit{CP}_A)$,
        \item PQ-TLS certificate $\mathit{Cert}^{\mathit{tls}}_A$ and identity 
        key $pk_A$,
        \item signatures $\sigma^{U}_{\mathit{ID}}$ and $\sigma^{U}_{A}$.
    \end{itemize}
    The agent stores locally the private keys $(sk^{\mathit{tls}}_A, sk_A)$ 
    corresponding to the public keys submitted to the provider.

    \item \textbf{Provider verification.} The provider processes the 
    registration by checking that $\mathit{aid}_A$ and $\mathit{ED}_A$ are 
    globally unique, verifying $\mathit{Cert}^{\mathit{tls}}_A$, and checking 
    the signatures:
    \[
        \mathsf{HS.Verify}_{pk_U}\!\left(
            \langle\, \mathit{aid}_A,\; \mathit{ED}_A,\; pk^{\mathit{tls}}_A,\; 
            pk^{\mathit{tls}}_{\mathit{TA}},\; pk_{\mathit{TA}}\, \rangle,\; 
            \sigma^{U}_{A}
        \right)
    \]
    \[
        \mathsf{HS.Verify}_{pk_U}\!\left(
            \langle\, \mathit{aid}_A,\; pk_A\, \rangle,\; 
            \sigma^{U}_{\mathit{ID}}
        \right)
    \]

    \item \textbf{Completion.} If verification passes, the provider stores the 
    following in the agent registry, where the agent's metadata is:
    \[
        M_A = \{\, \mathit{ED}_A,\ \mathit{Cert}^{\mathit{tls}}_A,\ 
        pk^{\mathit{tls}}_A,\ pk_A\,\}
    \]
    and the registry entry is:
    \[
        D_A[\mathit{aid}_A] \leftarrow \langle\, \mathit{uid}_U,\ M_A,\ 
        \mathit{CP}_A,\ \sigma^{U}_{\mathit{ID}},\ \sigma^{U}_{A}\, \rangle
    \]
    associating agent $A$ with user $U$. The provider then signs the agent's 
    information:
    \[
        \sigma^{\mathit{TA}}_{A} = \mathsf{HS.Sign}_{sk_{\mathit{TA}}}\!\left(
            \mathit{aid}_A,\; \mathit{Cert}^{\mathit{tls}}_A,\; \mathit{ED}_A,\; 
            pk_A,\; \sigma^{U}_A
        \right)
    \]
    and returns it as confirmation to the user. The user stores 
    $\sigma^{\mathit{TA}}_{A}$ in the agent's local registry for use when 
    initiating agent communication, completing the registration.

\end{enumerate}

\remove{
\begin{figure}[t]
    \centering
    \includegraphics[width=0.9\columnwidth]{Figures/AgentReg.jpg}
    \caption{Agent registration flow}
    \label{fig:AgentReg}
\end{figure}
}

\begin{figure}
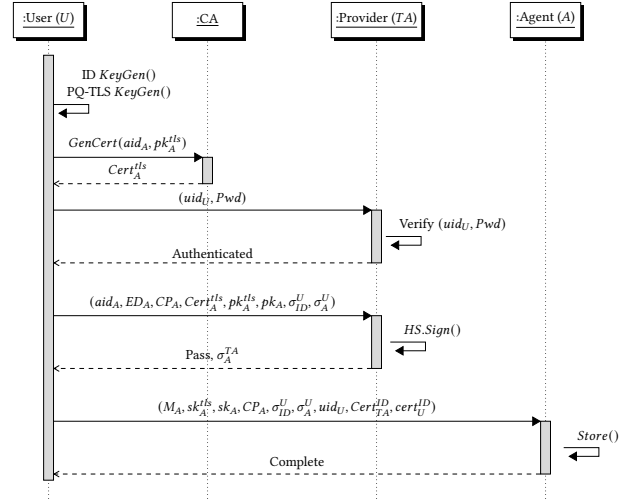

\centering
\resizebox{\columnwidth}{!}{
\begin{sequencediagram}
\usetikzlibrary{fit}
    \newthread{U}{:User ($U$)}
    \newinst[2]{CA}{:CA}
    \newinst[2]{TA}{:Provider ($TA$)}
    \newinst[2]{A}{:Agent ($A$)}
    \postlevel
    \callself{U}{\shortstack{ID $KeyGen()$\\ PQ-TLS $KeyGen()$}}{U}
    \postlevel
    \begin{call}{U}{$GenCert(aid_A,pk^{tls}_A)$}{CA}{$Cert^{tls}_A$}
    \end{call}
    \begin{call}{U}{$(uid_U,Pwd)$}{TA}{Authenticated}
    \postlevel
    \end{call}
    \prelevel \prelevel
    \callself{TA}{Verify $(uid_U,Pwd)$}{TA}
    \postlevel
    \postlevel
    \begin{call}{U}{$(aid_A,ED_A, CP_A, Cert^{tls}_A, pk^{tls}_A, pk_A, \sigma^{U}_{ID}, \sigma^{U}_{A})$}{TA}{Pass, $\sigma^{TA}_{A}$}
    \postlevel
    \end{call}
    \prelevel \prelevel
    \callself{TA}{$HS.Sign()$}{TA}
    \postlevel
    \postlevel
    \begin{call}{U}{$(M_A,sk^{tls}_A,sk_A,CP_A, \sigma^{U}_{ID}, \sigma^{U}_{A}, uid_U, Cert^{ID}_{TA}, cert^{ID}_{U})$}{A}{Complete}
    \postlevel
    \end{call}
    \prelevel\prelevel
    \callself{A}{$Store()$}{A}
    \postlevel
\end{sequencediagram}
}
\caption{Agent registration}
\label{fig:AgentReg}
\Description{To avoid warning.}
\end{figure}

\subsection{Agent-to-Agent Intercommunication}
\label{sec:one-to-one}

We use the {\em following notations} 
to describe the agent's policies: (i) Contact Policy (CP)  specifies, for each agent $A_i$, its authorized \ases\ peers and the maximum
number of \ases\ it may establish with them. 
(ii) Initiator Contact Policy (ICP) declares the \caler \ 
\tmsg budget and time budget for interacting with an \caled
(in the case of two-agent setting, ICP is defined for an \ases, and in the one-to-many setting ICP is defined as the total budget for the \cses), and (iii) Responder Contact Policy (RCP) defines the \caled\ \tmsg budget and time budget for interacting with an \caler \ in a \ases.
Below we assume that all communication is over PQ-TLS and omit the details of secure channel establishment. 

\subsubsection{Agent discovery.} Before establishing an {\em A-session}, an \caler \ must first obtain information about the responding agent \caled \  by contacting the \provider. The goal of this protocol is to allow  \caler \ to obtain authenticated information about the identity and the policies of \caled. The \provider\ checks whether communication between \caler\ and \caled\ is
authorized under $CP_{A_I} \cap CP_{A_R}$ and, if eligible, returns signed
information for \caled \
(see Figure \ref{fig:AgentProv}). The specific steps are:
    
\textbf{Step 1: Retrieval.}
Agent $A_I$ requests permission to contact agent $A_R$ by sending 
both identities, $\mathit{aid}_{A_I}$ and $\mathit{aid}_{A_R}$, to the Provider.
The Provider verifies mutual authorization between the two agents and 
checks that $\mathit{Counter}[\mathit{aid}_{A_R}][\mathit{aid}_{A_I}] > 0$.
If so, the Provider returns $A_R$'s access information together with a 
signature $\sigma^{\mathsf{TA}}_{\mathsf{ac}}$ over both agents' data.

The returned access information includes: user $U_R$'s identity certificate 
$\mathit{Cert}^{ID}_{U_R}$; agent $A_R$'s device and network information 
$\mathit{ED}_{A_R}$; $A_R$'s transport-layer key and certificate 
$(\mathit{Cert}^{\mathsf{tls}}_{A_R}, pk^{\mathsf{tls}}_{A_R})$; 
$U_R$'s signature on $A_R$'s identity key $\sigma^{U_R}_{ID}$; 
$U_R$'s signature on $A_R$'s full information $\sigma^{U_R}_{A_R}$; 
and the initiating agent's identity $\mathit{aid}_{A_I}$ and key $pk_{A_I}$.

$A_R$'s metadata is defined as:
\[
    M_{A_R} = \bigl\{\,
        \mathit{ED}_{A_R},\;
        \mathit{Cert}^{\mathsf{tls}}_{A_R},\;
        pk^{\mathsf{tls}}_{A_R},\;
        pk_{A_R}
    \,\bigr\}
\]

The Provider's signature is computed as:
\[
    \sigma^{\mathsf{TA}}_{\mathsf{ac}} = \mathsf{HS.Sign}_{sk_{\mathsf{TA}}}\!\bigl(
        \nu,\;
        \mathit{Cert}^{ID}_{U_R},\;
        \mathit{aid}_{A_R},\;
        M_{A_R},\;
        \sigma^{U_R}_{ID},\;
        \sigma^{U_R}_{A_R},\;
        \mathit{aid}_{A_I},\;
        pk_{A_I}
    \bigr)
\]

Finally, the Provider decrements $\mathit{Counter}[\mathit{aid}_{A_R}][\mathit{aid}_{A_I}]$ by one.

\textbf{Step 2: Verification.}
Agent $A_I$ verifies that $A_R$'s user certificate $\mathit{Cert}^{ID}_{U_R}$, 
including the public key $pk_{U_R}$, is valid. $A_I$ then verifies the 
following three signatures:
\[
    \mathsf{HS.Verify}_{pk_{U_R}}\!\bigl(
        \langle \mathit{aid}_{A_R},\, \mathit{ED}_{A_R},\, pk^{\mathsf{tls}}_{A_R},\,
        pk^{\mathsf{tls}}_{\mathsf{TA}},\, pk_{\mathsf{TA}} \rangle,\;
        \sigma^{U_R}_{A_R}
    \bigr)
\]
\[
    \mathsf{HS.Verify}_{pk_{U_R}}\!\bigl(
        \langle \mathit{aid}_{A_R},\, pk_{A_R} \rangle,\;
        \sigma^{U_R}_{ID}
    \bigr)
\]
\[
    \mathsf{HS.Verify}_{pk_{\mathsf{TA}}}\!\bigl(
        \langle \mathit{Cert}^{ID}_{U_R},\, \mathit{aid}_{A_R},\, M_{A_R},\,
        \sigma^{U_R}_{ID},\, \sigma^{U_R}_{A_R},\,
        \mathit{aid}_{A_I},\, pk_{A_I} \rangle,\;
        \sigma^{\mathsf{TA}}_{\mathsf{ac}}
    \bigr)
\]

\begin{figure}
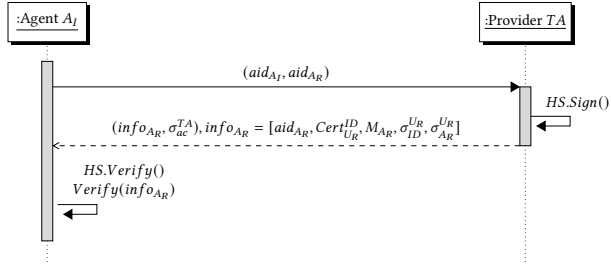

\centering
\resizebox{\columnwidth}{!}{
\begin{sequencediagram}
    \newthread{AI}{:Agent $A_I$}
    \newinst[8]{TA}{:Provider $TA$}
    \begin{call}[2]{AI}{$(aid_{A_I}, aid_{A_R})$}{TA}{$(info_{A_R}, \sigma^{TA}_{ac}), info_{A_R}=[aid_{A_R},Cert^{ID}_{U_R},M_{A_R},\sigma^{U_R}_{ID}, \sigma^{U_R}_{A_R}]$}
    \end{call}
    \prelevel\prelevel
    \callself{TA}{$HS.Sign()$}{TA}
    \postlevel
    \postlevel
    \callself{AI}{\shortstack{$HS.Verify()$\\ $Verify(info_{A_R})$}}{AI}
    \postlevel
\end{sequencediagram}
}
\caption{Agent discovery}
\label{fig:AgentProv}
\Description{To avoid warning.}
\end{figure}

\subsubsection{A-session protocol.}
Once $A_I$ obtained $A_R$'s information can establish an {\em A-session}.
Let the \tmsg \ budgets of the \caler \ and the \caled \ be denoted by $n$ and $n'$, respectively.
Establishing an A-session consists of two phases: {\em handshake} and {\em transmission.} (see Figure \ref{fig:Asession}). 

\begin{figure*}
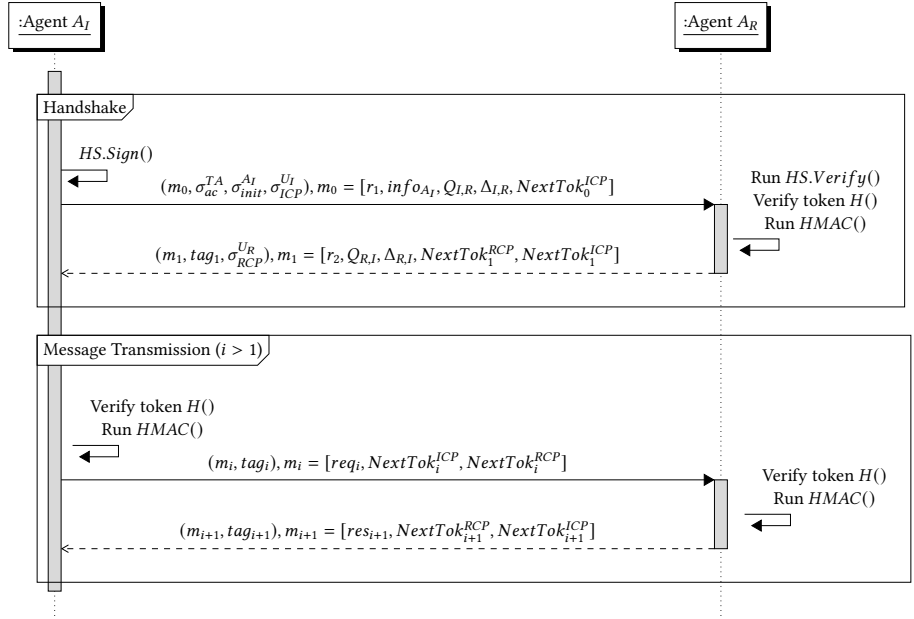

\centering
\resizebox{0.7\textwidth}{!}{
\begin{sequencediagram}
\usetikzlibrary{fit}
    \newthread{AI}{:Agent $A_I$}
    \newinst[10]{AR}{:Agent $A_R$}
    \begin{sdblock}{Handshake}{}
    \callself{AI}{$HS.Sign()$}{AI}
    \begin{call}[2]{AI}{$(m_0, \sigma^{TA}_{ac}, \sigma^{A_I}_{init}, \sigma^{U_I}_{ICP}), m_0=[r_1,info_{A_I}, Q_{I,R}, \Delta_{I,R},NextTok^{ICP}_0]$}{AR}{$(m_1, tag_1, \sigma^{U_R}_{RCP}), m_1=[r_2,Q_{R,I},\Delta_{R,I},NextTok^{RCP}_{1}, NextTok^{ICP}_{1}]$}
    \end{call}
    \prelevel\prelevel
    \callself{AR}{\shortstack{Run $HS.Verify()$\\Verify token $H()$\\ Run $HMAC()$}}{AR}
    \postlevel
    \end{sdblock}
    \begin{sdblock}{Message Transmission ($i > 1$)}{}
    \postlevel
    \callself{AI}{\shortstack{Verify token $H()$\\Run $HMAC()$}}{AI}
    \begin{call}[2]{AI}{$(m_i,tag_i), m_i=[req_{i},NextTok^{ICP}_{i}, NextTok^{RCP}_{i}]$}{AR}{$(m_{i+1}, tag_{i+1}), m_{i+1}=[res_{i+1},NextTok^{RCP}_{i+1}, NextTok^{ICP}_{i+1}]$}
    \end{call}
    \prelevel \prelevel
    \callself{AR}{\shortstack{Verify token $H()$\\Run $HMAC()$}}{AR}
    \postlevel
    \end{sdblock}
\end{sequencediagram}
}
\caption{A-session establishment}
\label{fig:Asession}
\Description{To avoid warning.}
\end{figure*}

 
{\em Handshake.}
The goal of the handshake phase is to establish an authorized, budgeted, and
cryptographically bound session between \caler\ and \caled. In this phase,
\caler\ presents the \provider-signed authorization token and its signed
commitment to the initiating hash chain, while \caled\ verifies these values,
sets its local message and time limits, commits to its responding hash chain,
and returns the corresponding signed policy information. The agents then derive
a shared session key from their exchanged randomness, which is used to protect
subsequent \tmsg\ in the session. The details steps during the handshake are:
    
\textbf{Step 1: Initiating the A-session.}
Agent $A_I$ generates a symmetric key $r_1$ and creates a personalized hash 
chain of length $Q_{I,R} = n'$ using the seed 
\[
    \rho_0 = \mathsf{PRF}(\mathit{sid},\, \mathit{aid}_{A_I},\, \mathit{aid}_{A_R},\, \mathit{addr})
\]
where $\mathit{addr} = 0$ for the first chain used in this A-session. The chain is:
\begin{align*}
    \rho_0, \quad
    \rho_1 &= H(\rho_0,\, 1,\, \mathit{sid},\, \mathit{aid}_{A_R}), \\
    \rho_2 &= H(\rho_1,\, 2,\, \mathit{sid},\, \mathit{aid}_{A_R}), \\
           &\;\vdots \\
    \rho_{n'} &= H(\rho_{n'-1},\, n',\, \mathit{sid},\, \mathit{aid}_{A_R})
\end{align*}
$A_I$ also obtains the user's signature on the root of the chain:
\[
    \sigma^{U}_{ICP} = \mathsf{HS.Sign}_{sk_{U}}(\rho_{n'},\, Q_{I,R})
\]

$A_I$ then initializes two counters: $\mathit{ctr}'_{ICP}[\mathit{aid}_{A_I}][\mathit{aid}_{A_R}]$, 
set to $n'$, and $\mathit{ctr}'_{RCP}[\mathit{aid}_{A_R}][\mathit{aid}_{A_I}]$. 
It decrements $\mathit{ctr}'_{ICP}[\mathit{aid}_{A_I}][\mathit{aid}_{A_R}]$ by one.

Next, $A_I$ sets $\mathit{NextTok}^{ICP}_0 = (\rho_{n'},\, \rho_{n'-1})$ and constructs:
\[
    m_0 = (r_1,\, \mathit{info}_{A_I},\, Q_{I,R},\, \Delta_{I,R},\, \mathit{NextTok}^{ICP}_0)
\]
where $\mathit{info}_{A_I} = (\mathit{Cert}^{ID}_{U_I},\, \mathit{aid}_{A_I},\, M_{A_I},\, 
\sigma^{U_I}_{ID},\, \sigma^{U_I}_{A_I})$ is $A_I$'s own information, 
$Q_{I,R}$ is the interaction budget, and $\Delta_{I,R}$ is the time budget.

$A_I$ then signs $m_0$ using a hash-based signature scheme:
\[
    \sigma^{A_I}_{init} = \mathsf{HS.Sign}_{sk_{A_I}}(m_0)
\]
and sends $(m_0,\, \sigma^{\mathsf{TA}}_{ac},\, \sigma^{A_I}_{init},\, \sigma^{U_I}_{ICP})$ 
to $A_R$, where $\sigma^{U_I}_{ICP}$ is the user's signature on the policies defined in 
the ICP. Finally, $A_I$ computes the session expiry time as 
$t'_{exp} = \tau_{now} + \Delta_{I,R}$, where $\tau_{now}$ is the current time 
and $\Delta_{I,R}$ is the maximum allowed session duration.
    
 \textbf{Step 2: Verifying the first message.}
Agent $A_R$ verifies $\sigma^{\mathsf{TA}}_{ac}$ and $A_I$'s signature using:
\[
    \mathsf{HS.Verify}_{pk_{\mathsf{TA}}}\!\bigl(
        \langle \nu,\, \mathit{Cert}^{ID}_{U_R},\, \mathit{aid}_{A_R},\, M_{A_R},\,
        \sigma^{U_R}_{ID},\, \sigma^{U_R}_{A_R} \rangle,\;
        \sigma^{\mathsf{TA}}_{ac}
    \bigr)
\]
\[
    \mathsf{HS.Verify}_{pk_{A_R}}\!\bigl(m_0,\; \sigma^{A_I}_{init}\bigr)
\]
If verification passes, $A_R$ verifies the user's signature on the ICP policy:
\[
    \mathsf{HS.Verify}_{pk_{U_I}}\!\bigl(
        \langle \rho_{n'},\, n' \rangle,\; \sigma^{U_I}_{ICP}
    \bigr) = 1
\]
It then checks that $\mathit{NextTok}^{ICP}_0 = (\rho_{n'},\, \rho_{n'-1})$ is 
correct by verifying $\rho_{n'} = H(\rho_{n'-1},\, n',\, \mathit{sid},\, \mathit{aid}_{A_R})$.

$A_R$ initializes two counters, $\mathit{ctr}_{ICP}\allowbreak[\mathit{aid}_{A_R}]\allowbreak[\mathit{aid}_{A_I}]$ 
and $\mathit{ctr}_{RCP}\allowbreak[\mathit{aid}_{A_R}]\allowbreak[\mathit{aid}_{A_I}]$, to track the number 
of requests, setting each according to the policies of $A_I$ and $A_R$. It then 
decrements $\mathit{ctr}_{ICP}\allowbreak[\mathit{aid}_{A_R}]\allowbreak[\mathit{aid}_{A_I}]$ by one. 
It also checks the maximum time budget $\Delta_{R,I}$ it is allowed to allocate 
to $A_I$ as determined in the RCP, and computes the session expiry time as 
$t_{exp} = \tau_{now} + \Delta_{R,I}$. It stores $t_{exp}$ and 
$\mathit{ctr}_{RCP}\allowbreak[\mathit{aid}_{A_R}]\allowbreak[\mathit{aid}_{A_I}]$ as part of the session state.

$A_R$ then generates a seed 
$s_0 = \mathsf{PRF}(\mathit{sid},\, \mathit{aid}_{A_R},\, \mathit{aid}_{A_I},\, \mathit{addr})$
(where $\mathit{addr} = 0$ for the first chain) and builds a personalized hash 
chain of length $Q_{R,I} = n$:
\begin{align*}
    s_0, \quad s_1 &= H(s_0,\, 1,\, \mathit{sid},\, \mathit{aid}_{A_I}), \\
    s_2 &= H(s_1,\, 2,\, \mathit{sid},\, \mathit{aid}_{A_I}), \\
        &\;\vdots \\
    s_{n} &= H(s_{n-1},\, n,\, \mathit{sid},\, \mathit{aid}_{A_I})
\end{align*}
It also obtains the user's signature on the root of the chain:
\[
    \sigma^{U_R}_{RCP} = \mathsf{HS.Sign}_{sk_{U_R}}(s_{n},\, Q_{R,I})
\]

Next, $A_R$ generates a random value $r_2$ and computes a session key 
$k_{ses} = H(r_1,\, r_2)$, used throughout all communications with $A_I$ 
during the A-session. It constructs the response message:
\[
    m_1 = (r_2,\, Q_{R,I},\, \Delta_{R,I},\, \mathit{NextTok}^{RCP}_1,\, \mathit{NextTok}^{ICP}_1)
\]
where $\mathit{NextTok}^{ICP}_1 = \rho_{n'-1}$ is the token received from $A_I$, 
and $\mathit{NextTok}^{RCP}_1 = (s_{n},\, s_{n-1})$ is the next token it has generated.
It computes an authentication tag $\mathit{tag}_1 = \mathsf{HMAC}_{k_{ses}}(m_1)$ 
and sends $(m_1,\, \sigma^{U_R}_{RCP},\, \mathit{tag}_1)$ back to $A_I$.
Finally, $A_R$ decrements $\mathit{ctr}_{RCP}[\mathit{aid}_{A_R}][\mathit{aid}_{A_I}]$ 
by one and removes $s_{n}$ from storage.
    
    {\em Transmission.} 
After the handshake, \caler\ and \caled\ exchange \tmsg\ protected under the
session key. Each message carries the next hash-chain token, allowing the
receiver to verify message authenticity, token progression, and compliance with
the remaining \tmsg \ and time budgets. The exchange continues until the task
ends, a budget is exhausted, or the session expires, and may span multiple
PQ-TLS sessions.
    The steps during the transmission phase are:  
    
\textbf{Step $i$: Making request $i \geq 2$.}
Agent $A_I$ verifies the response received in the previous round by checking:
\[
    \mathit{tag}_{i-1} == \mathsf{HMAC}_{k_{ses}}(m_{i-1})
\]
where $k_{ses} = H(r_1, r_2)$ and:
\[
    m_1 = (r_2,\, \mathit{res}_1,\, n,\, \Delta_{R,I},\, \mathit{NextTok}^{ICP}_1,\,
    \mathit{NextTok}^{RCP}_1,\, \sigma^{U_R}_{RCP})
\]

If $i = 2$, i.e. this is the first time $A_I$ receives a response from $A_R$,
it additionally verifies that $\mathit{NextTok}^{ICP}_1 = \rho_{n'-1}$, checks
the user's signature on the RCP:
\[
    \mathsf{HS.Verify}_{pk_{U_R}}\!\bigl(
        \langle s_n,\, Q_{R,I} \rangle,\; \sigma^{U_R}_{RCP}
    \bigr) = 1
\]
and verifies the chain link:
\[
    s_{n} = H(s_{n-1},\, n,\, \mathit{sid},\, \mathit{aid}_{A_I})
\]
It then checks $t_{exp} \leq \tau'_{now} + \Delta_{R,I}$, initializes counter
$\mathit{ctr}'_{RCP}[\mathit{aid}_{A_I}]\allowbreak[\mathit{aid}_{A_R}] = n$,
and decrements it by one.

If $i > 2$, $A_I$ instead verifies that $\mathit{NextTok}^{RCP}_i$ is correct
by checking:
\[
    s_{n-i+1} = H(s_{n-i},\, n-i+1,\, \mathit{sid},\, \mathit{aid}_{A_I})
\]
and that $\mathit{ctr}'_{RCP}[\mathit{aid}_{A_I}]\allowbreak[\mathit{aid}_{A_R}]
= n-i+1$, then decrements it by one.

In both cases, $A_I$ sets the next initiator token $\mathit{NextTok}^{ICP}_{i}
= \rho_{n'-i}$, removes $\rho_{n'-i+1}$, decrements
$\mathit{ctr}'_{ICP}[\mathit{aid}_{A_I}]\allowbreak[\mathit{aid}_{A_R}]$ by one,
and constructs and sends the following to $A_R$:
\[
    m_i = (\mathit{req}_i,\, \mathit{NextTok}^{ICP}_i,\, \mathit{NextTok}^{RCP}_i),
\]
\[
    \mathit{tag}_{i} = \mathsf{HMAC}_{k_{ses}}(m_i)
\]
Note that if the PQ-TLS session was closed, $A_I$ first re-establishes a PQ-TLS
channel with $A_R$ before sending. If
$\mathit{ctr}'_{RCP}[\mathit{aid}_{A_I}]\allowbreak[\mathit{aid}_{A_R}] = 0$,
$A_I$ verifies the chain seed:
\[
    s_0 = \mathsf{PRF}(\mathit{sid},\, \mathit{aid}_{A_R},\,
    \mathit{aid}_{A_I},\, \mathit{addr})
\]

\textbf{Step $i+1$: Responding to request $i$.}
Agent $A_R$ first checks that
$\mathit{ctr}_{RCP}[\mathit{aid}_{A_R}]\allowbreak[\mathit{aid}_{A_I}] > 0$.
If the quota $Q_{R,I} = n$ has been reached, it returns a quota-exhausted message
and terminates the connection. It also checks that the session expiry time
$t_{exp}$ has not been reached; if it has, it returns a session-expired message
and terminates the connection.

Otherwise, $A_R$ verifies the tag and checks the next responder token:
\[
    \mathit{tag}_i == \mathsf{HMAC}_{k_{ses}}(m_i), \qquad
    \mathit{NextTok}^{RCP}_i = s_{n-i+1}
\]
If verification passes, $A_R$ executes the request and constructs the response:
\[
    m_{i+1} = (\mathit{res}_{i+1},\, \mathit{NextTok}^{ICP}_{i+1},\, \mathit{NextTok}^{RCP}_{i+1})
\]
\[
    \mathit{tag}_{i+1} = \mathsf{HMAC}_{k_{ses}}(m_{i+1})
\]
where $\mathit{NextTok}^{ICP}_{i+1} = \rho_{n'-i}$ and
$\mathit{NextTok}^{RCP}_{i+1} = s_{n-i}$, and sends
$(m_{i+1}, \mathit{tag}_{i+1})$ to $A_I$. It then decrements both
$\mathit{ctr}_{ICP}[\mathit{aid}_{A_I}]\allowbreak[\mathit{aid}_{A_R}]$ and
$\mathit{ctr}_{RCP}[\mathit{aid}_{A_R}]\allowbreak[\mathit{aid}_{A_I}]$ by one,
and removes $s_{n-i+1}$. If
$\mathit{ctr}_{ICP}[\mathit{aid}_{A_I}]\allowbreak[\mathit{aid}_{A_R}] = 0$,
$A_R$ verifies the chain seed:
\[
    \rho_0 = \mathsf{PRF}(\mathit{sid},\, \mathit{aid}_{A_I},\,
    \mathit{aid}_{A_R},\, \mathit{addr})
\]
 \subsection{Multi-Agent Intercommunication}
 \label{sec:one-to-many}
For  tasks that involve multiple agents, the \masterw\ agent receives the task and, under a static plan,
identifies the agents to contact. It generates the hash chains for the
collaborating agents in the \cses\ and sends the roots to the user, who builds
and signs a Merkle root over them and returns the tree and the signature to the
\masterw. Details are given in Appendix~\ref{Appendix:CSessionUserAgent}.

\subsubsection{Agent discovery.}
The Agent discovery protocol is similar to that for the two-agent protocol; it is run before the inter-agent communication and allows the \masterw agent to receive a token from the \provider for each A-session that it wants to establish if it is authorized to contact other agents according to CP. The inter-agent communication stage handles the A-sessions that the \masterw agent establishes with other agents according to the plan of the task. In each A-session RCP is determined by the receiving agent, and ICP is determined by the orchestrator agent (enforcing ICP in each session corresponds to enforcing an interaction budget and time budget in each A-session). 
 Agent $A^{*}_{I}$ who wants to establish a connection with agent $A_j$, where $j \in [1,t]$, first contacts the \provider to receive $A_j$'s information. Additionally, the \provider checks whether $A^{*}_{I}$ is allowed to communicate with $A_j$ (according to the contact policy) or not and sends the signed information of $A_j$ if $A^{*}_{I}$ is eligible (see Figure \ref{fig:AgentProv}).

\subsubsection{$C$-session protocol.}
 We consider that the \masterw agent $A^{*}_{I}$ communicates with agents $\{A_1, \cdots, A_t\}$ through established A-sessions according to the workflow to fulfill the given task. Each A-session has two policies $RCP$ and $ICP$ that are enforced by the (user and) agents involved in that A-session. 
 For efficiency of \master \ computation in a \cses, a single user-signed Merkle root commits to the
initiator-side hash-chain roots of all planned component \ases. Each component
\ases\ then uses its usual \ases\ authorization token together with the relevant
Merkle proof to show that it belongs to the user-authorized \cses.
 
\noindent\textbf{Establishing subroutine A-sessions with $\boldsymbol{A_j}$, 
$\boldsymbol{j \in [1,t]}$. }
Agent $A^{*}_{I}$ follows the A-session establishment protocol for the two-agent 
setting described earlier, with the following differences.

First, $A^{*}_{I}$ uses the signature $\sigma^{U^*}_{ICP}$ and the personalized 
hash chains obtained during the \textit{User-agent interaction} protocol, rather 
than generating new hash chains. These chains are defined as:
\begin{align*}
    \rho^i_0 &= \mathsf{PRF}(\mathit{sid},\, \mathit{aid}_{A^{*}_{I}},\, 
                \mathit{aid}_{A_j},\, \mathit{addr}), \\
    \rho^i_1 &= H(\rho^i_0,\, 1,\, \mathit{sid}_j,\, \mathit{aid}_{A_j}), \\
    \rho^i_2 &= H(\rho^i_1,\, 2,\, \mathit{sid}_j,\, \mathit{aid}_{A_j}), \\
             &\;\vdots \\
    \rho^i_{n'} &= H(\rho^i_{n'-1},\, n',\, \mathit{sid}_j,\, \mathit{aid}_{A_j})
\end{align*}
for all $i \in [0, m-1]$.

Second, when $A^{*}_{I}$ sends its request together with $\sigma^{U^*}_{ICP}$, 
it also sends the Merkle root $\mathit{Mroot}$ and provides a Merkle proof 
showing that the tuple $\rho^i_{n'}$ is included in the Merkle tree whose 
root has been signed.

Finally, when agent $A_j$ handles the request, it verifies the correctness of 
the Merkle proof. Throughout these communication rounds, $A^{*}_{I}$ also 
checks before each new request that the maximum allowed time specified in the 
ICP has not been exceeded.

\section{A Formalization of \sagaplus~ Properties}
\label{sec:model}
In this section we describe the security properties and proof sketch for \sagaplus. The complete proof is presented in Appendix \ref{Appendix:securityanalysis}.

\subsection{Security Properties} \label{sec-prop}
\sagaplus provides the following security properties:\\
-- {\em Confidentiality.} The contents of \tmsg exchanged between two entities remain private and are accessible only to the participating entities, except as revealed by the intended protocol outputs.\\
-- {\em 
Authentication.} All communications between entities are mutually authenticated at both the network and application layers.\\
-- {\em Message integrity.} \Tmsg exchanged between two entities are protected against unauthorized modification, forgery, replay, and re-ordering.\\
-- {\em Accountability.} \Tmsg are publicly verifiably attributable to their sender agent and, through the certification/binding mechanism, to the associated user.\\
-- {\em Authorization soundness.} No agent can interact with another agent unless authorized by its owner. The system enforces user-defined contact lists and \emph{budgets} on the number and duration of 
sessions, and on the number of \tmsg, as specified by policy.

\subsection{Modeling \sagaplus in the UC Approach}
\label{Security model}
We model and analyze the security of \sagaplus\ using Canetti's Universally 
Composable (UC) framework~\cite{Canetti-18}. The UC framework captures the 
security of a protocol $\pi$ in a realistic setting where $\pi$ may run 
concurrently with other protocols, and where these protocols may interact with $\pi$ and with one another. More details about UC can be found in Appendix \ref{Appendix:ideal}.
 We assume \emph{static corruption}, where the adversary $\mathcal{A}$ corrupts
parties before protocol execution begins. This is a common simplifying assumption
in UC analyses \cite{dziembowski2018fairswap,canetti2020efficient,hesse2023password}.

Figure~\ref{UCsteps} gives an outline of the \sagaplus\ ideal functionality and
the corresponding real-world construction, together with the auxiliary ideal
functionalities available to each. The ideal functionalities of \ases\ and \cses\ formalize  the above security properties. See section \ref{ases-func}.

\begin{figure}[t]
    \centering
    \includegraphics[width=0.8\columnwidth]{
    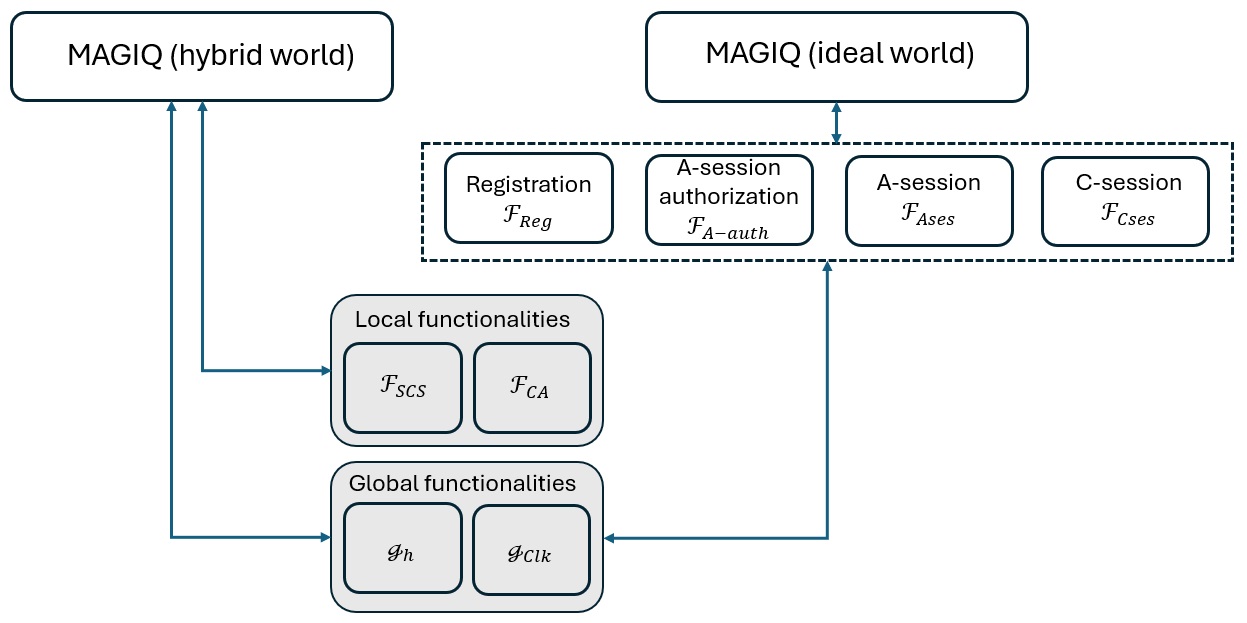}
    \caption{The modular model of \sagaplus \ in the hybrid and ideal world. 
    The dotted box shows the ideal functionalities we have defined for \sagaplus\ and the gray boxes show the  existing ideal functionalities used as subroutine in our model.}
    \label{UCsteps}
    \Description{To avoid warning.}
\end{figure}

\noindent
{\bf  Modeling \sagaplus\ in ideal-world.} \sagaplus \ in the ideal world uses the {\em global random oracle functionality $\mathcal{G}_h$} and the {\em global clock ideal functionality  $\mathcal{G}_{\mathsf{clk}}$} as subroutines. 
The former captures the assumptions of a synchronized global clock. Our analysis is in the $\mathcal{G}_h$-hybrid model, and so the proof is in the UC global random-oracle model.

\sagaplus\ consists of several protocols, whose security requirements are captured by corresponding ideal functionalities. These include the {\em registration functionality }$\mathcal{F}_{\mathsf{Reg}}$, modeling registration of users and agents with the \provider; the \ases-{\em authorization functionality } $\mathcal{F}_{A\text{-}\mathsf{auth}}$, modeling {\em agent discovery}  with the help of \provider\ that allows getting authorization to initiate an \ases; the {\em \ases\ functionality }$\mathcal{F}_{\mathsf{Ases}}$, modeling authenticated inter-agent communication and policy enforcement in a two-agent session; and the {\em \cses\ functionality} $\mathcal{F}_{\mathsf{Cses}}$, modeling the corresponding properties for a one-to-many agent session. The latter two functionalities interact with the global functionalities $\mathcal{G}_{\mathsf{clk}}$ and $\mathcal{G}_h$.

The functionalities $\mathcal{F}_{\mathsf{Reg}}$ and
$\mathcal{F}_{A\text{-}\mathsf{auth}}$ capture, respectively, ideal registration and authorization by checking requests against an authorization table. Their details are given in Appendix~\ref{Appendix:ideal}.

\noindent
{\em \ases \ Ideal Functionality.}
The ideal functionality of an \ases, denoted by
$\mathcal{F}_{\mathsf{Ases}}$, captures a policy-enforced secure session
involving three types of entities: a user who provides the task and two agents that interact to perform the task. In addition to enforcing
the session policy, the functionality captures communication security guarantees, including message confidentiality, integrity, authenticity, and attribution to the sending agent. $\mathcal{F}_{\mathsf{Ases}}$ has
access to the global clock functionality $\mathcal{G}_{\mathsf{clk}}$ and the
global random-oracle functionality $\mathcal{G}_h$.

\noindent
{\em \cses \ Ideal Functionality.}
A \cses\ is a one-to-many agent session in which \masterw distributes
its budget across multiple responders, enforcing its own
budget.
The ideal functionality of a \cses, denoted by
$\mathcal{F}_{\mathsf{Cses}}$, captures communication security and attribution guarantees, analogous to those of an \ases, together with policy enforcement across all participating agents.

\noindent
{\bf Modeling \sagaplus \ hybrid world.}
 %
%
\sagaplus\ is analyzed using widely adopted UC functionalities as subroutines:
the global random oracle $\mathcal{G}_h$ \cite{camenisch2018wonderful}, the global clock
$\mathcal{G}_{\mathsf{clk}}$ \cite{canetti2017clock}, the secure communication-session functionality
$\mathcal{F}_{\mathsf{scs}}$ \cite{gajek2008universally}, and the certificate-authority functionality
$\mathcal{F}_{\mathsf{CA}}$ \cite{canetti2004signature}. \sagaplus 
consists of several protocols
including registration, agent discovery, \ases, and \cses. An \ases\ uses
$\mathcal{F}_{\mathsf{scs}}$ as a post-quantum adaptation of the TLS secure
session functionality with security considered against quantum
polynomial-time adversaries and environments,  may span multiple PQ-TLS sessions, and relies on
$\mathcal{F}_{\mathsf{CA}}$ for certificate-based trust. A \cses\ is realized by
the one-to-many \sagaplus\ protocol using multiple instances of
$\mathcal{F}_{\mathsf{Ases}}$, with access to $\mathcal{G}_{\mathsf{clk}}$ and
$\mathcal{G}_h$.

\subsection{Security Proof Sketch}
\label{Proofsketch}

We analyze \sagaplus \ in the UC framework; the following theorem gives the security of our protocol:

\begin{theorem}
\label{thm:main}
\sagaplus\   (Section \ref{sec:sagaplus}) UC-realizes
$\mathcal{F}_{\mathsf{Reg}}$, $\mathcal{F}_{A\text{-}\mathsf{auth}}$, \Fases,
and \Fcses\ in the
$(\mathcal{F}_{\mathsf{scs}},\mathcal{F}_{\mathsf{CA}}, \mathcal{G}_{\mathsf{clk}},\mathcal{G}_h)$-hybrid model. 
The cryptographic building blocks of \sagaplus\ are (i) a digital signature scheme that is EUF-CMA, and (ii) hash function, HMAC and PRF that are modeled by global random oracles.
\end{theorem}

(It is sufficient for the signature scheme that are used by all \sagaplus\ participants  to be EUF-CMA in the standard model (see proof \ref{Appendix:securityanalysis})).

 \paragraph{Proof sketch.}
We construct the protocol
$\pi_{\sagaplus}^{\mathcal{F}_{\mathsf{scs}},
\mathcal{F}_{\mathsf{CA}},\mathcal{G}_{\mathsf{clk}},\mathcal{G}_h}$ in the $(\mathcal{F}_{\mathsf{scs}},\mathcal{F}_{\mathsf{CA}},
\mathcal{G}_{\mathsf{clk}},\mathcal{G}_h)$-hybrid model and prove Theorem~\ref{thm:main} through a sequence of lemmas, each capturing the security of one protocol in \sagaplus. 
These 
protocols are executed sequentially: an agent may participate in an \ases\ or a \cses\ only after registration, and an initiating agent may start the session only after obtaining authorization from the \provider. The lemmas are stated below; full details are given in Appendix~\ref{Appendix:securityanalysis}.

\begin{lemma} 
    The  registration protocol $\pi_{Reg}$ of $\pi^{\mathcal{F}_{\mathsf{SCS}}, \mathcal{F}_{\mathsf{CA}}, \mathcal{G}_{\mathsf{clk}},\mathcal{G}_h}_{\sagaplus}$ for the user and the agent, described  in Appendix \ref{Appendix:UserAgentReg}, UC realizes the registration ideal functionality $\mathcal{F}_{\mathsf{Reg}}$ against static adversaries. 
\end{lemma}

{\em Proof sketch.}  The user and the \provider\ are honest. The simulator $\mathsf{Sim}$ produces
outputs for the honest parties that are consistent with the ideal registration functionality: $\mathsf{Sim}$ generates the required key pairs and
simulates the CA interaction consistent with
$\mathcal{F}_{\mathsf{CA}}$. Since the simulated registration transcript with the \provider \ and the
certificates are distributed identically to those in the hybrid execution, no
environment can distinguish the ideal execution from the hybrid execution.

\begin{lemma}
    The \ases- agent discovery and authorization protocol $\pi_{A\text{-}\mathsf{auth}}$ of $\pi_{\sagaplus}^{\mathcal{F}_{\mathsf{scs}},
\mathcal{F}_{\mathsf{CA}},\mathcal{G}_{\mathsf{clk}},\mathcal{G}_h}$ described in Section~\ref{sec:one-to-one} UC-realizes the \ases-authorization ideal
functionality $\mathcal{F}_{A\text{-}\mathsf{auth}}$ against static adversaries. 
\end{lemma}

{\em Proof sketch.}
The simulator $\mathsf{Sim}$ runs the registration simulator described in previous Lemma, maintaining a consistent registration state. If
\caler\ and the \provider\ are honest, the authorization request and response are
generated honestly and consistently with $\mathcal{F}_{A\text{-}\mathsf{auth}}$.
If \caler\ is corrupted, $\mathsf{Sim}$ receives the request that $\mathcal{A}$ causes \caler\ to send, checks it against the registration state and authorization table, and returns the corresponding
provider response or rejection.

\begin{lemma}
    The \ases \ protocol $\pi_{Ases}$ of $\pi_{\sagaplus}^{\mathcal{F}_{\mathsf{SCS}}, \mathcal{F}_{CA}, \mathcal{G}_{\mathsf{clk}},\mathcal{G}_h}$ (described in Section \ref{sec:one-to-one}) UC realizes the \ases \ ideal functionality \Fases \ against static adversaries. 
\end{lemma}

{\em Proof sketch.}
 $\mathsf{Sim}$  uses the registration and authorization simulators to maintain consistent view for  registration and authorization tokens  and
policies, and simulates the \ases\ handshake and message
transmission. If both agents are honest, $\mathsf{Sim}$ relays the session messages to \Fases\ and delivers the outputs prescribed by \Fases\ that enforces policies.
If one agent is corrupted, $\mathsf{Sim}$ obtains adversarially chosen  messages of that agent through
$\mathcal{A}$, checks the  authorization token, user and agent signatures, HMAC tags, hash-chain token progression using $\mathcal{G}_h$, and the message/time budgets using $\mathcal{G}_{\mathsf{clk}}$.  Invalid messages cause $\mathsf{Sim}$ sends abort to \Fases. 
{\em In-distinguishability } follows from the EUF-CMA security of the signature scheme against quantum polynomial-time adversaries, and the  seudorandomness/unforgeability provided by the HMAC, PRF, and hash-chain checks in the $\mathcal{G}_h$ model.

\begin{lemma}
    The \cses \ protocol $\pi_{Cses}$ of $\pi^{\mathcal{F}_{\mathsf{SCS}}, \mathcal{F}_{CA}, \mathcal{G}_{\mathsf{clk}},\mathcal{G}_h}_{\sagaplus}$ (described in Section \ref{sec:one-to-one}) UC realizes the \cses \ ideal functionality \Fcses \ against static adversaries assuming that \ases \ $\pi_{Ases}$ UC realizes \Fases. 
\end{lemma}

 {\em Proof sketch.}  We consider three cases: (i) all agents are honest, (ii) only \master \ is corrupted, and (iii) all $A_R$'s are corrupted. In all cases $\mathsf{Sim}$ simulates each $\mathcal{F}_{\mathsf{A\text{-}auth}}$ and \Fases. In case (ii), $\mathsf{Sim}$ receives the subtasks for each \Fases \ from the corrupted \master \ and simulates \Fases \ accordingly, but in case (i) and (iii), $\mathsf{Sim}$ receives the task from the dummy \master \ and it runs the task internally and gets the subtasks for each \Fases \ that it will simulate.
\section{Experimental Results}
\label{sec:results}
In this section we present experimental results. 
We measure the overhead of \sagaplus{}
across cryptographic operations, protocol coordination, network
communication, and bandwidth. We further evaluate \sagaplus{} in a
multi-agent setting where an initiating agent coordinates with multiple
receiving agents to complete a task.

\subsection{Setup}
\textbf{Libraries.} We use XMSS\_\allowbreak SHA2\_\allowbreak 16\_\allowbreak 256 \cite{xmss_rfc}, an instance of XMSS using SHA-256, a Winternitz parameter of 16, and a 256-bit security level, to instantiate long-term agent identity keys. All inter-agent and agent–\provider communication is secured via post-quantum mutually authenticated TLS (PQ-mTLS), configured with X25519 and ML-KEM-768 for hybrid key exchange, ML-DSA-65 for transport-layer certificates and mutual authentication, and TLS\_\allowbreak AES\_\allowbreak 256\_\allowbreak GCM\_\allowbreak SHA384 as the cipher suite. The PQ-mTLS stack is implemented using OpenSSL 3.5, while post-quantum cryptographic primitives are provided via the liboqs-python library. SHA-256 serves as the hash function across all protocol operations. 

\noindent \textbf{Agents. }The LLM agents were implemented using AutoGen \cite{autogen}. For comparison with the experimental evaluation of \saga{}, we used gpt-4.1, gpt-4.1-mini, and locally hosted Qwen2.5-72B. In addition, we evaluated other cloud-hosted OpenAI models such as gpt-5.4, gpt-5.4-mini and locally deployed Qwen3.5-30B-Instruct to assess performance across a broader range of configurations.

\noindent \textbf{Device Specifications. }Experiments were conducted on Ubuntu 24.04, running on an AMD Ryzen Threadripper PRO 5955WX with 16 cores and 256~GB of RAM, with all reported results averaged over 100 runs. The Qwen 30B Instruct model ran on an NVIDIA RTX 4090 while Qwen-2.5 72B model used an NVIDIA H100. For all experiments, the user and the agents they own were assumed to reside on the same device.

\noindent \textbf{Baseline.} We use \saga as the baseline protocol. All cryptographic operations in \saga are built on Curve25519, with both long-term and ephemeral keys generated via the X25519 ECDH scheme using 256-bit shared secrets. Certificates follow the X.509 PKI standard and are issued by an internal certificate authority deployed as part of the \provider, with SHA-256 used for all digital signatures and key derivation. Cryptographic primitives were implemented using the cryptography library in Python 3.12. We used the LLM agents that came with the \saga distribution, implemented using the \texttt{smolagents} library. Experiments were conducted with a local Qwen-2.5 72B model on an NVIDIA H100, as well as cloud-hosted OpenAI models accessed via API. The code for \saga was obtained from the github repo \cite{saga_git}.

\subsection{Computation Overhead}

Table~\ref{tab:crypto_overhead} reports the computational overhead of protocol operations in \sagaplus against the \saga baseline. 
Note that for \saga, user registration, agent registration, and the setup phases on both the initiating and receiving agents are analogous to the respective handshake phases in \sagaplus. Notably, \saga does not include counterparts to the Handshake Phase (User) or Contact Resolution (Initiating) phase, as user involvement is not required during the handshake, and contact resolution is subsumed within the setup phase of the initiating agent. As reflected in Table \ref{tab:crypto_overhead}, the token decryption and validation for \saga are incorporated into the per-interaction data transfer costs on the initiating agent, with token validation similarly accounted for on the receiving agent. Token generation by the initiating agent is likewise subsumed within its handshake phase, rather than treated as a separate operation.

On the user side, both user and agent registration in \sagaplus incur $\sim$36~seconds, attributable to XMSS key tree pre-computation. This cost is paid only once at deployment. \provider-side registration, by contrast, takes at most 5.4~ms, confirming that post-quantum key management imposes no scalability burden on the infrastructure side.

Recurring per-session costs in \sagaplus remain uniformly sub-5~ms across all operations. Contact resolution, handshake phases across all three roles, and data transfer together constitute a modest operational footprint. Data transfer in \sagaplus is effectively free at 0.03~ms per interaction, compared to 1.44~ms on the initiating side in \saga --- a reduction of nearly two orders of magnitude. 

\begin{table}[t]
\caption{Computational overhead of \sagaplus{} and \saga{}. 
}
\label{tab:crypto_overhead}
\small
\begin{tabular}{lrr}
\toprule 
\textbf{Operation} & \textbf{\sagaplus{} (ms)} & \textbf{\saga{} (ms)}\\
\midrule
\multicolumn{3}{l}{\textit{User Registration}}\\ 
\quad User Registration (User)     & 36{,}215.33 & 2.34\\ 
\quad User Registration (Provider) & 0.26 & 194.09\\
\midrule
\multicolumn{3}{l}{\textit{Agent Registration}}\\
\quad Agent Registration (User)     & 36{,}002.27 & 15.09 \\
\quad Agent Registration (Provider) & 5.4 & 212.85\\
\midrule
\multicolumn{3}{l}{\textit{Agent Communication (per session)}} \\
\quad Contact Resolution (Provider)   & 2.96 & 1.46\\
\quad Contact Resolution (Initiating) & 4.1 & N/A\\
\quad Handshake Phase (Initiating)    & 4.01 & 3.17 \\
\quad Handshake Phase  (User)          & 1.55 & N/A \\
\quad Handshake Phase (Receiving)     & 4.71 & 1.83\\
\midrule
\multicolumn{3}{l}{\textit{Data Transfer (per interaction)}} \\
\quad Initiating & 0.03 & 1.44\\
\quad Receiving & 0.03 & 0.26\\
\bottomrule
\end{tabular}
\end{table}

\subsection{Protocol Overhead}
We evaluate the overhead introduced by the access control and \provider 
coordination mechanisms of \sagaplus compared with \saga. We consider a 
scenario in which an initiating agent \caler~ issues $m$ requests to a 
receiving agent \caled, subject to a maximum of $Q_{\max}$ requests per 
\ases between \caler~ and \caled. The total overhead comprises both a 
network component, associated with establishing secure communication, and 
a cryptographic component, denoted $t_{\mathrm{crypto}}$, which captures 
the costs of all phases involved in agent communication and data transfer, 
as detailed in Table~\ref{tab:crypto_overhead}. The total protocol overhead 
is modeled as:
\begin{equation}
  c_{\text{proto}}(m) =
    \bigl(\mathrm{RTT}_{A_i,P} + t_{\text{crypto}}\bigr)
    \cdot \left\lceil \frac{m}{Q_{\max}} \right\rceil,
  \label{eq:cproto}
\end{equation}
where $P$ is the \provider and $\mathrm{RTT}_{A_i,P}$ is the round-trip 
time between \caler~ and the \provider. $t_{\text{crypto}} = 20.33$~ms is 
the measured cryptographic overhead per \ases. Each authorization cycle 
must be performed once every $Q_{\max}$ requests.

We sample round-trip times from empirical measurement distributions using 
monitors in US-East, US-West, Europe, and Asia, made available by 
CAIDA~\cite{caida} and AWS~\cite{cloudping}. Figures~\ref{fig:amortised_overhead_a_anywhere} 
and~\ref{fig:amortised_overhead_p_fixed} show the amortized protocol 
overhead:
\begin{equation}
  \bar{c}_{\text{proto}}(m) = \frac{c_{\text{proto}}(m)}{m}
  \label{eq:cproto_bar}
\end{equation}
as a function of $Q_{\max}$, using $m = 100$ requests, under two 
configurations: one in which the \provider location varies 
across four geographic regions while agents are distributed worldwide 
(Fig.~\ref{fig:amortised_overhead_a_anywhere}), and one in which the 
\provider is fixed in US-West while the initiating agent \caler~ is placed 
across the same four regions 
(Fig.~\ref{fig:amortised_overhead_p_fixed}).

\begin{figure}[t]
    \centering
    \begin{subfigure}[b]{0.49\linewidth}
        \centering
        \includegraphics[width=\linewidth]{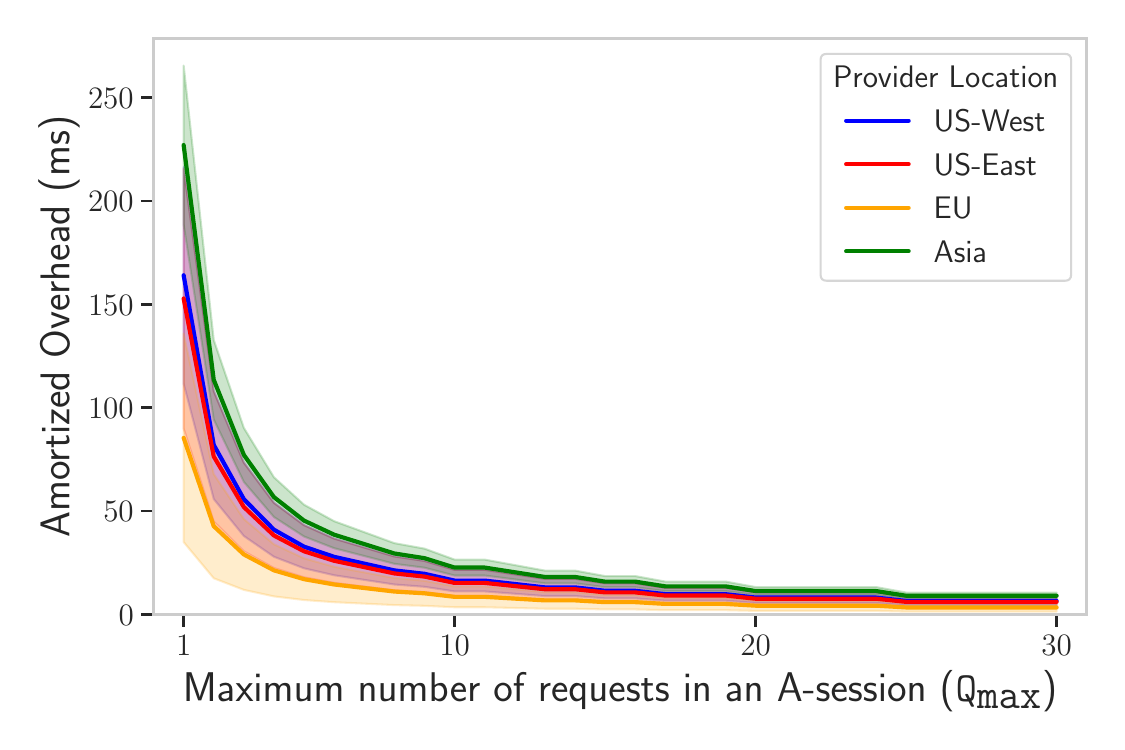}
        \caption{\sagaplus}
        \label{fig:amort_session_sagaplus}
    \end{subfigure}%
    \hfill
    \begin{subfigure}[b]{0.49\linewidth}
        \centering
        \includegraphics[width=\linewidth]{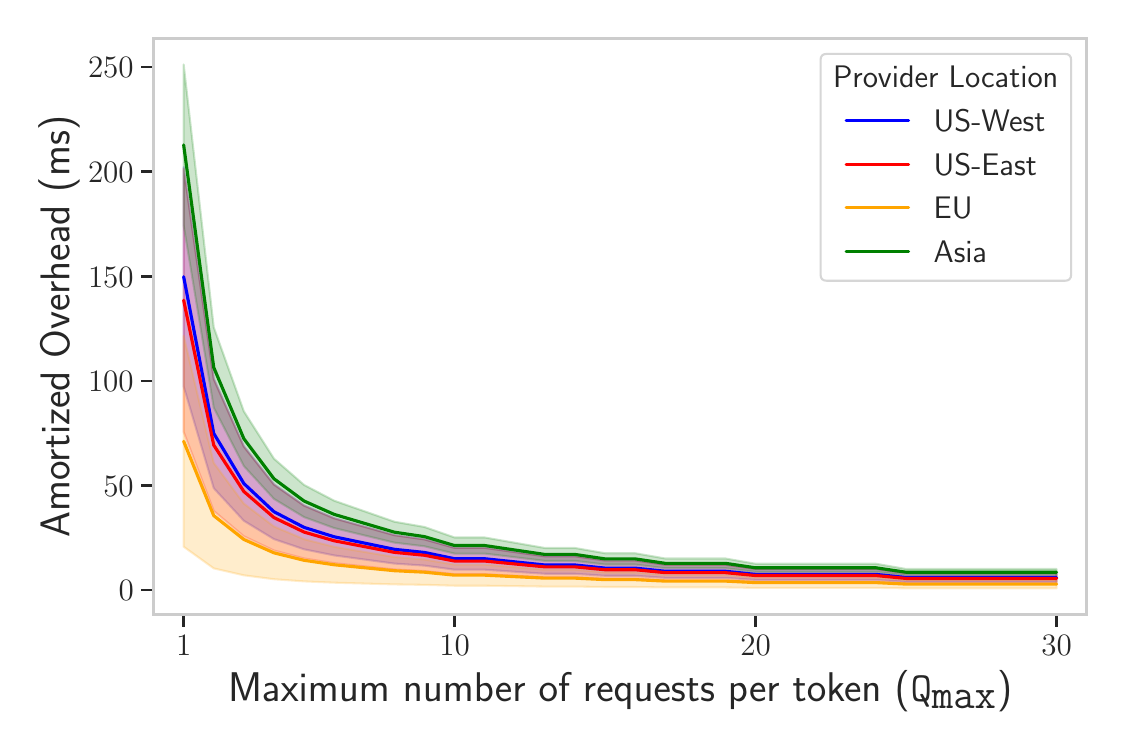}
        \caption{\saga}
        \Description{xx}
        \label{fig:amort_token_saga}
    \end{subfigure}
    \caption{Amortized protocol overhead across provider locations 
    (US-West, US-East, EU, Asia) under varying $Q_{\max}$ for $m = 100$ requests between \caler~ and 
    \caled. Shaded regions reflect the variability in overhead attributable 
    to differences in agent location worldwide.}
    \label{fig:amortised_overhead_a_anywhere}
\end{figure}

\begin{figure}[t]
    \centering
    \begin{subfigure}[b]{0.49\linewidth}
        \centering
        \includegraphics[width=\linewidth]{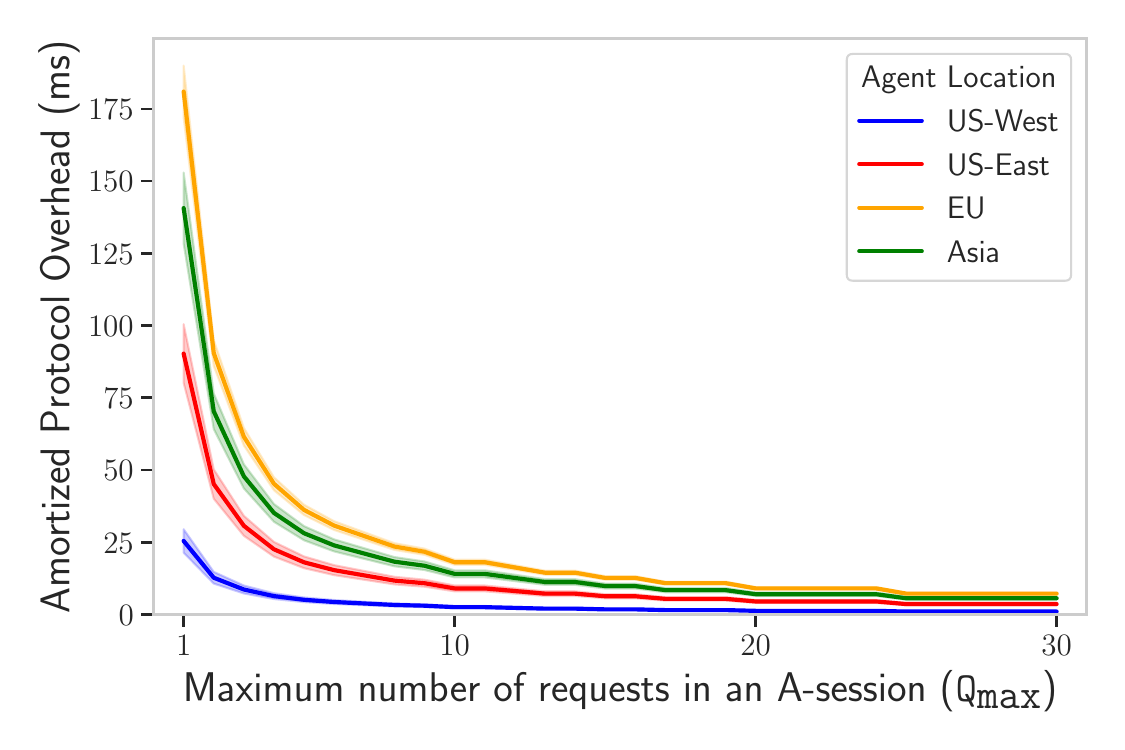}
        \caption{\sagaplus}
        \label{fig:amort_session}
    \end{subfigure}%
    \hfill
    \begin{subfigure}[b]{0.49\linewidth}
        \centering
        \includegraphics[width=\linewidth]{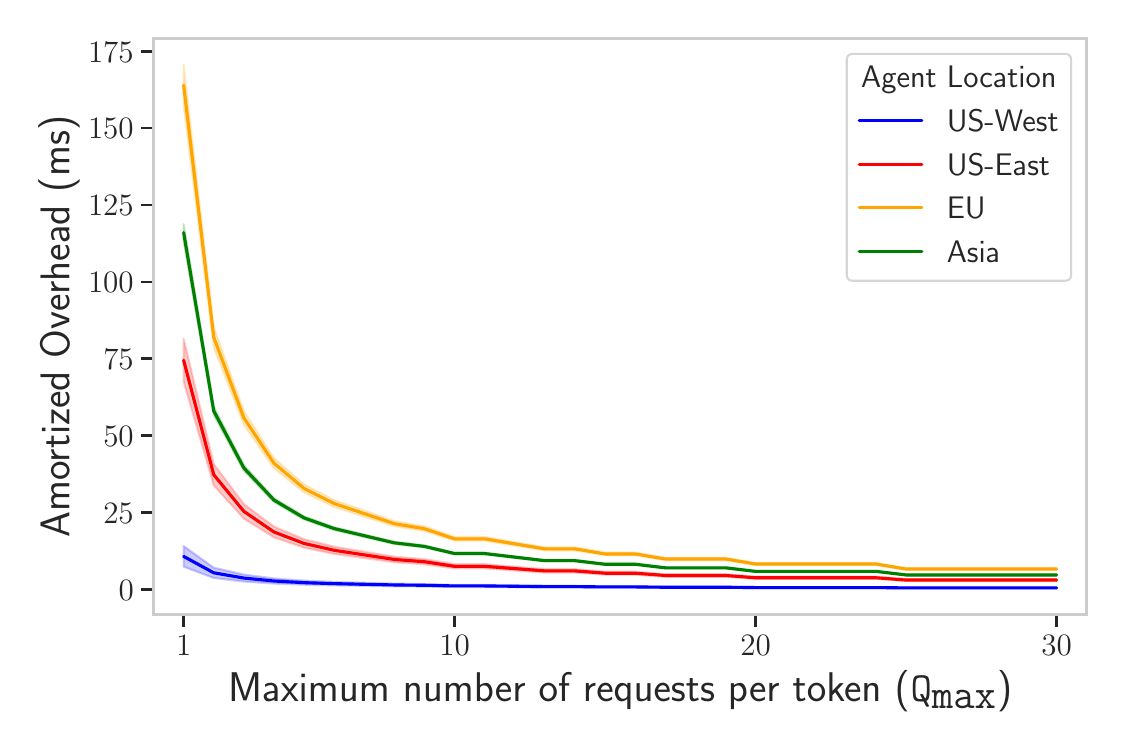}
        \caption{\saga}
        \label{fig:amort_token}
    \end{subfigure}
    \caption{Amortised protocol overhead across \caler~ locations 
    (US-West, US-East, EU, Asia) under varying $Q_{\max}$ for $m = 100$ requests between \caler~ and 
    \caled. Shaded regions reflect the variability in overhead attributable to differences in network conditions across agent locations worldwide.}
    \label{fig:amortised_overhead_p_fixed}
    \Description{To avoid warning.}
\end{figure}

Across both configurations, \sagaplus and \saga exhibit qualitatively 
identical amortization behavior: overhead decreases sharply with increasing 
$Q_{\max}$ in all regional configurations, falling below 40~ms by 
$Q_{\max} = 10$ and converging to within 10~ms of one another by 
$Q_{\max} = 15$. Although \sagaplus incurs an additional post-quantum 
cryptographic overhead absent in \saga, this cost is subsumed by the 
latency variation attributable to geographic separation, which constitutes 
the dominant source of variability in both systems. Consequently, the two 
systems are indistinguishable in their convergence characteristics at 
practical session sizes: an Asia-based agent at $Q_{\max} = 15$ incurs 
overhead comparable to a co-located deployment at $Q_{\max} = 5$, 
regardless of whether the \provider or the agent location is varied, 
confirming that a modest increase in session capacity effectively 
neutralizes geographic latency disparity and that the marginal cost of 
post-quantum security in \sagaplus is not discernible under realistic 
network conditions.

\subsection{Task Completion Overhead}

We evaluate task execution overhead for \saga and \sagaplus across the three 
agentic tasks from \saga: Calendar, Email, and Writing, adopting 
the same deployment configuration---initiating agent \caler~ in Europe, 
receiving agent \caled~ in Asia, \provider in US-West, and $Q_{\max} = 10$ 
requests per \ases~ in between \caler~ and \caled.

\begin{table}[t]
\centering
\caption{Comparison of task execution time (in seconds) between \saga and \sagaplus. We  assume that $A_i$, $A_r$, and the \provider are located in Europe, Asia and US-West respectively, with $Q_{\max}=10$ requests per session.}
\label{tab:task_overhead}
\small
\resizebox{\columnwidth}{!}{%
\begin{tabular}{@{}llrrrrr@{}}
\toprule
& & & \multicolumn{2}{c}{\textbf{Networking (s)}} & \multicolumn{2}{c}{\textbf{Overhead (s)}} \\
\cmidrule(lr){4-5}\cmidrule(lr){6-7}
\textbf{Task} & \textbf{LLM Backend} & \textbf{LLM (s)}
& \saga & \sagaplus
& \saga & \sagaplus \\
\midrule
Calendar & GPT-4.1-mini & 35.524  & 0.791 & 1.140 & 0.165 & 0.176 \\
Email    & GPT-4.1      & 17.8  & 1.319 & 1.140 & 0.165 & 0.176 \\
Writing  & Qwen-2.5     & N/A & 1.319 & N/A & 0.165 & N/A \\
\bottomrule
\end{tabular}%
}
\end{table}

Table~\ref{tab:task_overhead} presents a direct comparison between \saga 
and \sagaplus. The measured overhead of \saga{} is $0.165$~s, while \sagaplus{} incurs $0.176$~s, corresponding to a modest $1.07\times$ increase due to post-quantum cryptographic operations. The protocol overhead is dominated by a single \provider round-trip per $Q_{\max}$ requests (Equation~\ref{eq:cproto}). Nevertheless, in both systems the protocol overhead constitutes a negligible fraction of cumulative LLM inference time. For the Calendar and Email tasks, the \sagaplus overhead of $0.176$~s represents less than $0.8\%$ of cumulative LLM inference time in both cases. The Writing task  could not be completed under Qwen-2.5, as the model's tool-call outputs did not conform to the format expected by the agent framework; we 
attribute this to the limitations of an older model generation.


Table~\ref{tab:task_overhead_sagaplus} evaluates \sagaplus across the same three 
tasks using current-generation LLMs: GPT-5.4, GPT-5.4-mini, and 
Qwen-3.5. Across all configurations, \sagaplus overhead remains fixed at 
$0.176$~s, confirming that it is independent of the underlying LLM, as expected from the protocol design. The newer models 
yield substantially reduced LLM inference times and varying networking costs, reflecting the differing number of inter-agent interactions required to complete each task. Qwen-3.5 achieves the lowest inference times on Calendar ($7.275$~s) and 
Email ($7.298$~s) tasks, while GPT-5.4-mini offers competitive performance across all three tasks at reduced cost relative to GPT-5.4. In all 
cases, \sagaplus overhead of $0.176$~s represents well under $2\%$ of total execution time, confirming that the cost of post-quantum access control is effectively subsumed by LLM inference latency across 
all evaluated models and task types.

\begin{table}[t]
\centering
\caption{Task execution time (in seconds) for \sagaplus.  We  assume that \caler, \caled, and \provider are located in Europe, Asia, and US-West respectively, with $Q_{\max}=10$ requests per session.}
\label{tab:task_overhead_sagaplus}
\small
\begin{tabular}{llccc}
\hline
\textbf{Task} & \textbf{\shortstack{LLM\\Backend}} & \textbf{LLM (s)} & \textbf{\shortstack{Network\\(s)}} & \textbf{\shortstack{\sagaplus\\Overhead (s)}} \\
\hline
\multirow{3}{*}{Calendar}
  & GPT-5.4      & 12.507 &0.899 & \multirow{3}{*}{0.176} \\
  & GPT-5.4-mini & 11.367 &0.899 & \\
  & Qwen-3.5     & 7.275  & 1.140 & \\
\hline
\multirow{3}{*}{Email}
  & GPT-5.4      & 14.009 & 0.656 & \multirow{3}{*}{0.176} \\
  & GPT-5.4-mini & 8.848  & 0.656 & \\
  & Qwen-3.5     & 7.298  &0.899 & \\
\hline
\multirow{3}{*}{Writing}
  & GPT-5.4      & 30.498 & 1.140 & \multirow{3}{*}{0.176} \\
  & GPT-5.4-mini & 10.421 &0.899 & \\
  & Qwen-3.5     & 10.652 & 1.140 & \\
\hline
\end{tabular}
\end{table}

\subsection{Performance Evaluation of Provider}
We evaluate the cumulative overhead imposed on the \provider as a function 
of \ases~ lifetime, varying the number of initiating agents. Because 
each new \ases~ requires one contact-resolution operation at the 
\provider (2.96~ms, Table~\ref{tab:crypto_overhead}), a shorter session 
lifetime increases the rate at which this operation must be performed and 
thus the total \provider burden. For $n$ initiating agents each 
refreshing sessions of lifetime $\ell$~minutes, the daily overhead is:
\begin{equation}
  C_{\text{\provider}}(n, \ell)
    = n \cdot \frac{1440}{\ell} \cdot t_{\text{crypto}},
  \label{eq:provider_overhead}
\end{equation}
where $t_{\text{crypto}} = 2.96$~ms is the measured contact-resolution 
cost and $1440$ is the number of minutes in a day. Compared with \saga, 
which incurs a contact-resolution cost of 1.46~ms per-session, 
\sagaplus introduces approximately $2\times$ the per-session \provider 
cost, owing to its additional post-quantum cryptographic operations.

Figure~\ref{fig:provider_overhead_session} plots $C_{\text{\provider}}$ 
for $n \in \{1, 10, 100\}$ agents across session lifetimes ranging from 
one minute to one day. Two structural observations follow directly from 
Equation~\ref{eq:provider_overhead}. First, overhead scales linearly with 
agent count: the 100-agent curve lies exactly one order of magnitude above 
the 10-agent curve at every point. Second, overhead is inversely 
proportional to session lifetime: moving from a 1-minute to a 1-day 
session reduces \provider cost by three orders of magnitude for any fixed 
agent population. Even accounting for the $2\times$ increase in 
per-session cost relative to \saga, the absolute \provider burden remains 
negligible across all evaluated configurations, confirming that the 
post-quantum overhead of \sagaplus does not materially alter the 
scalability characteristics of the system.

\begin{figure}[t]
\centering
\includegraphics[width=0.6\columnwidth]{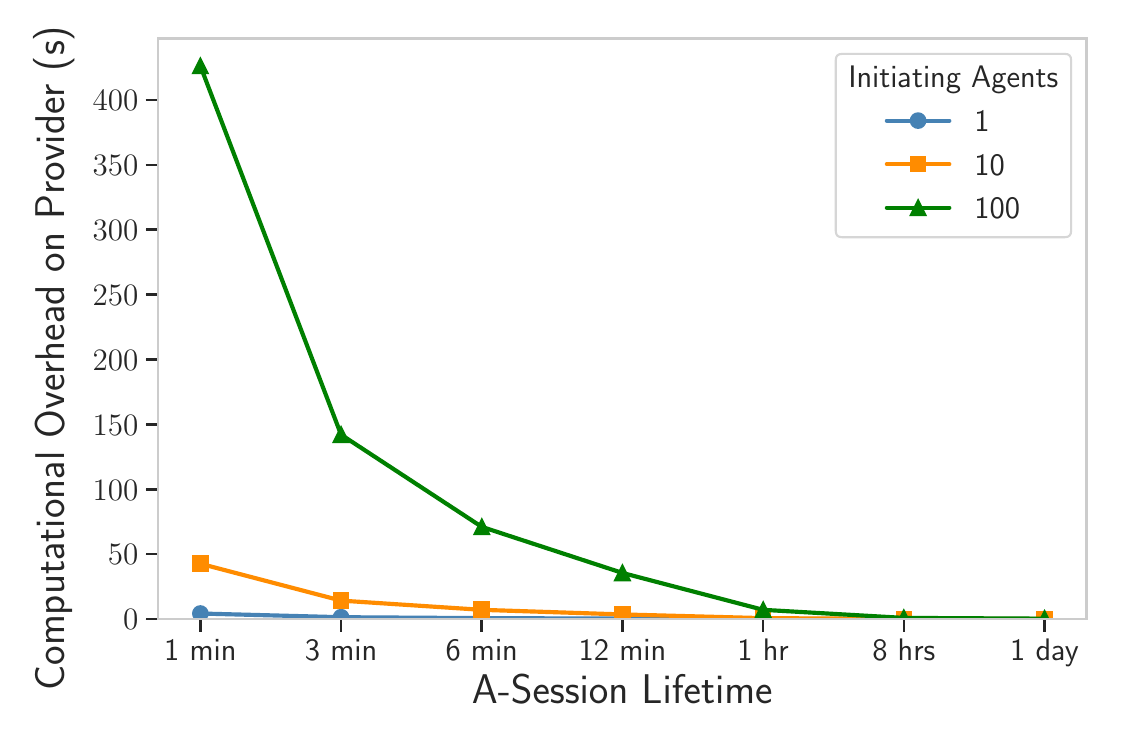}
\caption{Computational overhead on the \provider as a function
of \ases lifetime, for 1, 10, and 100 initiating agents.
}
\label{fig:provider_overhead_session}
\Description{To avoid warning.}
\end{figure}

\subsection{Bandwidth Overhead}

Table~\ref{tab:bandwidth} summarizes the bandwidth requirements across protocol phases for \saga{} and \sagaplus{}. For \saga{}, we assume $N = 100$ one-time keys (OTKs) are registered by the user on behalf of the agent. For both protocols, we consider a setting in which an initiating agent \caler~ issues $m = 10$ requests to a receiving agent \caled, with a maximum of $Q_{\max} = 10$ requests per \ases.

User and agent registration incur one-time costs and therefore do not affect per-task bandwidth. In \saga{}, these costs amount to 0.6~KB for user registration and 10.40~KB for agent registration, whereas in \sagaplus{} they increase to 6.98~KB and 27.28~KB, respectively. The most significant difference arises during session establishment, which constitutes the recurring per-token-cycle overhead: \saga{} requires 2.12~KB per cycle, while \sagaplus{} incurs 88.65~KB, representing a $\sim$42$\times$ increase. The bulk of this overhead is due to the post quantum certificates and signatures being transmitted in between \caler~ and \caled~ during the handshake phase. In contrast, per-request data transfer remains comparable between the two protocols, at 0.59~KB for \saga{} and 0.56~KB for \sagaplus{}. When using \sagaplus for tasks involving more than 2 agents, only the bandwidth required for the agent communication phase grows linearly with the number of agents. Specifically, with $t$ agents the cost scales $(t-1)\times$ the per session bandwidth.

\begin{table}[t]
\centering
\caption{Bandwidth requirements (in KB) per protocol phase for \saga
and \sagaplus. \saga parameters: $N=100$ OTKs,
$m=10$, $Q_{\max}=10$. \sagaplus parameters: $m=10$,
$Q_{\max}=10$.}
\label{tab:bandwidth}
\small
\begin{tabular}{lrr}
\toprule
\textbf{Operation} & \textbf{\saga} & \textbf{\sagaplus} \\
\midrule
User Registration                 & 0.6     & 6.98  \\
Agent Registration                & 10.40  & 27.28 \\
Agent Communication (per session) & 2.12   & 88.65  \\
Data Transfer (per request)       & 0.59    & 0.56 \\
\bottomrule
\end{tabular}
\end{table}

\subsection{One-Agent to Many-Agents Overhead}

We evaluate the overhead borne by an \masterw (\caler) that
coordinates with $t$ receiving agents to complete a task.
For each receiving agent \caled, \caler~ must establish one \ases~
comprising a contact-resolution step and a handshake phase.
We assume all \ases s share the same lifetime $\ell$ and that
each \ases~ comprises $Q_{max}$=10 requests.
The per-session cryptographic cost on \caler~ is:
\begin{equation}
  c_{\text{init}}(t)
    = \underbrace{4.1\,t}_{\text{contact resolution}}
    + \underbrace{4.01\,t}_{\text{handshake}}
    = 8.11\,t \quad \text{(ms)},
  \label{eq:init_cost}
\end{equation}
where costs are taken from Table~\ref{tab:crypto_overhead}.
Because a new session must be established every $\ell$ minutes,
the total daily overhead on \caler~ is:
\begin{equation}
  C_{\text{init}}(t,\,\ell)
    = c_{\text{init}}(t) \cdot \frac{1440}{\ell},
  \label{eq:init_daily}
\end{equation}
with 1440 denoting the number of minutes in a day.
Figure~\ref{fig:oh_initiating_mas} plots $C_{\text{init}}$ for
$t \in \{1, 2, 5, 10, 15\}$ across session lifetimes ranging from
one minute to one day.
The overhead is governed by two independent axes.
At a fixed lifetime, cost scales linearly with $t$: at a 1-minute
session lifetime, overhead ranges from $\sim$12 seconds for $t=1$ to
$\sim$175~seconds for $t=15$, a 15$\times$ ratio consistent with
Equation~\ref{eq:init_cost}.
At a fixed $t$, cost falls inversely with $\ell$: for $t=15$,
extending the session lifetime from 1~minute to 1~hour reduces
daily overhead from $\sim$175~seconds to under 3~seconds, and to
under 0.4~seconds at 8~hours or beyond.
For all values of $t$, the curves converge to negligible overhead
beyond a 1-hour lifetime, confirming that session lifetime is the
dominant design parameter and that \sagaplus\ imposes no meaningful
burden on \caler~ under practical deployment conditions.

\begin{figure}[t]
\centering
\includegraphics[width=0.6\columnwidth]{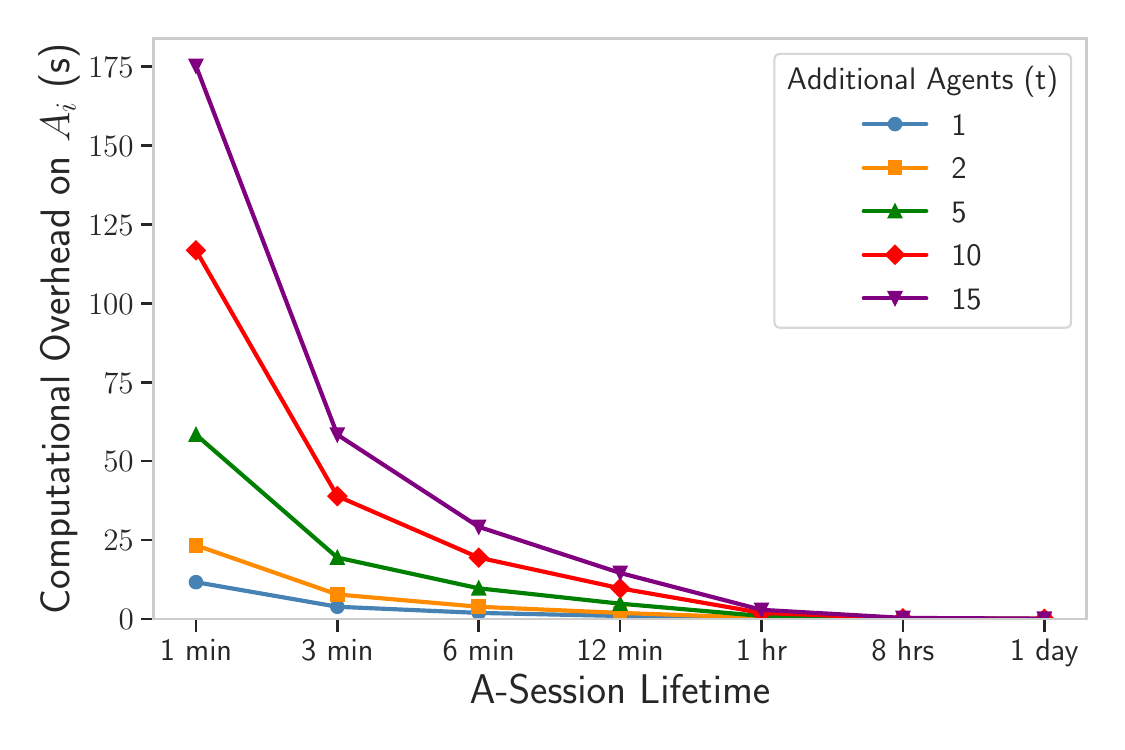}
\caption{Daily computational overhead on \caler~ as a function of
\ases~ lifetime, for $t \in \{1, 2, 5, 10, 15\}$ receiving agents.
Per-session cost follows Equation~\ref{eq:init_cost} using measured
cryptographic costs; daily overhead follows
Equation~\ref{eq:init_daily}.
All \ases s are assumed to share the same lifetime.}
\label{fig:oh_initiating_mas}
\Description{To avoid warnings.}
\end{figure}

\subsection{Multi-Agent Task Completion Overhead}

\begin{table}[t]
\centering
\caption{Task execution time (seconds) for \sagaplus with 3 agents. \caler~ is located in Europe, 2 \caled-s in Asia, and \provider in US-West, with $Q_{\max}=10$ requests per session.}
\label{tab:task_overhead_sagaplus_ma}
\small
\begin{tabular}{llccc}
\hline
\textbf{Task} & \textbf{\shortstack{LLM}} & \textbf{LLM(s)} & \textbf{\shortstack{Network(s)}} & \textbf{\shortstack{\sagaplus(s)}} \\
\hline
\multirow{3}{*}{Calendar}
  & GPT-5.4      & 50.606 & 1.708 & \multirow{3}{*}{0.336} \\
  & GPT-5.4-mini & 42.203 & 1.950 & \\
  & Qwen-3.5     & 23.487  & 1.950 & \\
\hline
\multirow{3}{*}{Email}
  & GPT-5.4      & 47.669 & 1.223 & \multirow{3}{*}{0.336} \\
  & GPT-5.4-mini & 34.261  & 1.465 & \\
  & Qwen-3.5     & 31.194  & 1.708 & \\
\hline
\multirow{3}{*}{Writing}
  & GPT-5.4      & 66.731 & 1.223 & \multirow{3}{*}{0.336} \\
  & GPT-5.4-mini & 10.421 & 1.223 & \\
  & Qwen-3.5     & 36.471 & 2.192 & \\
\hline
\end{tabular}
\end{table}

We evaluate task execution overhead for \sagaplus across the three agentic tasks---Calendar, Email, and Writing---extending the two-agent deployment of \sagaplus to a multi-agent setting. For both \sagaplus and \saga the \masterw (\caler) is in Europe, receiving agents \caled~ in Asia, and the Provider in US-West, with $Q_{\max} = 10$ requests per \ases~ between \caler~ and \caled. Each task is remodeled for the multi-agent setting, comprising one \masterw agent and two receiving agents executing a sequential, static workflow, where the total overhead encompasses task completion across all three agents. As reported in Table \ref{tab:task_overhead_sagaplus_ma}, \sagaplus incurs a constant overhead of 0.336 s across all tasks and LLM backends---identical to that observed in the two-agent case (Table 4, 0.176 s)---confirming that the protocol cost scales predictably with the number of agents. Crucially, this overhead remains well under 4\% of the LLM inference time across all configurations (ranging from 10.421 s to 66.731 s), demonstrating that \sagaplus's overhead is effectively subsumed by model inference latency, irrespective of agent topology or task type.
\section{Related Work}
\label{sec:relwork}


\textbf{Designs for Agent Governance. }
Recent work has moved beyond static permissions toward dynamic, machine-verifiable delegation. Chen \citep{aith} introduces AITH, a post-quantum continuous delegation protocol that replaces discrete signing events with sub-microsecond boundary checks, allowing agents to operate autonomously within cryptographically attested ``standing instructions''. Addressing the non-deterministic nature of agent behavior, Bhardwaj \citep{abc} proposes Agent Behavioral Contracts (ABC), which apply Design-by-Contract principles to enforce preconditions, invariants, and recovery mechanisms at runtime. Kaptein et al. \citep{runtimegovernanceaiagents} further formalize this via ``Policies on Paths,'' arguing that governance must treat the agent's execution trajectory as a deterministic function of its identity and prior actions to prevent information barrier breaches.

\noindent\textbf{Inter-Agent Protocol Implementations. }
The transition toward an ``Agentic Web'' has spurred the development of standardized communication layers. While early implementations like Agent Protocol from langchain lacked robust security, the Agent Communication Protocol (ACP) \citep{acp} now provides a federated orchestration model integrating decentralized identifiers (DIDs) and automated service-level agreements. Google's Agent-to-Agent (A2A) protocol \citep{A2A} facilitates discovery via ``Agent Cards,'' though it lacks the centralized mediation found in the SAGA architecture. Furthermore, Anbiaee et al. \citep{securitythreatmodelingemerging} provide a systematic security analysis of emerging protocols (MCP, A2A, Agora, and ANP), identifying design-induced risk surfaces such as token scope escalation and naming collisions.

\noindent\textbf{Attacks on Multi-Agent Systems. }
Several works examine adversarial propagation in multi-agent communication \cite{gu2024agent,he2025red,yu2024infecting,lee2024prompt,amayuelas2024multiagent}, where rogue agents can propagate malicious outputs via interactions with other agents. Triedman et al. \citep{cfh} demonstrate Control-Flow Hijacking (CFH) attacks, where orchestrators are manipulated via indirect prompt injection to misroute sub-tasks to malicious agents. Chen et al. \citep{adapam} introduce AdapAM, a framework for black-box adversarial attacks that utilizes adaptive selection to induce malicious actions in victim agents with high stealth. 
Wang et al. \citep{killchaincanaries} introduces the ``Kill-Chain Canary'' methodology, which tracks cryptographic tokens through discrete stages (Exposed, Persisted, Relayed, Executed) to localize where defenses fail. 
\section{Conclusion}
\label{sec:conclusion}
In this paper we presented \sagaplus, a framework for policy definition
and enforcement of multi-agentic AI systems using novel highly
efficient quantum-resistant cryptographic protocols with proved
security guarantees. \sagaplus (i) allows users to define rich communication and access control policy budgets; (ii) enforces such policies using post-quantum cryptographic primitives; (iii) supports session-based enforcement of policies for
agent-to-agent and one-to-many agent sessions; and (iv) provides
accountability of agents to their users in the form of \tmsg attribution. We formally model and prove correctness and security
of the system using the UC framework.
Our evaluation results demonstrate that \sagaplus is a viable solution for agentic AI governance systems.



\bibliographystyle{ACM-Reference-Format}
\bibliography{bib/arch,bib/pg_f,bib/crypto,bib/tushin}

\appendix 




\section{Ethical Considerations}
Our paper is not an attack paper, it does not use any public dataset, or human data collection, so we believe that there are no ethical concerns.



\section{Notations}
\label{App:notations}

We present the notations used throughout the paper in Table \ref{tab:notation}.

\begin{table}
    \centering
    \caption{Notations}
    \label{tab:notation}
    \small
    \begin{tabular}{|l|l|}
    \hline
    \textbf{Symbol} & \textbf{Description} \\
    \hline
    \hline
    \multicolumn{2}{|c|}{{\em Entities}}\\
    \hline
     $A_I$ & Initiator agent\\
    \hline
     $A_R$ & Responding agent\\
    \hline
     $A^{*}_{I}$ & Orchestrator agent \\
    \hline 
    $U$ & User \\
    \hline 
    $TA$ & Provider \\
    \hline
    $CA$ & Certificate authority \\
    \hline
    \hline
    \multicolumn{2}{|c|}{{\em \sagaplus \ Protocols}}\\
    \hline
     $Cert$ & Certificate issued by $CA$\\
     \hline
     $uid_U$ & User $U$ identifier \\
     \hline
     $aid_A$ & Agent $A$ identifier\\
     \hline
     $ED_A$ & Agent $A$ endpoint descriptor\\
     \hline 
     $M_A$ & Metadata of agent $A$\\
     \hline
     $(pk^{tls}_A, sk^{tls}_A)$ & Agent $A$ public and private PQ-TLS credentials\\
     \hline 
     $(pk_A, sk_A)$ & Agent $A$ public and private keys for hash based signature\\
     \hline
     $CP_A$ & Contact policy of Agent $A$ (ACL and \ases\ budget)\\
     \hline
     $\sigma^{Y}_{X}$ & Signature generated by $Y$ with type $X$ \\
     \hline
     $NextTok$ & Next message count token for \tmsg \ budget \\ 
     \hline
     $DU$ & Provider's user registry\\
     \hline
     $DA$ & Provider's agent registry \\
     \hline
     $HS$ & Hash-based signature\\
     \hline
     $req$ & Request\\
     \hline
     $res$ & Response\\
     \hline
     $s_0, \rho_0$ & Random seeds for \tmsg \ count tokens \\
     \hline
     $Mroot$ & Merkle root \\
     \hline
     $n, n'$ & hash chain length\\
     \hline
     \hline
     \multicolumn{2}{|c|}{{\em Security Model and Analysis}}\\
    \hline
    $\mathsf{Sim}$ & Simulator \\
    \hline
     $\mathsf{Sim}^A$ & Simulator with oracle access to $A$ \\
    \hline
    $\mathcal{A}$ & Adversary \\
    \hline
    $\mathcal{Z}$ & Environment\\
    \hline
     $\mathcal{F}$ & Ideal functionality\\
     \hline
     $\mathcal{G}$ & Global ideal functionality\\
     \hline
     $\pi$ & Protocol \\
     \hline
     $\pi^{\mathcal{F}_1, \mathcal{F}_2, \cdots}$ & Hybrid protocol that uses $\mathcal{F}_1, \mathcal{F}_2, \cdots$ as subroutines\\
     \hline
     $\tau$ & Current time read from $\mathcal{G}_{\mathsf{clk}}$\\
     \hline
     $T_{exp}$ & Expiry time \\
     \hline
     $Q$ & \tmsg budget \\
     \hline
     $\Delta$ & Time budget\\
     \hline
     $sid$ & Session identifier \\
     \hline
     $IDI$ & Initiator's identity (in $\mathcal{F}_{\mathsf{SCS}}$) \\
     \hline
     $P$, $P'$ & Party (in $\mathcal{F}_{\mathsf{SCS}}$) \\
     \hline
     $I$ & Initiator \\
     \hline
     $R$ & Responser  \\
     \hline
     $L$ & Set representation \\
     \hline 
     $Apolicy$ & Policy of agent consisting of \tmsg and time budgets\\
     \hline
     $\ell()$ & Leakage function\\
     \hline
     $\ell$ & Leakage value\\
     \hline
     $ctr$ & Counter\\
     \hline
     $m$ & \tmsg \\
     \hline
     $attr_A$ & Attribute of agent $A$\\
     \hline
     $Pwd$ & Password \\
     \hline 
     $tsk$, $subtsk$ & Task and subtask, respectively\\
     \hline
     $Budget$ & Contact budget \\
     \hline
     $ExeReq()$ & The request execution function\\ 
     \hline 
     $ExeRes()$ & The response execution function\\ 
     \hline
    \end{tabular}
\end{table}

\section{Cryptographic Enforcement of Policies}
\label{Appendix:policyEnforce}

Below we define each of the three policies of CP, RCP, and ICP and give the mechanisms that are used to enforce each of them in our protocol.


\subsection{Contact Policy $CP$} Users define and administer contact policies for their registered agents. The contact policy $CP$ states that which agents can establish a connection with whom and for how many times (budget). $CP$ consists of a set of declarative rules such as ("receive, *@company.com:email-agent", 10) which allows an email handling agent from a specified domain to initiate contact at most 10 times, or ("send, *@company.com:email-agent", 20) which allows the agent to initiate an A-session with an email handling agent for 20 times.  If multiple rules match, the one with the highest specificity is selected. The number of connections between an initiating agent $A_I$ and responding agent $A_R$, i.e., the budget, is defined as below:

$Budget(aid_{A_R}, aid_{A_I})= \Bigg \{ \begin{tabular}{l l} $-1$ & if $R=\emptyset$ \\ $B(r^*)$  & if $R\neq \emptyset$ \end{tabular} $

Which state that $r^*$ is the most specific rule among all rules $R \in CP_{A_R}$ when $A_R$ and $A_I$ are establishing an A-session, and $B(r^*)$ corresponds to the rule $r^*$. $R=\emptyset$ allows the user to distinguish between no match in policy and an expired budget.

\textbf{Enforcing the contact policy $CP$:} when an initiating agent $A_I$ queries the provider to contact $A_R$, the provider first verifies that the agents participating in the A-session satisfy the contact policy $CP_{A_R}\cap CP_{A_I}$, which means that $CP_{A_R}\cap CP_{A_I}\neq \emptyset$.  If this is the first time $A_I$ and $A_R$ connects, then the provider creates a counter $Counter[aid_{A_R}][aid_{A_I}]$ 
to keep track of the number of contacts and initializes it with $Min\{Budget(aid_{A_R},aid_{A_I}), Budget(aid_{A_R},aid_{A_I})\}$.

If the policy check succeeds, the provider returns a signature $\sigma^{TA}_{ac}$ to agent $A_I$ on the recipient's data and decreases the counter by one: $\sigma^{TA}_{ac}=HS.Sign_{sk_{TA}}(\nu, Cert^{ID}_{U_R},aid_{A_R},M_{A_R}, \sigma^{U_R}_{ID}, \sigma^{U_R}_{A_R})$, where $\nu$ is a nonce used to ensure token freshness. We assume the receiver agent will keep track of the nonce. 
If $Budget$ is negative, or exhausted, then provider sends a failure message and terminates the connection.
The user of agent $A_R$ can update the contact policy and dedicate a new budget at any time. 

\remove{
\begin{figure}
    \centering
    \includegraphics[width=0.5\linewidth]{Figures/CP.jpg}
    \caption{Enforcing contact policy CP by the provider}
    \label{CP}
\end{figure}
}

\subsection{Responder Contact Policy $RCP$} 

We consider RCP is defined by the user for a responding agent $A_R$, and consists of the tuples of the form $(A_R, A_I, Q_{R,I}, \Delta_{R,I})$ which states that in an A-session between $A_I$ and $A_R$,  the maximum number of requests that $A_R$ can receive is $Q_{R,I}$, and the maximum duration of the A-session is $\Delta_{R,I}$. 

\textbf{Enforcing the interaction budget in $RCP$:} We assume the responding agent can be corrupted and define a scheme which allows the responding agent to keep the total budget counting correctly, while it ensures the budget has indeed been determined by the user, and the initiating agent can \textit{efficiently} verify when the responding agent exceed the maximum allowed budget.

\begin{itemize}

    \item When $A_I$ establishes an A-session with $A_R$, the responder agent 
    $A_R$ derives a random seed:
    \[
        s_0 = \mathsf{PRF}(\mathit{sid},\; \mathit{aid}_{A_I},\; 
        \mathit{aid}_{A_R},\; \mathit{addr})
    \]
    where $\mathit{addr}$ is the chain index for this A-session (set to $0$ 
    for the first chain). It then builds a \emph{personalized hash chain} of 
    length $n = Q_{R,I}$:
    \[
        s_j = H(s_{j-1},\; j,\; \mathit{sid},\; \mathit{aid}_{A_I}), 
        \qquad j = 1, \ldots, n
    \]
    Binding $\mathit{aid}_{A_I}$ into each step prevents a corrupted $A_R$ 
    from reusing the chain with another agent to exceed $Q_{R,I}$; binding 
    $\mathit{sid}$ prevents reuse across A-sessions; and the index $j$ allows 
    $A_I$ to track which hash values $A_R$ opens and detect any point of 
    cheating. These values serve as \emph{hash-based message count tokens} 
    that enforce the interaction budget. Agent $A_R$ then sends the terminal 
    value $s_n$ to the user $U_R$.

    \item User $U_R$ signs the terminal value and chain length:
    \[
        \sigma^{U_R}_{\mathit{RCP}} = \mathsf{HS.Sign}_{sk_{U_R}}(s_n,\; 
        Q_{R,I})
    \]
    and sends $\sigma^{U_R}_{\mathit{RCP}}$ to $A_R$.

    \item Agent $A_R$ forwards $s_n$, $s_{n-1}$, $Q_{R,I}$, $\Delta_{R,I}$, 
    and $\sigma^{U_R}_{\mathit{RCP}}$ to the initiator $A_I$.

    \item Agent $A_I$ verifies the user's signature:
    \[
        \mathsf{HS.Verify}_{pk_{U_R}}\!\left(
            \langle\, s_n,\; Q_{R,I}\, \rangle,\; 
            \sigma^{U_R}_{\mathit{RCP}}
        \right) = 1
    \]
    and checks that $Q_{R,I} = n$ and $H(s_{n-1},\, n,\, \mathit{sid},\, 
    \mathit{aid}_{A_I}) = s_n$. If all checks pass, $A_I$ is assured that 
    $A_R$ has correctly initialized the interaction budget, and initializes a 
    counter $\mathit{ctr}_{\mathit{RCP}} = Q_{R,I}$, decrementing it by one. 
    In subsequent rounds, each time $A_I$ receives a token it verifies:
    \[
        s_i = H(s_{i-1},\; i,\; \mathit{sid},\; \mathit{aid}_{A_I})
    \]
    and decrements $\mathit{ctr}_{\mathit{RCP}}$. When the counter reaches 
    zero, $A_I$ performs a final check:
    \[
        s_0 = \mathsf{PRF}(\mathit{sid},\; \mathit{aid}_{A_I},\; 
        \mathit{aid}_{A_R},\; \mathit{addr})
    \]
    At this point $A_R$ must cease accepting further requests; an honest $A_I$ 
    will stop sending them regardless, though a corrupted $A_R$ may not comply.

\end{itemize}

We guarantee that if both agents that are involved in an A-session are honest or only one of them is honest then the interaction budget will be correctly enforced. 
Note that we do not ensure any guarantee for the case when both agents are corrupted.

\remove{
\begin{figure}
    \centering
    \includegraphics[width=0.5\linewidth]{Figures/RCP.jpg}
    \caption{Enforcing the interaction budget in the receiver contact policy RCP}
    \label{RCP}
\end{figure}
}


\subsection{Initiator Contact Policy $ICP$} We assume that the user chooses $Q_{tot}$, that is, the maximum number of requests for an initiator agent that can send to other agents toward fulfilling a task, and $\Delta_{tot}$ that is the maximum duration of time that an initiator agent can interact with other agents to fulfill a given task.  These two values are stored as the Initiator contact policy $ICP$, that is, $ICP=(Q_{tot},\Delta_{tot})$.

%

\textbf{Enforcing the interaction budget in $ICP$ (static workflow).} We assume that the \masterw \ agent $A^{*}_{I}$ acts as an initiator and orchestrates the communication with other agents $A_1,A_2, \cdots, A_t$ through sessions $S_1, S_2, ..., S_t$. We design a scheme that enforces the total budget $Q_{tot}$ throughout all these sessions that  allows $A^{*}_{I}$ to keep the total budget counting correctly while ensuring that the budget has indeed been determined by user $U^*$. Also, the responding agents
$A_1$, $A_2$, $A_3$, ....,  can \textit{efficiently} identify when $A^{*}_{I}$ exceeds the budget without knowing the history of interactions of $A^{*}_{I}$ with other agents and without interacting with other agents (just using their local views). We require that the communication should stay almost constant (independent of the request number)  for each pair of agents.

We note that a single hash chain of length $Q_{tot}$ will not work for $ICP$ since the agents need to know the history of the consumed hash tokens in order to verify the token presentations. This will also lead to a communication cost that grows linearly with the number of requests. Therefore, this naive approach will not work. To circumvent around this, we first consider that the agents follow an static workflow where the \masterw \ agent knows which agents it has to contact to fulfill the task as soon as it receives the task. We design the hash-based tokens to enforce the interaction budget for an static workflow as below. 

\begin{itemize}

    \item The user $U^*$ determines the maximum interaction budget $Q_{tot}$ 
    for $\mathit{ICP}$.

    \item Upon receiving a task from $U^*$, agent $A^*_I$ processes it and 
    obtains the workflow, identifying the list of agents $A_1, \ldots, A_t$ 
    to contact. It then creates $m$ personalized hash chains of length $n$, 
    where each chain is seeded as:
    \[
        s_0 = \mathsf{PRF}(\mathit{sid},\; \mathit{aid}_{A^*_I},\; 
        \mathit{aid}_{A_i},\; \mathit{addr})
    \]
    and extended as $s_j = H(s_{j-1},\, j,\, \mathit{sid},\, 
    \mathit{aid}_{A_i})$ for $j = 1, \ldots, n$. The parameters satisfy 
    $Q_{tot} = n \times m$ and $m \geq t$, ensuring the total chain length 
    does not exceed $Q_{tot}$ and that at least one chain is available per 
    A-session; $A^*_I$ may allocate two or more chains to a single A-session. 
    Agent $A^*_I$ then sends the terminal values $\{s_n, s'_n, \ldots, 
    s^\ell_n\}$ of all personalized hash chains to user $U^*$.

    \item The user $U^*$ constructs a Merkle tree over the terminal values and 
    signs the root together with the chain length:
    \[
        \sigma^{U^*}_{\mathit{ICP}} = \mathsf{HS.Sign}_{sk_{U^*}}
        (\mathit{Mroot},\; n),
    \]
    and sends the Merkle tree and signature $\sigma^{U^*}_{\mathit{ICP}}$ to 
    $A^*_I$.

    \item When $A^*_I$ initiates an A-session with agent $A_i$, $i \in [1,t]$, 
    it sends to $A_i$:
    \begin{itemize}
        \item the Merkle root $\mathit{Mroot}$ and user signature 
        $\sigma^{U^*}_{\mathit{ICP}}$
        \item a Merkle proof that the hash chain root is included in the tree 
        signed by $U^*$
        \item the chain openings $s_n$ and $s_{n-1}$, where 
        $s_n = H(s_{n-1},\, n,\, \mathit{sid},\, \mathit{aid}_{A_i})$
    \end{itemize}

    \item Agent $A_i$ verifies the following:
    \begin{itemize}
        \item the user's signature: $\mathsf{HS.Verify}_{pk_{U^*}}
        (\langle \mathit{Mroot},\, n \rangle,\, 
        \sigma^{U^*}_{\mathit{ICP}}) = 1$
        \item the Merkle proof is valid
        \item the hash step: $s_n = H(s_{n-1},\, n,\, \mathit{sid},\, 
        \mathit{aid}_{A_i})$
    \end{itemize}
    If all checks pass, $A_i$ is assured that $U^*$ has authorized the use of 
    a personalized hash chain of length $n$ for agent $A_i$ to fulfill the 
    task. Agent $A_i$ then initializes a counter $\mathit{ctr}_{\mathit{ICP}} 
    = n$ and decrements it by one.

\end{itemize}
   
    In the next rounds, the initiator agent just sends a hash token from the chain $s_j$, and the responding agent $A_i$ verifies that $s_j=H(s_{j-1},j, sid, aid_{A_i})$, and then decreases the counter $ctr_{ICP}$. When the counter $ctr_{ICP}$ reaches 0 it checks that $s_0=PRF(sid,aid_{A^{*}_{I}}, aid_{A_i}, addr)$. If the agent wants to continue interactions and use more hash chains, it can increment $addr$ by 1, and repeat the last two steps mentioned above to start using a new personalized hash chain.

\remove{
\begin{figure}
    \centering
    \includegraphics[width=0.5\linewidth]{Figures/ICP.jpg}
    \caption{Enforcing the interaction budget in the initiator contact policy ICP (static workflow)}
    \label{ICP}
\end{figure}
}

 \textbf{Enforcing the interaction budget in $ICP$ (dynamic workflow).}
 For a dynamic workflow, the \masterw \ agent cannot know the list of agents who wants to contact with them ahead of time. Therefore, it cannot allocate the hash chains to the responding agents and create personalized hash chains from beginning. To solve this issue, the naive approach will be to let the \masterw \ agent to create the chains as soon as it finds with whom it has to contact. For example, if it initially learns that it has to contact with $A_1$ and $A_2$, it generates two or more hash chains of length $n$, and sends the root of these personalized hash chain to user to sign. The user generates a Merkle tree on top of the hash chain terminal values and signs the Merkle root, and sends the Merkle tree and the signature to agent $A^{*}_{I}$. Every time the agent creates a new hash chain,  the user updates its Merkle tree and its signature. 
 This approach, however, is not efficient and requires too many interactions between the \masterw \ agent and the user. 
 
 Below, we give a more efficient scheme that allows the ICP to be enforced for a dynamic workflow (plan) with minimal interactions with the user. The idea is that user creates $m$ one time signature (OTS) key pairs and generate a Merkle tree on the OTS public keys. Then, it signs the root of the Merkle tree, and sends the Merkle tree to the \masterw \ agent together with delegating the OTS private keys to the \masterw \ agent. On receiving these keys, the \masterw \ agent can generate personalized hash chains for any agent $A_i$ and it can plug in the hash chain to the Merkle tree by signing the terminal values of the hash chain using the delegated private keys. See below for more details. 

 {\em One time signature (OTS)} is a digital signature scheme that allows using a private and public key pair to sign a message only once. The examples of OTS are Lamport signature scheme \cite{lamport1979constructing}, and WOTS+ signature scheme \cite{buchmann2013security}. Both of these OTS schemes are hash-based stateful signatures and need secure state management. We consider that OTS  consists of: (i) $OTS.Gen(\lambda)$ which takes the security parameter $\lambda$ as input and outputs a private and public key pair.   (ii) $OTS.Sign_{Sk}(m)$, which takes a message $m$ and the OTS private key $Sk$ and outputs a signature $\sigma$. (iii) $OTS.Verify_{Pk}(<m>, \sigma)$ which takes a message $m$, an OTS signature $\sigma$ and  a public key $Pk$ and outputs 1 if the signature is valid, and 0 otherwise.

\remove{
 \begin{figure}
    \centering
    \includegraphics[width=0.5\linewidth]{Figures/ICP-Dynamic.jpg}
    \caption{Enforcing the interaction budget in the initiator contact policy ICP (dynamic workflow)}
    \label{ICP-Dynamic}
\end{figure}
}

\begin{itemize}

    \item The user $U^*$ determines the maximum message budget $Q_{tot}$ for 
    $\mathit{ICP}$, and creates $m$ OTS key pairs such that $Q_{tot} = n 
    \times m$. It constructs a Merkle tree over the OTS public keys and signs 
    the Merkle root along with the hash chain length:
    \[
        \sigma^{U^*}_{\mathit{ICP}} = \mathsf{HS.Sign}_{sk_{U^*}}
        (\mathit{Mroot},\; n).
    \]
    It then sends the Merkle tree and signature to the master agent $A^*_I$, 
    and delegates the OTS private keys to $A^*_I$.

    \item Upon receiving a task from $U^*$, agent $A^*_I$ processes it and 
    obtains the workflow, which identifies the first agent $A_i$, $i \in 
    [1,t]$, to contact. It derives a personalized hash chain of length $n$ 
    using the seed:
    \[
        s_0 = \mathsf{PRF}(\mathit{sid},\; \mathit{aid}_{A^*_I},\; 
        \mathit{aid}_{A_i},\; \mathit{addr})
    \]
    and extends it as $s_j = H(s_{j-1},\, j,\, \mathit{sid},\, 
    \mathit{aid}_{A_i})$ for $j = 1, \ldots, n$. It then signs the chain root 
    using one of the delegated OTS private keys:
    \[
        \sigma^{A^*_I}_{\mathit{OTS}} = 
        \mathsf{OTS.Sign}_{\mathit{OTS.sk}}(s_n)
    \]
    Secure state management is assumed to prevent reuse of any OTS key.

    \item When $A^*_I$ initiates an A-session with $A_i$, it sends the 
    following to $A_i$:
    \begin{itemize}
        \item the Merkle root $\mathit{Mroot}$ and user signature 
        $\sigma^{U^*}_{\mathit{ICP}}$
        \item the OTS public key used to sign $s_n$, together with a Merkle 
        proof that this key is included in the tree signed by $U^*$
        \item the OTS signature $\sigma^{A^*_I}_{\mathit{OTS}}$ on $s_n$,
        \item the chain openings $s_n$ and $s_{n-1}$,
        where 
        $s_n = H(s_{n-1},\, n,\, \mathit{sid},\, \mathit{aid}_{A_i})$
    \end{itemize}

    \item Agent $A_i$ verifies the following:
    \begin{itemize}
        \item the user's signature: $\mathsf{HS.Verify}_{pk_{U^*}}
        (\langle \mathit{Mroot},\, n \rangle,\, 
        \sigma^{U^*}_{\mathit{ICP}}) = 1$,
        \item the OTS signature: $\mathsf{OTS.Verify}_{\mathit{OTS.pk}}
        (\langle s_n \rangle,\, \sigma^{A^*_I}_{\mathit{OTS}}) = 1$,
        \item the Merkle proof is valid,
        \item the hash step: $s_n = H(s_{n-1},\, n,\, \mathit{sid},\, 
        \mathit{aid}_{A_i})$.
    \end{itemize}
    If all checks pass, $A_i$ is assured that $U^*$ has authorized a hash 
    chain of length $n$ for agent $A_i$ to fulfill the task. Agent $A_i$ then 
    initializes a counter $\mathit{ctr}_{\mathit{ICP}} = n$ and decrements it 
    by one.

\end{itemize}
   
  In the next rounds, the initiator agent just sends a hash token from the chain $s_j$, and the responding agent $A_i$ verifies that $s_j=H(s_{j-1},j, sid, aid_{A_i})$, and then decreases the counter $ctr_{ICP}$. When the counter $ctr_{ICP}$ reaches 0 it checks that $s_0=PRF(sid,aid_{A^{*}_{I}}, aid_{A_i}, addr)$. If the agent wants to continue interactions and use more hash chains, it can increment $addr$ by 1, and repeat the last two steps mentioned above to start a new personalized hash chain.

 \textbf{ICP in the two-agent case.} In the two-agent scenario, ICP becomes similar to RCP, and we deal with it in the same way using the same message count tokens used in RCP description. Therefore, the agent creates the hash chains of the length determined by the policy, and the user signs the hash terminal value. Then, agent sends the message count token and the signature to the responding agent.

\section{User-Agent Interaction for Multi-agent Setting}
\label{Appendix:CSessionUserAgent}

We assume the communication between the user and the master agent is local. In the agent registration subprotocol, the user has determined the ICP and RCP policies for all agents. For ICP which defines the maximum interaction budget for total requests, user has determined $m$ and $n'$ such that $Q_{tot}=n' \times m$, and has shared $m$ and $n'$ with the master agent. When the user allocates the task to the master agent $A^{*}_{I}$, the following subprotocol is run:

\begin{enumerate}

    \item \textbf{Giving the task to the agent.} The user $U^*$ determines 
    the task $\mathit{tsk}$ and sends it to the master agent $A^*_I$.

    \item \textbf{Executing the task and obtaining the workflow.} The master 
    agent $A^*_I$ processes the task and derives the workflow, which specifies 
    the list of agents $\{A_1, \ldots, A_t\}$ that $A^*_I$ must contact to 
    fulfill the task.

    \item \textbf{Creating the hash chain tokens.} Agent $A^*_I$ generates 
    $m$ personalized hash chains. For each chain $i \in [0, m-1]$, the seed 
    is derived as:
    \[
        \rho^i_0 = \mathsf{PRF}(\mathit{sid},\; \mathit{aid}_{A^*_I},\; 
        \mathit{aid}_{A_j},\; \mathit{addr}),
    \]
    and the chain is extended as:
    \[
        \rho^i_k = H(\rho^i_{k-1},\; k,\; \mathit{sid}_j,\; 
        \mathit{aid}_{A_j}), \qquad k = 1, \ldots, n',
    \]
    where $\mathit{aid}_{A_j}$ is the identity of the agent assigned to chain 
    $i$ according to the workflow, such that each agent receives at least one 
    chain. Agent $A^*_I$ then sends the hash roots 
    $\{\rho^i_{n'}\}_{i \in [0,m-1],\, j \in [1,t]}$ to the user $U^*$.

    \item \textbf{Creating the Merkle tree and signing the root.} The user 
    $U^*$ constructs a Merkle tree over the hash roots 
    $\{\rho^i_{n'}\}_{i \in [0,m-1],\, j \in [1,t]}$ received from $A^*_I$, 
    obtaining the root $\mathit{Mroot}$. The user signs the root as:
    \[
        \sigma^{U^*}_{\mathit{ICP}} = \mathsf{HS.Sign}_{sk_{U^*}}
        (\mathit{Mroot},\; n')
    \]
    The user then sends the Merkle tree and signature 
    $\sigma^{U^*}_{\mathit{ICP}}$ to $A^*_I$.

    \item \textbf{Verifying the signature.} Agent $A^*_I$ verifies 
    $\sigma^{U^*}_{\mathit{ICP}}$ by checking:
    \[
        \mathsf{HS.Verify}_{pk_{U^*}}\!\left(
            \langle\, \mathit{Mroot},\; n'\, \rangle,\; 
            \sigma^{U^*}_{\mathit{ICP}}
        \right) = 1
    \]
    and stores the Merkle tree upon success.

\end{enumerate}

 For the dynamic workflow, note that the Merkle tree is constructed over the OTS keys and the protocol have been described in Section \ref{Appendix:policyEnforce}.

\section{UC Model and Ideal Functionalities}
\label{Appendix:ideal}

Below we give an overview of the UC approach and the description of each ideal functionality that is needed in our analysis.

\subsection{UC Approach}
In the UC approach, the desired behavior of a system is specified by an 
\emph{ideal functionality} $\mathcal{F}$, which describes the working of the system in  ideal world. Security of a protocol $\pi$ is defined by comparing 
the real-world execution of $\pi$ in the presence of an adversary $\mathcal{A}$ 
with the ideal-world execution of $\mathcal{F}$, where the adversary's behavior 
is emulated by a \emph{simulator} $\mathsf{Sim}$. 
Protocol participants,  $P_1$, $P_2 \cdots P_\ell$, are modeled by polynomially-bounded Interactive Turing Machines (ITM). 
A special ITM $\mathcal{Z}$, called the \emph{environment}, models the influence
of the surrounding protocol execution on $\pi$. The environment provides inputs
to the participants, receives the outputs of honest parties, has access to the corrupted parties through $\mathcal{A}$, and may interact arbitrarily with
$\mathcal{A}$.
{\em The protocol $\pi$ is UC-secure }if no environment $\mathcal{Z}$ can distinguish between the real-world and ideal-world execution of the protocol.

\subsection{Global Clock Ideal Functionality $\mathcal{G}_{\mathsf{clk}}$}
We follow the clock functionality of \cite{canetti2017clock} that provides a universal reference (physical) clock. When it is queried it provides an abstract notion of the time represented by $\tau$. The clock is monotonic and only the environment can increment it; also the simulator cannot forge this reference time.

\begin{figure}[H]
    \centering
    \begin{center}
        \fbox{\begin{minipage}{\linewidth}
        {\centering \textbf{Global clock ideal functionality $\mathcal{G}_{\mathsf{clk}}$}\\}

        {maintains an integer $\tau$ corresponding to the reference time. When created, it initializes to $\tau=0$.}
        
        \begin{itemize}
            \item 
            Upon receiving message $(ClockUpdate,sid)$ from $\mathcal{Z}$, update $\tau=\tau+1$ and send $(ClockUpdate, sid, ok)$ to $\mathcal{Z}$. Ignore any clock update message sent by an other party.  
            \item 
            Upon receiving $(ClockRead,sid)$  from any entity return $(clock\text{-}read,\;sid,\;\tau)$ to the requester.
        \end{itemize}
        
        \end{minipage}
        }
    \end{center}
    \caption{Global clock ideal functionality $\mathcal{G}_{\mathsf{clk}}$ \cite{canetti2017clock}}
    \label{Fclock}
    \Description{To avoid warning.}
\end{figure}

\subsection{Secure Communication Session Ideal Functionality $\mathcal{F}_{\mathsf{SCS}}$ }
This functionality captures the security requirements of the TLS channel and allows a secure communication between entities in a single protocol instance. TLS consists of two phases: handshake and message transmission. The handshake protocol aims at securely sharing uniformly distributed session keys, and the message transmission provides authenticated encryption of session messages. The following ideal functionality $\mathcal{F}_{\mathsf{SCS}}$ \cite{gajek2008universally} consists of session establishment and message sending and captures the security requirements of the TLS session which are: (i) {\em Uni- and bi-directional authentication (non-transferrable authentication)}. In uni-directional authentication mode, this functionality allows the adversary to impersonate the initiator. However, the adversary cannot impersonate the initiator after the two parties start communicating with each other. (ii) {\em  Confidentiality of messages.} This property is captured since adversary can only see the message of corrupted parties, and only receives a leakage $l(m)$ for messages of honest parties.  The leakage function in TLS includes the message length and the TLS error messages. (iii) {\em Integrity.}
Adversary cannot change the honest parties' messages (in practice these messages have a sequence number that comes from the application layer to prevent changing their order). Adversary can only send corrupted messages to corrupted parties of its choice. We consider that  $\mathcal{F}_{\mathsf{SCS}}$ captures the security properties of the post quantum TLS channel (that uses post-quantum secure primitives) as well.

\begin{figure}
\framebox{
\begin{minipage}{\linewidth}
    \centering
    \textbf{Functionality $F_{scs}$}, proceeds as follows, when parameterized by a leakage function $l : \{0, 1\}^* \rightarrow \{0, 1\}^*$.
\begin{itemize}
\item Upon receiving an input $(establish\text{-}session, sid, IDI)$ from some party, where $IDI \in (\perp, I)$, record $IDI$ as initiator, and send the message to the adversary. Upon receiving input $(establishsession, sid, R)$ from some party, record $R$ as responder, and forward the message to the adversary.
\item Upon receiving a value $(impersonate, sid)$ from the adversary, do: If $(IDI=\perp)$, check that no ready entry exists, and record the adversary as initiator. Else ignore the message.
\item Upon receiving a value $(send, sid, m, P')$ from party $P$, which is either initiator or responder, check that a record $(sid, P, P')$ exists, record ready (if there is no such entry) and send $(sent, sid, l(m))$ to the adversary and a private delayed value $(receive, sid, m, P)$ to $P'$. Else ignore the message. If the sender is corrupted, then disclose $m$ to the adversary. Next, if the adversary provides $\hat{m}$ and no output has been written to the receiver, then send $(send, sid, \hat{m} , \hat{P})$ to the receiver unless $\hat{P}$ is an identity of an uncorrupted party.
\end{itemize}   
\end{minipage}
}
\caption{The secure communication session ideal functionality, $\mathcal{F}_{\mathsf{SCS}}$}
\label{Fscs}
\Description{To avoid warning.}
\end{figure}

\subsection{Global Random Oracle Ideal Functionality $\mathcal{G}_h$}
The random oracle functionality $\mathcal{G}_h$ \cite{camenisch2018wonderful}  has a $query$ interface, and responds to all queries with uniformly random sampled values $ r \leftarrow \{0,\;1\}^{l(\kappa)}$, and outputs the same value for the same query. 
$\mathcal{G}_h$ allows the adversary to program the random oracle to its desired values. At the same time, to prevent the adversary from programming collisions, it allows the parties to check whether a query has been programmed or not and drop that value if so. Additionally, the adversary can observe all the adversarial queries made from a wrong session $L_{sid}$.

 \begin{figure}
  \vspace{3pt}
 	\centering
 	\begin{center}
 		\small
 		\fboxsep=2pt
 		\framebox{\begin{minipage}{0.96\columnwidth}
 				{\centering \textbf{Functionality $\mathcal{G}_h$}\\}

 				{Shared functionality $\mathcal{G}_h$ is globally available to all participants. It takes as input queries $q \in \{0,1\}^*$ and outputs values $r \in \{0,1\}^k$. Internally, it stores initially empty sets $L$, $L_p$, and a set $L_{sid}$ for all sessions $sid$.}
 				
 				{\raggedright
 					\textbf{Query}\\
 					-- Upon receiving $(query,\;sid,\;q)$ from a party of session $sid'$ proceed as follows:\\
 					\begin{itemize}
 						\item If $(sid,\;q,\;r) \in L$ respond with $(query,\; q,\; r)$.
 						\item If $(sid,\;q,\;r) \notin L$, samples $r \in \{0,1\}^k$, store $(sid,\;q,\;r)$ in $Q$ and respond with $(query,\; q,\; r)$.
 						\item If the query is made from a wrong session $(sid \neq sid')$, store $(q,\;r)$ in $L_{sid}$.
 					\end{itemize} 
                    
 				\textbf{Program}\\
 			-- Upon receiving message $(program,\;sid,\;q,\; r)$ by the adversary, check if $(sid,\;q,\;r')$ is defined in $L$. If this is the case, abort. Otherwise, if $r \in \{0,1\}^k$ store $(sid,\;q,\;r)$ in $Q$ and $(sid,\; q)$ in $L_p$.\\
 					-- Upon receiving $(isPrgrmd,\;q)$ from a party of session $sid$, check if $(sid,\; q)$ in $L_p$. If this is the case respond with $(isPrgmd,\;1)$.\\
 					
 					\textbf{Observe}\\
 					-- Upon receiving message $(observe)$ from the adversary $\mathcal{A}$ of session id $sid$ respond with $(observe,\; L_{sid})$.
 				}
 			\end{minipage}
 		}
 	\end{center}
 	\caption{Global restricted programmable and observable random oracle functionality $\mathcal{G}_h$ \cite{camenisch2018wonderful}, provides three interfaces \textit{Query}, \textit{Program}, and \textit{Observe} to participants. 
 	} 
 	\label{RO}
    \Description{To avoid warning.}
 \end{figure}

\subsection{Certification Authority Ideal Functionality $\mathcal{F}_{\mathsf{CA}}$}
We follow \cite{canetti2004signature} for the ideal functionality of $\mathcal{F}_{\mathsf{CA}}$ (see Figure \ref{Fca}). This ideal functionality is bound to a single identity via sid. $\mathcal{F}_{\mathsf{CA}}$ accepts only the first registered message and does not allow for modification or revocation. $\mathcal{F}_{\mathsf{CA}}$ does not perform any computation on the received message and only acts as the public bulletin board.  

\begin{figure}
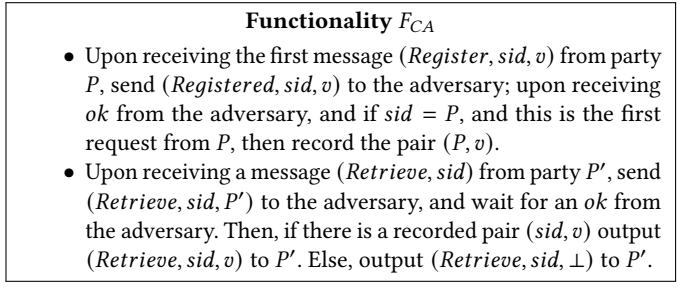

\framebox{
\begin{minipage}{\linewidth}
    \centering
    \textbf{Functionality $F_{CA}$}
\begin{itemize}
\item Upon receiving the first message $(Register, sid, v)$ from party $P$, send $(Registered, sid, v)$ to the adversary; upon receiving $ok$ from the adversary, and if $sid=P$, and this is the first request from $P$, then record the pair $(P,v)$.
\item Upon receiving a message $(Retrieve, sid)$ from party $P'$, send $(Retrieve, sid, P')$ to the adversary, and wait for an $ok$ from the adversary. Then, if there is a recorded pair $(sid,v)$ output $(Retrieve, sid, v)$ to $P'$. Else, output $(Retrieve, sid, \perp)$ to $P'$.
\end{itemize}   
\end{minipage}
}
\caption{The certification authority ideal functionality, $\mathcal{F}_{\mathsf{CA}}$}
\label{Fca}
\Description{To avoid warning.}
\end{figure}

\subsection{Secure Registration Ideal Functionality $F_{Reg}$} 
 We define $F_{Reg}$ to capture the secure registration of the users and the  agents (see Figure \ref{F-Reg}). The user $U$ can choose a specific provider $TA$ to register with and choose a password $Pwd$ that will be stored with its identity (capturing that the user holds a certified public key)  
and will be used for authentication when registering the 
agents.
It also allows a user who has registered with an authorized 
$TA$  
to register its agent $A$ with attribute $attr$. 
The agent information consisting of the user identity, the agent identity, and the agent attribute are stored in the list of registered agents $L_{AReg}$. The adversary learns about the identity of the registered agents (
we assume they are publicly known). 
$F_{Reg}$ also allows the user $U$ 
to add or update a contact policy for its agent. The contact policy
$CP_A=\{[\mathsf{Send}/\mathsf{Receive},A_i,\mathsf{Budget}_i]\}$ specifies
which agents $A_i$ may send requests to, or receive requests from, $A$, together
with the corresponding contact budget $\mathsf{Budget}_i$. 
The adversary is only informed about the call but nothing beyond that, which captures the fact that the communication is over PQ-TLS channel and the adversary can only learn if the user $U$ has contacted the provider without knowing for which of the agents it has contacted nor for what purpose, that is, adding a policy or updating a policy.

\begin{figure}
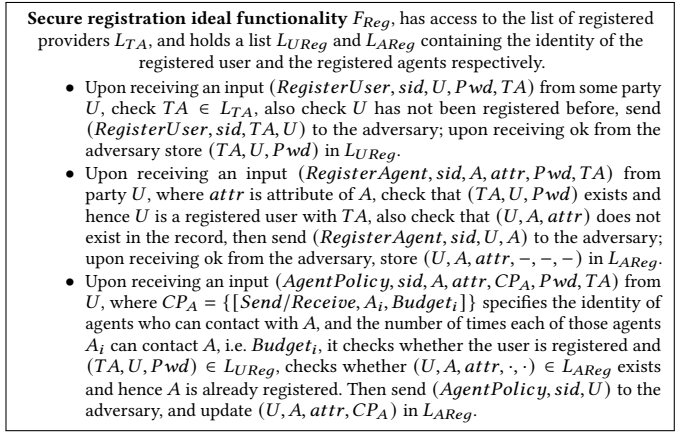

\framebox{
\footnotesize
\begin{minipage}{\linewidth}
    \centering
    \textbf{Secure registration ideal functionality $F_{Reg}$}, has access to the list of registered providers $L_{TA}$, and holds a list $L_{UReg}$ and $L_{AReg}$  containing the identity of the registered user and the registered agents respectively.
\begin{itemize}
\item  Upon receiving an input $(RegisterUser, sid, U, Pwd, TA)$ from some party $U$, check $TA \in L_{TA}$, also check $U$ has not been registered before, send $(RegisterUser, sid, TA, U)$ to the adversary; upon receiving ok from the adversary store $(TA, U, Pwd)$ in $L_{UReg}$.
\item Upon receiving an input $(RegisterAgent, sid, A, attr, Pwd, TA)$ from party $U$, where $attr$ is attribute of $A$, check that $(TA, U, Pwd)$ exists and hence $U$ is a registered user with $TA$, also check that $(U,A,attr)$ does not exist in the record, then send  $(RegisterAgent, sid, U, A)$ to the adversary; upon receiving ok from the adversary,  store $(U,A,attr, -, -, -)$ in $L_{AReg}$.
\item Upon receiving an input $(AgentPolicy, sid, A, attr, CP_A, Pwd, TA)$ from $U$,  where $CP_A=\{[Send/Receive, A_i, Budget_i]\}$ specifies the identity of agents who can contact with $A$, and the number of times each of those agents $A_i$ can contact $A$, i.e. $Budget_i$,  
 it checks whether  the user is registered and $(TA, U, Pwd) \in L_{UReg}$, checks whether $(U,A,attr, \cdot, \cdot) \in L_{AReg}$ exists and hence $A$ is already registered.  
Then send $(AgentPolicy,sid, U)$ to the adversary, and update $(U, A, attr, CP_A)$ in $L_{AReg}$.
\end{itemize}

\end{minipage}
}
\caption{Secure user and agent registration ideal functionality $\mathcal{F}_{\mathsf{Reg}}$}
\label{F-Reg}
\Description{To avoid warning.}
\end{figure}

\subsection{\ases \ Authorization Ideal Functionality $\mathcal{F}_{\mathsf{A\text{-}auth}}$} 

It is run between an agent $A$ and the Provider $TA$, and ensures that an agent is authorized (according to the user-defined policies $CP_A$) to establish an \ases \ with another agent (see Figure \ref{F-aAuth}). 

\begin{figure}
\framebox{
\footnotesize
\begin{minipage}{\linewidth}
    \centering
    \textbf{\ases \ Authorization ideal functionality $\mathcal{F}_{\mathsf{A\text{-}auth}}$}, 
    holds a list $L_{AReg}$  containing the identity and attribute of the registered agents, and their contact policy $CP_{A}$ of all registered agents, and a counter $CTR$ set that keeps the counter states for enforcing the budget. 

\begin{itemize}
\item Upon receiving an input $(GetAuthorization, sid, A, attr, A', attr', TA)$ from $A$, the ideal functionality checks whether $(A,attr) \in L_{AReg}$ and $(A',attr') \in L_{AReg}$.  It then retrieves $CP_{A}$ and $CP_{A'}$ and checks whether $(A, A') \in CP_{A} \cap CP_{A'}$. If yes, it finds the counter $ctr_{A,A'} \in CTR$ which has ben initialized with $B=min\{Budget_A, Budget_{A'}\}$ the first time, and checks whether $ctr_{A,A'} \leq B$. If so, it decrements $ctr_{A,A'}$ by one, and sends $(Allowed, sid, A, attr, A', attr')$ to the agent $A$, the $TA$, and the adversary, and stores $(Allowed, sid, A, attr, A', attr')$ in $L_{Auth}$.
\end{itemize}

\end{minipage}
}
\caption{\ases \ Authorization ideal functionality $F_{A\text{-}auth}$}
\label{F-aAuth}
\Description{To avoid warning.}
\end{figure}

\subsection{Secure A-session Ideal Functionality $\mathcal{F}_{\mathsf{Ases}}$} \label{ases-func}
 The secure \ases\ ideal functionality, shown in Figure~\ref{F-SAS}, involves a
user and two agents, \caler\ and \caled. It captures the security requirements
of an \ases, including \tmsg\ confidentiality, participant authentication,
\tmsg\ integrity, accountability, and policy enforcement with respect to
\tmsg\ and time budgets.
In this functionality, the users specify their agents' policies, and the user of
\caler\ provides the task to be performed. \caler\ acts as the initiator and contacts \caled\ multiple times to complete
the given task.

 The task, in our model, determines when the session starts and when it
terminates. We model task execution as a stateful request-response computation
between the agents. The initiator-side algorithm $\mathsf{ExeRes}$ takes the
task $tsk$, the previous response, and the initiator's code as input, and outputs
the next request. The responder-side algorithm $\mathsf{ExeReq}$ takes a request
and the responder's code as input, and outputs a response. The responder need
not know the underlying task $tsk$.

The algorithm $\mathsf{ExeReq}$ may itself be stateful and interactive: the
responding agent may interact with other agents before producing its response to
\caler.

 The ideal functionality is parameterized by a leakage function $\ell()$ on the
request/response transcript and the \tmsg\ and time budgets. If one agent is
corrupted, $\ell()$ leaks the budgets fully, reflecting the policy information
exchanged during an \ases. If both agents are honest, $\ell()$ leaks only
unavoidable transcript metadata, such as message lengths and execution times.

The ideal functionality handles requests and responses internally, since the
number of rounds may depend on the task. When the initiator $A_I$ is honest, the
adversary does not know the task; hence, without the ideal functionality running the task internally and suitable leakage the simulator
may be unable to reproduce the observable round structure of the real execution,
allowing the environment to distinguish the two worlds.  Following the
internal simulation of rounds and leakage-based approach of~\cite{eckey2020optiswap,avizheh2024refereed},
the functionality simulates the task execution internally and therefore leaks specified metadata in each round, enabling the
simulator to match the observable transcript while preserving the intended
communication-security guarantees.

  \Fases\ captures the stated security requirements as follows:
\begin{itemize}
    \item \emph{Confidentiality of \tmsg.} If both agents are honest, the
    adversary learns only the  leakage about the requests and
    responses, as defined by the leakage function $\ell()$.

    \item \emph{Authentication.} Only registered and authorized agents can
    participate in an \ases; the adversary cannot impersonate an honest entity.

     \item \emph{Integrity.} Requests and responses are simulated internally by \Fases\ and 
    cannot be modified, injected, replayed, or reordered by the adversary. 

  \item \emph{Accountability.} If a corrupted party sends an invalid,
    malformed, or policy-violating \tmsg, \Fases\ can identify the corrupted
    party whose behavior causes an abort through the adversary, and attribute that to the corrupted party.  

    \item \emph{Policy enforcement and authorization.} \Fases\ receives the two agents' \ases\
    policies, including their \tmsg\ and time budgets, and enforces the effective
    session policy obtained from their intersection. It maintains the relevant
    counters and timers, and terminates the session once a budget is exhausted,
    the time bound expires, or a corrupted party triggers an abort (capturing the early termination of session by a corrupted party).
\end{itemize}
 The {\em authorization soundness} in Section \ref{sec:model} is guaranteed by the composition  of 
$\mathcal{F}_{\mathsf{A\text{-}auth}}$ and \Fases. 

Note that our protocol also ensures {\em Forward security} for an \ases\  which means that
compromise of long-term credentials does not compromise previously established
\ases \ keys. While our protocol  provides this property,  we leave modeling and  proving this  property for future work.

\begin{figure}[!htb]
\framebox{
\footnotesize
\begin{minipage}{\linewidth}
    \centering
    \textbf{The Secure A-session ideal functionality $\mathcal{F}_{\mathsf{Ases}}$}, is defined for a given task $tsk$, and interacts with a global clock functionality $\mathcal{G}_{\mathsf{clk}}$, and it is parametrized by a leakage function $\ell(\cdot): \{0,1\}^* \times \{0,1\}^{|Q|} \times \{0,1\}^{|\Delta|} \rightarrow \{0,1\}^*$ (note that $sid$ includes $\mathcal{G}_h(tsk)$).
    $\mathcal{F}_{\mathsf{Ases}}$ holds the list of registered agents $L_{AReg}$ and the list of agents that are authorized to establish an A-session $L_{Auth}$.

\begin{itemize}
\item  Upon receiving $(AgentPolicy, sid, APolicy)$ from the user $U$, where $Apolicy$ consists of the tuple of the form $[(A, A'), Q_{A,A'}, \Delta_{A,A'}]$, store $(sid, APolicy)$ in $L_{APo}$. If $A$ is corrupted and later it receives $(ChangePolicy, sid, APolicy')$ from the adversary replace $APolicy$ with $Apolicy'$ in $L_{APo}$.
\item Upon receiving a $(Task, sid, tsk)$ from $(A_I, attr_I)$, where $tsk$ is given to agent $A_I$ with attribute $attr_I$ (if $tsk=\perp$ abort), it checks $(A_I,attr_I) \in L_{AReg}$, and stores $(sid, tsk, A_I, attr_I)$. 
Then, run $(A_R, attr_R, req_1) \leftarrow ExeRes(A_I,tsk, \perp)$, where $Execute$ runs the code of $A_I$ on the given task $tsk$ and outputs (i) the identity and attribute $(A_R, attr_R)$ of the agent that $A_I$ has to connect to,  (ii) the first request to be sent from $A_I$ denoted by $req_1$. 
Then it checks $(A_R,attr_R) \in L_{AReg}$, if not aborts.
Next, retrieve and parse $APolicy$ from $L_{APo}$ for both agents and retrieve $[(A_I, A_R), Q_{I}, \Delta_{I}]$ and $[(A_R, A_I), Q_{R}, \Delta_{R}]$. Then, set $Q=min\{Q_{I},Q_{R}\}$ and $\Delta=min\{\Delta_{I}, \Delta_{R}\}$. Next compute $T_{exp}=\tau+\Delta$, and start a counter $ctr$ for interaction limit. Then,
simulates the task execution internally as below by starting with {\em Request handling} where $req_0= \perp$:
\begin{description}
    \item[--\textbf{Request handling}:] upon receiving $(Request, sid, A_I, attr_I, req_i)$ from $A_I$ with attribute $attr_I$, check $(AuthorizationAllowed, \cdot, A_I, attr_I, A_R,attr_R) \in L_{Auth}$; if so then
    send $(ClockRead, sid)$ to $\mathcal{G}_{\mathsf{clk}}$ and get the time $\tau$.
    Then check whether $ctr \leq Q$, and increment $ctr$ by one; Then, check $\tau \leq T_{exp}$. If the checks do not pass, abort and inform the adversary.  Otherwise,   run $res_i \leftarrow ExeReq(A_I,req_i)$,  and send $(Response, sid, A_R, attr_R, \ell(res_i, Q_R, \Delta_R))$ to the adversary (if $A_I$ or $A_R$ is corrupted $\ell()$ leaks $Q_R$ and $\Delta_R$ to the adversary). 
    Upon receiving abort from the adversary, abort and output the identity of the corrupted agent, else proceed to {\em response handling}. 
    \item[--\textbf{Response handling:}] upon receiving $(Response, sid, A_R, attr_R, res_i)$ from $A_R$ with attribute $attr_R$,  check whether $ctr \leq Q$, and increment $ctr$ by one; Then, check $\tau \leq T_{exp}$. If the checks do not pass, abort and inform the adversary. Otherwise, run $(req_i, result) \leftarrow ExeRes(A_R,tsk,res_i)$.  If $req_i=\perp$, then output $result$ and terminate. Else send $(Request, sid, A_I, attr_I, \ell(req_i, Q_I, \Delta_I))$ to adversary (if $A_I$ or $A_R$ is corrupted $\ell()$ leaks $Q_I$ and $\Delta_I$ to the adversary). Upon receiving $terminate$ from adversary, output $res_i$ and terminate. Upon receiving abort from the adversary, abort and output the identity of the corrupted agent, else proceed to {\em request handling}.
\end{description}
\end{itemize}   
\end{minipage}
}
\caption{The secure  A-session ideal functionality, $\mathcal{F}_{\mathsf{Ases}}$}
\label{F-SAS}
\Description{To avoid warning.}
\end{figure}

\subsection{The Secure Multi-Agent Composite Session (C-session) Ideal functionality, $\mathcal{F}_{\mathsf{Cses}}$ with Static Workflow}  This ideal functionality allows  a user to give a task to an \masterw \ agent \master, where  \master \ contacts a set of responding agents $A_i$ (according to a static work plan) to fulfill the task (see Figure \ref{F-MAS}). This ideal functionality ensures the security properties listed in \Fases \ in a similar way but in the communication across mutiple responding agents (see description of \Fases). These properties are: \tmsg\ confidentiality, participant authentication, \tmsg\ integrity, accountability, and policy enforcement with respect to \tmsg\ and time budgets of \master \ and the set of $A_i$.
In \Fcses \,  the users determine the policy for their agents at the beginning of the session. The agent policy is in the form of $[(A, A'), Q_{A,A'}, \Delta_{A,A'}]$ and consists of the interaction budget $Q_{A,A'}$ and time budget $\Delta_{A,A'}$. $[(A, \perp), Q_{tot}, \Delta_{tot}]$ represents the policy of \masterw \ agent (i.e., ICP). 

We highlight three points in our model: (i) Similar to the A-session functionality, we model the task execution as a stateful interactive computation, in which $A^{*}_{I}$ acts as initiator  and contacts with other agents $A_i$ several times in order to complete the given task. We only consider one-hop interactions, which means $A^{*}_{I}$ runs $ExeRes()$, and contact another agent $A_i$ back and forth, but we do not model the case where $A_i$ interacts with other agents to respond the requests from $A^{*}_{I}$ as it makes it very complex for the analysis. Without loss of generality, we assume that the $A_i$ uses the reactive function $ExeReq()$ to respond to a request. This function/ computation by itself is an stateful interactive function which allows $A_i$ to interact with other agents and respond to $A^{*}_{I}$ at the end. (ii) We consider a static workflow for task execution, where the \masterw \ agent $A^{*}_{I}$ first determines the \textit{schedule} consisting of the list of agents $(A_i,attr_i)$ that should be contacted based on the given task. 
(iii) The ideal functionality simulates the handling of requests and responses internally since the number of requests and responses depend on the given task, and the adversary who has not corrupted the \masterw \ agent $A^{*}_{I}$ and does not know the task cannot simulate the number of rounds of requests and response and the agents who should be contacted with correctly. This will enable the environment to distinguish the ideal world from the real world. Therefore, we follow the approach of \cite{eckey2020optiswap,avizheh2024refereed} and consider that the ideal functionality simulate the request and response process internally, enforces the policy, and leak some information to the adversary. With this consideration, we only need to show that the simulator can simulate each round correctly even if only a negligible information is leaked to it.

\begin{figure}
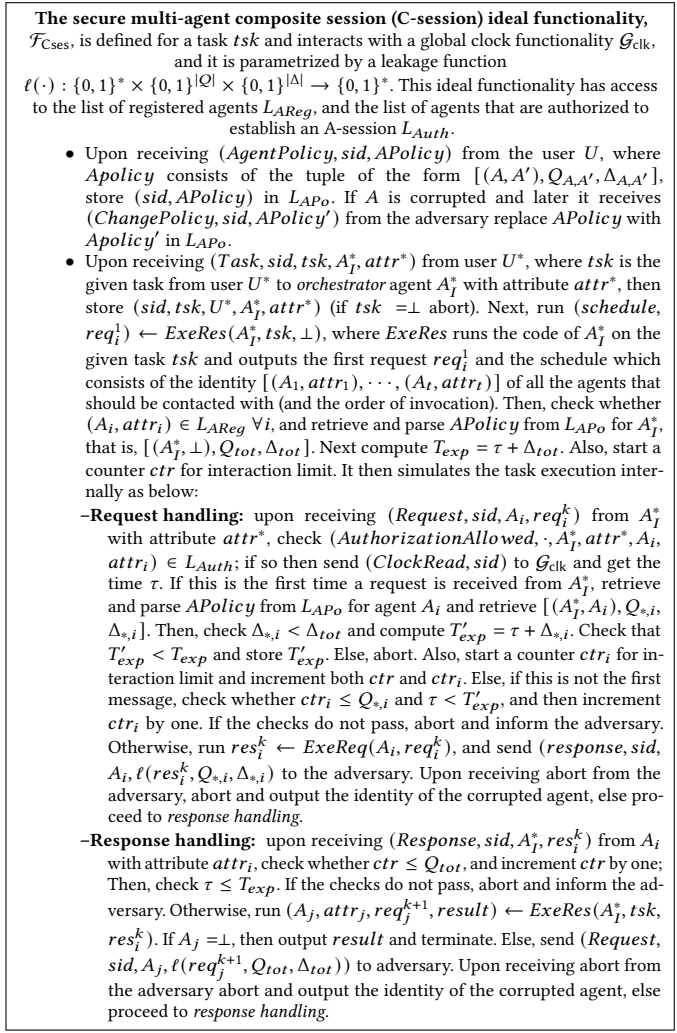

\framebox{
\footnotesize
\begin{minipage}{\linewidth}
    \centering
    \textbf{The secure multi-agent composite session (C-session) ideal functionality, $\mathcal{F}_{\mathsf{Cses}}$}, is defined for a task $tsk$ and interacts with a global clock functionality $\mathcal{G}_{\mathsf{clk}}$, and it is parametrized by a leakage function $\ell(\cdot): \{0,1\}^* \times \{0,1\}^{|Q|} \times \{0,1\}^{|\Delta|} \rightarrow \{0,1\}^*$. This ideal functionality has access to the list of registered agents $L_{AReg}$, and the list of agents that are authorized to establish an A-session $L_{Auth}$.

\begin{itemize}
\item  Upon receiving $(AgentPolicy, sid, APolicy)$ from the user $U$, where $Apolicy$ consists of the tuple of the form $[(A, A'), Q_{A,A'}, \Delta_{A,A'}]$, store $(sid, APolicy)$ in $L_{APo}$. If $A$ is corrupted and later it receives $(ChangePolicy, sid, APolicy')$ from the adversary replace $APolicy$ with $Apolicy'$ in $L_{APo}$.
\item Upon receiving $(Task, sid, tsk, A^{*}_{I}, attr^*)$ from user $U^*$, where $tsk$ is the given task from user $U^*$ to \masterw \ agent $A^{*}_{I}$ with attribute $attr^*$, then store $(sid, tsk, U^*, A^{*}_{I}, attr^*)$ (if $tsk=\perp$ abort).  
Next, run $(schedule, req^{1}_i) \leftarrow ExeRes(A^{*}_{I},tsk,\perp)$, where $ExeRes$ runs the code of $A^{*}_{I}$ on the given task $tsk$ and outputs the first request $req^{1}_i$ and the schedule which consists of the identity $[(A_1, attr_1), \cdots, (A_t, attr_t)]$ of all the agents that should be contacted with (and the order of invocation).  Then,  check whether $(A_i,attr_i) \in L_{AReg}$ $\forall i$, and retrieve and parse $APolicy$ from $L_{APo}$ for $A^{*}_{I}$, that is,  $[(A^{*}_{I}, \perp), Q_{tot}, \Delta_{tot}]$. Next compute $T_{exp}=\tau+\Delta_{tot}$. Also, start a counter $ctr$ for interaction limit. 
It then simulates the task execution internally as below:
\begin{description}
    \item[--\textbf{Request handling}:] upon receiving $(Request, sid, A_i, req^{k}_i)$ from $A^{*}_{I}$ with attribute $attr^*$,  check $(AuthorizationAllowed,\cdot, A^{*}_{I}, attr^*, A_i, attr_i) \in L_{Auth}$; 
    if so then
    send $(ClockRead, sid)$ to $\mathcal{G}_{\mathsf{clk}}$ and get the time $\tau$.  If this is the first time a request is received from $A^{*}_{I}$, retrieve and parse $APolicy$ from $L_{APo}$ for agent $A_i$ and retrieve $[(A^{*}_{I}, A_i), Q_{*,i}, \Delta_{*,i}]$. Then, check $\Delta_{*,i}< \Delta_{tot}$ and compute $T'_{exp}=\tau+\Delta_{*,i}$. Check that $T'_{exp} <T_{exp}$ and store $T'_{exp}$.  Else, abort. Also, start a counter $ctr_i$ for interaction limit  and increment both $ctr$ and $ctr_i$. 
    Else, if this is not the first message, check whether $ctr_i \leq Q_{*,i}$ and $\tau <T'_{exp}$, and then increment $ctr_i$ by one. 
    If the checks do not pass, abort and inform the adversary.  Otherwise,   run $res^{k}_i \leftarrow ExeReq(A_i,req^{k}_i)$, and send $(response, sid, A_i, \ell(res^{k}_i, Q_{*,i}, \Delta_{*,i})$ to the adversary.  
    Upon receiving abort from the adversary, abort and output the identity of the corrupted agent, else proceed to {\em response handling}. 
    \item[--\textbf{Response handling:}] upon receiving $(Response, sid, A^{*}_{I}, res^{k}_i)$ from $A_i$ with attribute $attr_i$,  check whether $ctr \leq Q_{tot}$, and increment $ctr$ by one; Then, check $\tau \leq T_{exp}$. If the checks do not pass, abort and inform the adversary.  Otherwise, run $(A_j, attr_j, req^{k+1}_j, result) \leftarrow ExeRes(A^{*}_{I},tsk,res^{k}_i)$. If $A_j= \perp$, then output $result$ and terminate. Else, send $(Request, sid, A_j, \ell(req^{k+1}_j, Q_{tot}, \Delta_{tot}))$ to adversary.  Upon receiving abort from the adversary abort and output the identity of the corrupted agent, else proceed to {\em response handling}.  
\end{description}
\end{itemize}   
\end{minipage}
}
\caption{The secure multi-agent composite session (C-session) ideal functionality, $\mathcal{F}_{\mathsf{Cses}}$ with static workflow}
\label{F-MAS}
\Description{To avoid warning.}
\end{figure}

\subsection{The Secure Multi-agent Composite Session (C-session) Ideal Functionality, $\mathcal{F}_{\mathsf{Cses}}$ with Dynamic Workflow} This ideal functionality allows  the users to determine the policy for their agents, where the agent policy is in the form of $[(A, A'), Q_{A,A'}, \Delta_{A,A'}]$ and consists of the interaction budget $Q_{A,A'}$ and time budget $\Delta_{A,A'}$. The policy $[(A, \perp), Q_{tot}, \Delta_{tot}]$ represents the policy of a \masterw \ agent (i.e., ICP). \Fcses ensures: \tmsg\ confidentiality, participant authentication, \tmsg\ integrity, accountability, and policy enforcement with respect to \tmsg\ and time budgets of \master \ and the set of $A_i$, across multiple agents.  It ensures these security properties in a similar way described in \Fases \ but in the communication across many responding agents (see description of \Fases). 

In $\mathcal{F}_{\mathsf{Cses}}$, a user gives a task $tsk$ to an \masterw agent $A^{*}_{I}$ with attribute $attr^*$ which  orchestrates the communication with other agents (see Figure \ref{F-MAS2}).  The \masterw \ agent can be corrupted in which case there is no guarantee that it performs the task correctly (we allow the adversary to replace the task which affects the order of agent calls and the requests it makes, and also to replace the result with adversarial ones).  If the \masterw \ agent $A^{*}_{I}$ is honest, the ideal functionality execute the task in a trusted way, which means that it first gets  identity of the first agent that needs to be called and the first request by running the reactive function $ExeRes(A^{*}_{I},tsk, \perp)$ on the given task $tsk$ using the code of the agent $A^{*}_{I}$.  It, then, internally simulates the request sending and receiving of the agents, and leaks some information about each request and response to the adversary. We also let the adversary to terminate the request and response handling to capture the fact that in case the interacting agents are corrupted they may abort earlier and then no results will be returned to the user. 

Note that we need to clarify three points: (i) We model the task execution as an stateful interactive computation, in which $A^{*}_{I}$ acts as initiator  and contacts with other agents $A_i$ several times in order to complete the given task. We only consider one-hop interactions, which means $A^{*}_{I}$ runs $ExeRes()$, and contact another agent $A_i$ back and forth, but we do not model the case where $A_i$ interacts with other agents to respond the requests from $A^{*}_{I}$ as it makes it very complex for the analysis. Without loss of generality, we assume that the $A_i$ uses the reactive function $ExeReq()$ to respond to a request. This function/ computation by itself in an stateful interactive function which allows $A_i$ to interact with other agents and respond to $A^{*}_{I}$ at the end. (ii) We consider a dynamic workflow for task execution, where the \masterw \ agent $A^{*}_{I}$ first determines the first contacting agent $(A_i,attr_i)$ based on the given task and then based on the responses that it receives from this agent it decides the next agent that should be contacted. Therefore, we consider that $ExeRes()$ always output an identity and the next request. (iii) The ideal functionality simulates the handling of requests and responses internally since the number of requests and responses depend on the given task, and the adversary who has not corrupted the \masterw \ agent $A^{*}_{I}$ and does not know the task cannot simulate the number of rounds of requests and response and the agents who should be contacted with correctly. This will enable the environment to distinguish the ideal world from the real world. Therefore, we follow the approach of \cite{eckey2020optiswap,avizheh2024refereed} and consider that the ideal functionality simulate the request and response process internally, enforces the policy, and leak some information to the adversary. With this consideration, we only need to show that the simulator can simulate each round correctly even if only a negligible information is leaked to it.

\begin{figure}
\framebox{
\footnotesize
\begin{minipage}{\linewidth}
    \centering
    \textbf{The secure multi-agent composite session (C-session) ideal functionality, $\mathcal{F}_{\mathsf{Cses}}$}, is defined for a task $tsk$ and interacts with a global clock functionality $\mathcal{G}_{\mathsf{clk}}$, and it is parametrized by a leakage function $\ell(\cdot): \{0,1\}^* \times \{0,1\}^{|Q|} \times \{0,1\}^{|\Delta|} \rightarrow \{0,1\}^*$. This ideal functionality has access to the list of registered agents $L_{AReg}$, and the list of agents that are authorized to establish an A-session $L_{Auth}$.

\begin{itemize}
\item  Upon receiving $(AgentPolicy, sid, APolicy)$ from the user $U$, where $Apolicy$ consists of the tuple of the form $[(A, A'), Q_{A,A'}, \Delta_{A,A'}]$, store $(sid, APolicy)$ in $L_{APo}$. If $A$ is corrupted and later it receives $(ChangePolicy, sid, APolicy')$ from the adversary replace $APolicy$ with $Apolicy'$ in $L_{APo}$.
\item Upon receiving $(Task, sid, tsk, A^{*}_{I}, attr^*)$ from user $U^*$, where $tsk$ is the given task from user $U^*$ to \masterw \ agent $A^{*}_{I}$ with attribute $attr^*$, then store $(sid, tsk, U^*, A^{*}_{I}, attr^*)$ (if $tsk=\perp$ aborts). 
Next, run $(A_i, attr_i, req^{1}_i) \leftarrow ExeRes(A^{*}_{I},tsk,\perp)$ where $ExeRes$ runs the code of $A^{*}_{I}$ on the given task $tsk$ and outputs the first request $req^{1}_i$ and the identity $(A_i, attr_i)$ of the first agent that needs to be contacted. Then, retrieve and parse $APolicy$ from $L_{APo}$ for $A^{*}_{I}$, that is,  $[ (A^{*}_{I}, \perp), Q_{tot}, \Delta_{tot}]$. Next compute $T_{exp}=\tau+\Delta_{tot}$. Also, start a counter $ctr$ for interaction limit.  It then simulates the task execution internally as below:
\begin{description}
   \item[--\textbf{Request handling}:] upon receiving $(Request, sid, A_i, req^{k}_i)$ from $A^{*}_{I}$ with attribute $att^*$, check $(AuthorizationAllowed,\cdot, A^{*}_{I}, attr^*, A_i, attr_i) \in L_{Auth}$; if so then
    send $(ClockRead, sid)$ to $\mathcal{G}_{\mathsf{clk}}$ and get the time $\tau$. If this is the first time a request is received from $A^{*}_{I}$, retrieve and parse $APolicy$ from $L_{APo}$ for agent $A_i$ and retrieve $[(A^{*}_{I}, A_i), Q_{*,i}, \Delta_{*,i}]$. Then, check $\Delta_{*,i}< \Delta_{tot}$ and compute $T'_{exp}=\tau+\Delta_{*,i}$. Check that $T'_{exp} <T_{exp}$ and store $T'_{exp}$. Else, abort. Also, start a counter $ctr_i$ for interaction limit  and increment both $ctr$ and $ctr_i$. 
    Else, if this is not the first message, check whether $ctr_i \leq Q_{*,i}$  and $\tau <T'_{exp}$, and then increment $ctr_i$ by one. 
    If the checks do not pass, abort and inform the adversary.  Otherwise,   run $res^{k}_i \leftarrow ExeReq(A_i,tsk,req^{k}_i)$, and send $(Response, sid, A_i, \ell(res^{k}_i, Q_{*,i}, \Delta_{*,i}))$ to the adversary. 
    Upon receiving abort from the adversary abort and output the identity of the corrupted agent, else proceed to {\em response handling}.
    \item[--\textbf{Response handling:}] upon receiving $(Response, sid, A^{*}_{I}, res^{k}_i)$ from $A_i$ with attribute $attr_i$, check whether $ctr \leq Q_{tot}$, and increment $ctr$ by one; Then, check $\tau \leq T_{exp}$. If the checks do not pass, abort and inform the adversary.  Otherwise, run $(A_j, attr_j, req^{k+1}_j, result) \leftarrow ExeRes(A^{*}_{I},tsk,res^{k}_i)$.  If $A_j= \perp$, then output $result$ and terminate. Else, send $(Request, sid, A_j, \ell(req^{k+1}_j, Q_{tot}, \Delta_{tot}))$ to adversary. Upon receiving abort from the adversary abort and output the identity of the corrupted agent, else proceed to {\em response handling}.  
\end{description}
\end{itemize}   
\end{minipage}
}
\caption{The secure multi-agent composite session (C-session) ideal functionality, $\mathcal{F}_{\mathsf{Cses}}$ with dynamic workflow}
\label{F-MAS2}
\Description{To avoid warning.}
\end{figure}

\section{Security Analysis}
\label{Appendix:securityanalysis}

In this section, we give the security analysis of \sagaplus~in multiple steps as described in Figure \ref{UCsteps}. 


\subsection{Main Theorem}
The following theorem gives the security of \sagaplus:

\begin{theorem}
\sagaplus\   (Section \ref{sec:sagaplus}) UC-realizes
$\mathcal{F}_{\mathsf{Reg}}$, $\mathcal{F}_{A\text{-}\mathsf{auth}}$, \Fases,
and \Fcses\ in the
$(\mathcal{F}_{\mathsf{scs}},\mathcal{F}_{\mathsf{CA}}, \mathcal{G}_{\mathsf{clk}},\mathcal{G}_h)$-hybrid model. 
The cryptographic building blocks of \sagaplus\ are (i) a digital signature scheme that is EUF-CMA, and (ii) hash function, HMAC and PRF that are modeled by global random oracles.
\end{theorem}

(It is sufficient for the signature scheme that are used by all \sagaplus\ participants  to be EUF-CMA in standard model).

\textbf{Proof}. \paragraph{Proof sketch.}
We construct the protocol $\pi_{\sagaplus}^{\mathcal{F}_{\mathsf{scs}},
\mathcal{F}_{\mathsf{CA}},\mathcal{G}_{\mathsf{clk}},\mathcal{G}_h}$ in the $(\mathcal{F}_{\mathsf{scs}},\mathcal{F}_{\mathsf{CA}},
\mathcal{G}_{\mathsf{clk}},\mathcal{G}_h)$-hybrid model and prove Theorem~\ref{thm:main} through a sequence of lemmas, each capturing the security of one \sagaplus\ protocol.  These protocols are executed sequentially: an agent may participate in an \ases\ or a \cses\ only after registration, and an initiating agent may start the session only after obtaining authorization from the \provider. The lemmas are stated below.

\textbf{Hybrid world.} We consider a hybrid protocol $\pi^{\mathcal{F}_{\mathsf{SCS}}, \mathcal{F}_{\mathsf{CA}}, \mathcal{G}_{\mathsf{clk}}, \mathcal{G}_h}_{\sagaplus}$. This means that we replace the PQ-TLS channel in the protocol with calls to $\mathcal{F}_{\mathsf{SCS}}$ 
Note that the protocol (specially the \ases ) may run over multiple instances of $\mathcal{F}_{\mathsf{SCS}}$. Additionally, we assume parties have access to a certificate authority ideal functionality $\mathcal{F}_{\mathsf{CA}}$ (that allows registering the identities by a trusted party and has been used in the security of TLS as well \cite{gajek2008universally}), the global clock ideal functionality $\mathcal{G}_{\mathsf{clk}}$ (that captures the reference clock) and the global random oracle ideal functionality  $\mathcal{G}_h$ (that models the hash functions). 
We consider the adversary $\mathcal{A}$ to be static which means it  corrupts the parties only before the protocol starts. 

\begin{figure}[t]
    \centering
    \includegraphics[width=0.6\columnwidth]{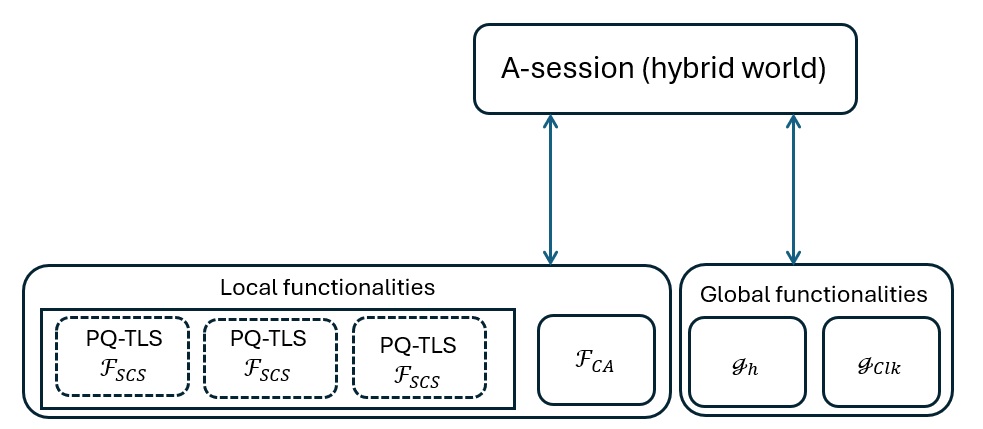}
    \caption{The \ases \ in the two-agent \sagaplus\ protocol}
    \label{fig:Ases}
    \Description{To avoid warning.}
\end{figure}

 \begin{figure}[t]
    \centering
    \includegraphics[width=0.4\columnwidth]{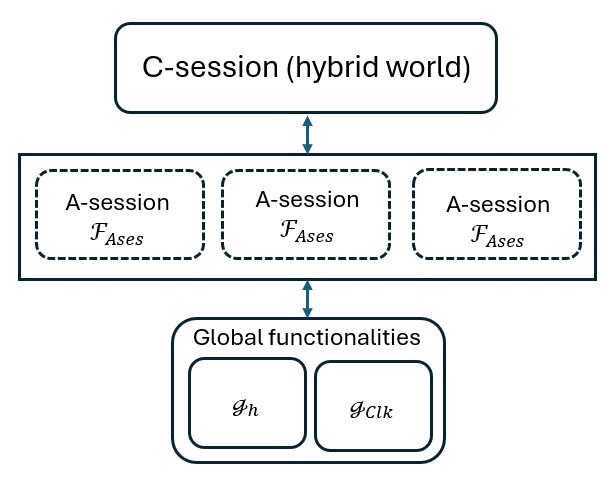}
    \caption{The \cses \ in the multi-agent \sagaplus\ protocol}
    \label{fig:Cses}
    \Description{To avoid warning.}    
\end{figure}

\textbf{Ideal world.} In the ideal world, we consider the ideal functionality $\mathcal{F}$ that captures the security of the protocol and interacts with a dummy user $U_I$, and two agents $A_I$ and $A_R$, and a \provider \ and uses the global clock ideal functionality $\mathcal{G}_{\mathsf{clk}}$ and the global random oracle functionality $\mathcal{G}_{\mathsf{clk}}$ as its subroutine (note that in our analysis, we only consider one protocol at a time, the entities that are involved in that protocol and the ideal functionality $\mathcal{F}$ that captures the security of that protocol.) Parties receive their inputs from the environment $\mathcal{Z}$ and relay them to $\mathcal{F}$. When they receive their output from  $\mathcal{F}$, they just pass it to $\mathcal{Z}$.  We consider a simulator $\mathsf{Sim}$ that is the ideal world adversary which interacts with $\mathcal{F}$, and its goal is to simulate the protocol correctly such that $\mathcal{Z}$ cannot distinguish the ideal world protocol execution from the hybrid world. $\mathsf{Sim}$ has access to the code of the honest parties, some times denoted by $Sim^A$ for clarity, and can perform the computations internally. $\mathsf{Sim}$ can corrupt parties and in this case it gets oracle access to their codes. 

\subsection{Security Analysis of the Registration Protocol}
In this section, we show that the user and agent registration protocol, described in Appendix \ref{Appendix:UserAgentReg}, UC realizes the registration ideal functionality $\mathcal{F}_{\mathsf{Reg}}$. 

\begin{lemma}
    The registration protocol of $\pi^{\mathcal{F}_{\mathsf{SCS}}, \mathcal{F}_{\mathsf{CA}}, \mathcal{G}_{\mathsf{clk}},\mathcal{G}_h}_{\sagaplus}$ for the user and agent (described in Appendix \ref{Appendix:UserAgentReg}) UC realizes the registration ideal functionality $\mathcal{F}_{\mathsf{Reg}}$ against static adversaries. 
\end{lemma}

{\bf Proof.} We analyze the security of the registration protocol considering that the user and provider are honest. Note that we have considered the provider to be semi-honest but we do not model the leakage against the provider since our goal is to only limit the information that provider can see from the agents communication, and we do not provide any cryptographic guarantee against a semi-honest provider. Below, we give the description of simulator $\mathsf{Sim}$.

\textbf{Simulating the provider.} $\mathsf{Sim}$ generates the identity and PQ-TLS key pairs $(sk_{TS}, pk_{TA})$, and $(sk^{tls}_{TA}, pk^{tls}_{TA})$ for the provider, it also simulates $\mathcal{F}_{\mathsf{CA}}$ by storing the identities of the provider $(TA, pk_{TA})$ and $(TA, pk^{tls}_{TA})$. Note that $\mathsf{Sim}$ also simulates the provider's certificates $Cert^{tls}_{TA}$ and $Cert^{ID}_{TA}$ by signing its public keys using a generated key pair for the CA.

\textbf{Simulating the user registration.} Upon receiving $(RegisterUser, sid, TA, U)$, $\mathsf{Sim}$ generates the identity and PQ-TLS key pairs $(sk_U, pk_U)$, and $(sk^{tls}_U, pk^{tls}_U)$ for the user, it also simulates $\mathcal{F}_{\mathsf{CA}}$ by storing the identities of the user $(U, pk_U)$ and $(U, pk^{tls}_U)$. Then, it creates certificates for user's public keys $Cert^{tls}_{U}$, and $Cert^{ID}_{U}$ by signing the user's public keys using a $CA$ private key (generated by $\mathsf{Sim}$). $\mathsf{Sim}$ also chooses a random value as password $Pwd'$ and stores it, sends ok to $\mathcal{F}_{\mathsf{Reg}}$, and sends a confirmation to the user.

\textbf{Simulating the agent registration.} 
\begin{itemize}
    \item Upon receiving $(RegisterAgent, sid, U, A)$ from $\mathcal{F}_{\mathsf{Reg}}$, $\mathsf{Sim}$ generates PQ-TLS key pairs $(sk_{A}, pk_{A})$, and $(sk^{tls}_{A}, pk^{tls}_{A})$ for the agent, and simulates $\mathcal{F}_{\mathsf{CA}}$ by storing the identity of the agent $(A, Pk^{tls}_A)$. Then, it signs the identity public key of the agent using the user's simulated keys, that is, $\sigma^{U}_{ID}$. It also simulates the agent's metadata, and signs the agent's information $\sigma^U_A.$ using the user's key. Then, it sends ok to $\mathcal{F}_{\mathsf{Reg}}$, and sends the user the confirmation signature $\sigma^{TA}_{A}$ which is generated on the agent's info.
    \item Upon receiving $(AgentPolicy, sid, U)$ from $\mathcal{F}_{\mathsf{Reg}}$, $\mathsf{Sim}$ generates a contact policy $CP_A$ which is initially empty and stores it.
\end{itemize}

$\mathsf{Sim}$ fully simulates the user and agent protocol, since in this case, the user and provider are honest and $\mathcal{Z}$ only sees the output of the honest parties. The output of user in this protocol is a confirmation message and a signature $\sigma^{TA}_{A}$ that is simulated by $\mathsf{Sim}$ and it is indistinguishable from the real world values. We give the hybrid games  for the in-distinguishability of the hybrid world from the ideal world in the full version of the paper.  

\subsection{Security Analysis of Agent Discovery}
In this section, we show that the agent discovery protocol described in Section \ref{sec:sagaplus} UC realizes the \ases\ authorization ideal functionality $\mathcal{F}_{\mathsf{A\text{-}auth}}$. 

\begin{lemma}
    The agent discovery protocol of $\pi^{\mathcal{F}_{\mathsf{SCS}}, \mathcal{F}_{\mathsf{CA}}, \mathcal{G}_{\mathsf{clk}},\mathcal{G}_h}_{\sagaplus}$ (described in Section \ref{sec:one-to-one}) UC realizes the \ases\ authorization ideal functionality $\mathcal{F}_{\mathsf{A\text{-}auth}}$ against static adversaries. 
\end{lemma}

{\bf Proof.} The agent discovery protocol  is run between an agent and a provider. The provider is honest but the agent can be corrupted. We  therefore consider two case: (i) the provider and agent are honest, and (ii) the provider is honest but the agent is corrupted. In the following, we give the description of the simulator $\mathsf{Sim}$ in this protocol considering that the registration protocols have been done as described in the previous section:

\textbf{Simulating PQ-TLS connection.} If both entities are honest, this step is not needed since environment does not have access to it. If the agent is corrupted, $\mathsf{Sim}$ simulates $\mathcal{F}_{\mathsf{SCS}}$ itself, and respond to the agent similar to $\mathcal{F}_{\mathsf{SCS}}$.

\textbf{Simulating the agent's authorization for \ases.} We consider the following cases:
\begin{itemize}
    \item {\em Case (i) Agent $A$ is honest.} Upon receiving $(Allowed, sid, A, attr, A', attr')$, $\mathsf{Sim}$ simulates the information of agent $A'$ consisting of metadata, the endpoint descriptor, certificates, and user $U'$ signatures on the identity and information of $A'$ using random values, and generating private and public key pairs for $A'$. Then, it retrieves the agent $A$' identity and PQ-TLS public keys (note $\mathsf{Sim}$ can retrieve correct values since it simulates $\mathcal{F}_{\mathsf{CA}}$ and we assume the registration protocol have been run beforehand) and generates $\sigma^{TA}_{ac}$ on the $A'$ information and public keys of $A$. It then returns the $A'$ information and $\sigma^{TA}_{ac}$ to agent $A$.
    \item {\em Case (ii) Agent $A$ is corrupted.} In this case, $\mathsf{Sim}$ waits to receive $(aid_A, aid_{A'})$ from $A$ and uses $aid_A$ (which incorporates $Aid_U$) and $aid_U$ received from $A$, rather than randomly generated values for user and agent identities. It retrieves the information of $A$ and $A'$ and if they have not registered yet, it aborts. The rest of the simulation is similar to above. 
\end{itemize}

$\mathsf{Sim}$ fully simulates the agent discovery protocol; the reason is that $\mathsf{Sim}$ can simulate the view of the corrupted agent (by simulating $\mathcal{F}_{\mathsf{SCS}}$ and $\mathcal{F}_{\mathsf{CA}}$ internally) and the output of the agent (which is the agent information and provider's signature) correctly using randomly generated values and the information it receives from the corrupted agent. We give the hybrid games  for the in-distinguishability of the hybrid world from the ideal world in the full version of the paper. 

\subsection{Security Analysis of the  \ases~Protocol}
In this section, we show that the \ases protocol described in Section \ref{sec:one-to-one} (Figure \ref{fig:Asession}) between two agents $A_I$ and $A_R$ UC realizes the \ases ideal functionality \Fases. 

\begin{lemma}
    The \ases \ protocol (see Figure \ref{fig:Ases}) of $\pi^{\mathcal{F}_{\mathsf{scs}}, \mathcal{F}_{\mathsf{CA}}, \mathcal{G}_{\mathsf{clk}},\mathcal{G}_h}_{\sagaplus}$ (described in Section \ref{sec:one-to-one}) UC realizes the \ases \ ideal functionality \Fases \ against static adversaries. 
\end{lemma}

{\bf Proof.} We consider three cases: (i) both $A_I$ and $A_R$ are honest, (ii) only $A_I$ is corrupted, and (iii) only $A_R$ is corrupted. We give the simulations below:

\textbf{Case (i)} is straightforward since $\mathcal{Z}$ only sees the output of honest agents. Assuming that the registration and agent discovery protocols have been simulated as described in previous subsections, and \Fases \ simulates the request and responses internally, the view of $\mathcal{Z}$ is indistinguishable from the hybrid world. Note that $\mathsf{Sim}$ waits to receive  $(Task, sid, A_I, attr_I)$ from \Fases, and replies $Ok$. Then, whenever \Fases \ aborts or terminates, it also aborts and terminates. 

Below we only focus on the case where one of the agents is corrupted. 
$\mathsf{Sim}$ receives messages from the ideal functionality \Fases \ and the corrupted agent. Based on these messages, it has 
to simulate the output of the honest parties and the full view of the corrupted parties. 
Note that $\mathsf{Sim}$ simulates $\mathcal{F}_{\mathsf{scs}}$ and $\mathcal{F}_{\mathsf{CA}}$ internally, and we assume it simulates the registration and agent discovery protocols as described in previous subsections.

\textbf{Case (ii):  $A_I$ is corrupted.} In this case, $\mathsf{Sim}^{A_I}$ has oracle access to the code of $A_I$, and it has to simulate the view of $A_I$ and the output of the honest $A_R$. For that, $\mathsf{Sim}^{A_I}$ translates the messages it receives from the corrupted $A_I$ for the ideal functionality \Fases and vice versa. It also simulates $A_R$'s output.

\textbf{Handshake.} $\mathsf{Sim}^{A_I}$ waits until agent $A_I$ initiates an \ases.
\begin{itemize}
    \item Upon receiving the message $(m_0, \sigma^{TA}_{ac}, \sigma^{A_I}_{init}, \sigma^{U_I}_{ICP})$ from $A_I$,  $\mathsf{Sim}^{A_I}$ parses $m_0=(r_1, info_{A_I}, Q_{I,R}, \Delta_{I,R}, NextTk^{ICP}_0)$ It sends $(ChangePolicy, sid, APolicy)$ from the user $U$ to \Fases\ where $APolicy$ is $[(A_I, A_R), Q_{I,R}, \Delta_{I,R}$.  Note that $\mathsf{Sim}^{A_I}$ simulates $\mathcal{F}_{\mathsf{SCS}}$.
    It then verifies the validity of $\sigma^{TA}_{ac}$, $\sigma^{A_I}_{init}$, $\sigma^{U_I}_{ICP}$, and  $info_{A_I}$, if they are valid it parses $sid$ and gets $\mathcal{G}_{h}(tsk)$; Then, sends $Observe$ to $\mathcal{G}_h$ and retrieves $tsk$. It then sends $(Task, sid, tsk)$ from $A_I$ to \Fases. Else, it sends $(Task, sid, \perp)$ to \Fases.
\item Upon receiving the message $(Response, sid, A_R, attr_R, \ell=[Q^*_R, \Delta^*_R])$ from \Fases, store the message, and  choose a symmetric key $r_2$, compute $k_{ses}=\mathcal{G}_h(r_1, r_2)$. Also compute the seed $\rho_0$ through $\mathcal{G}_h$ and compute the hash chain based on \tmsg $Q_{R}$. Then, let $NextTok^{RCP}_{0}=(\rho_{Q^*_{R}},\rho_{Q*_{R}-1})$, send $m_1=r_2,Q^*_R, \Delta^*_R,NextTok^{RCP}_1, NextTok^{ICP}_1$. Also, compute a tag $tag_1$ using $\mathcal{G}_h$ on $m_1$ and $k_{ses}$, and send to $A_I$. 
\end{itemize}

\textbf{Message transmission.} 
\begin{itemize}
    \item  Upon receiving $(m_i, tag_i)$ from $A_I$, $\mathsf{Sim}^{A_I}$ parses $m_i=(req_i, NextTok^{ICP}_i, NextTok^{RCP}_i)$, then it checks the validity of $tag_i$ on $m_i$ and $k_{ses}$ by querying $\mathcal{G}_h$. 
    Also, program $\mathcal{G}_h$ by sending $(program, sid, NextTok^{ICP}_i, NextTok^{ICP}_{i-1})$. If $\mathcal{G}_h$ aborts or the checks do not pass, upon receiving  $(Request, sid, A_I, attr_I, \ell)$ abort. Else, send ok to \Fases. Additionally, if \Fases aborts, $\mathsf{Sim}^{A_I}$ also aborts and terminates.
    \item Upon receiving a $(Response, sid, A_R, attr_R, \ell)$, $\mathsf{Sim}^{A_I}$ simulates the response to the request $req_{i+1}$.To simulate the responder, $\mathsf{Sim}^{A_I}$ executes the request internally using $ExeReq(A_R, req_{i+1})$ and gets the response $res_{i+1}$.  
    Then, it generates $m_{i+1}=(res_{i+1}, NextTok^{ICP}_{i+1}, NextTok^{RCP}_{i+1})$, and computes the tag using $\mathcal{G}_h$ and $k_{ses}$. 
    If $res_{i+1} \neq \perp$ send $ok$ to \Fases. Else, aborts.   
\end{itemize}
Note that if during the simulation $\mathcal{F}{scs}$ aborts, $\mathsf{Sim}^{A_I}$ simulates another instance of $\mathcal{F}_{\mathsf{scs}}$. 
Also, whenever \Fases aborts, $\mathsf{Sim}^{A_I}$ aborts too.

\textbf{Case (iii):  $A_R$ is corrupted.} In this case, $\mathsf{Sim}^{A_R}$ has oracle access to the code of $A_R$, and it has to simulate the view of $A_R$ and the output of the honest $A_I$. For that, $\mathsf{Sim}^{A_R}$ translates the messages it receives from the corrupted $A_R$ for the ideal functionality \Fases and vice versa. It also simulates $A_I$ internally.

\textbf{Handshake.} $\mathsf{Sim}^{A_R}$ waits until agent $A_I$ initiates establishing the session with $A_R$.
\begin{itemize}
    \item Upon receiving the message $(Request, sid, A_I, attr_I, \ell=[Q^*_I,\Delta^*_I])$ from \Fases, $\mathsf{Sim}^{A_R}$ simulates the establishment of a secure communication session $\mathcal{F}_{\mathsf{scs}}$ between $A_I$ and $A_R$. 
    It also parses $sid$ and extracts the task $tsk$ by sending $Observe$ to $\mathcal{G}_{h}$, and stores $tsk$.
    It  then simulates $(m_0, \sigma^{TA}_{ac}, \sigma^{A_I}_{init}, \sigma^{U_I}_{ICP})$ and send it to $A_R$, where $m_0=(r_1, info_{A_I}, Q^*_I, \Delta^*_I, NextTk^{ICP}_0)$ and $r_1$ is chosen uniformly at random. Also, $info_{A_I}$ is generated based on the information $\mathsf{Sim}^{A_R}$ has obtained during the previous protocols such as registration and agent discovery, or generated randomly from the identical distribution as the real values. Next, send ok to \Fases. If \Fases \ aborts, $\mathsf{Sim}^{A_R}$ aborts and terminates.  
\item Upon receiving $(m_1, tag_1)$  from $A_R$, $\mathsf{Sim}^{A_R}$ parses $m_1=(r_2, Q_{R,I},\Delta_{R,I}, NextTok^{RCP}_{1}, NextTok^{ICP}_{1})$,  and verifies that $tag_{1}$ and $NextTok^{ICP}_1$ are valid. 
Also, parse $NextTok^{RCP}_{1}=[s_{Q^*_R}, s_{Q^*_R-1}]$ and program $\mathcal{G}_{h}$ to output $s_{Q^*_R}$ when queried by $s_{Q^*_R-1}$. If the checks pass and $\mathcal{G}_h$ does not abort, send ok to \Fases, else aborts and terminates. 
\end{itemize}

\textbf{Message transmission.} 
\begin{itemize}
    \item Upon receiving a $(Request, sid, A_I, attr_I, \ell)$, $\mathsf{Sim}^{A_R}$ simulates the request $req_i$. To simulate the request, $\mathsf{Sim}^{A_I}$ executes the response internally using $ExeRes(A_R,tsk, req_i)$ and gets the response $res_{i}$.  
    Then, it generates $m_{i}=(req_{i}, NextTok^{ICP}_{i}, NextTok^{RCP}_{i})$, and computes the tag using $\mathcal{G}_h$ and $k_{ses}$. 
    If $req_{i} \neq \perp$, it sends $ok$ to \Fases. Else, aborts.
    \item  Upon receiving $(m_{i+1}, tag_{i+1})$ from $A_R$, $\mathsf{Sim}^{A_R}$ parses $m_{i+1}=(req_{i+1}, NextTok^{ICP}_{i+1}, NextTok^{RCP}_{i+1})$, then it checks the validity of $tag_{i+1}$ on $m_i$ and $k_{ses}$ by querying $\mathcal{G}_h$. 
    Also, program $\mathcal{G}_h$ by sending $(program, sid, NextTok^{RCP}_{i+1}, NextTok^{RCP}_{i})$. If $\mathcal{G}_h$ aborts or the checks do not pass, upon receiving  $(Response, sid, A_R, attr_R, \ell)$ abort. Else, send ok to \Fases. Additionally, if \Fases aborts, $\mathsf{Sim}^{A_R}$ also aborts and terminates.
\end{itemize}
Note that if during the simulation $\mathcal{F}{scs}$ aborts, $\mathsf{Sim}^{A_R}$ simulates another instance of $\mathcal{F}_{\mathsf{scs}}$. 
Also, whenever \Fases \ aborts, $\mathsf{Sim}^{A_R}$ aborts too. 

The above simulation completely simulates the hybrid \ases \ protocol since the simulator can extract the task, run the code of the honest agents on the task, and simulate the task messages correctly. Additionally, we assume that the signature scheme is EUF-CMA secure, so the corrupted agent cannot forge them, and  the HMAC, hash function, and PRF are modeled by a random oracle, which means that their output are indistinguishable from random values. We give the hybrid games for the in-distinguishability of the hybrid world from the ideal world in the full version of the paper. 

\subsection{Security Analysis of the  \cses \ Protocol}
In this section, we show that the \cses protocol described in Section \ref{sec:one-to-many} (Figure \ref{fig:Asession}) between the \masterw \ agent \master \ and a set of $A_R$ UC realizes the \cses \ ideal functionality \Fcses. 

\begin{lemma}
    The \cses \ protocol (see Figure \ref{fig:Cses}) of $\pi^{\mathcal{F}_{\mathsf{scs}}, \mathcal{F}_{\mathsf{CA}}, \mathcal{G}_{\mathsf{clk}},\mathcal{G}_h}_{\sagaplus}$ (described in Section \ref{sec:one-to-one}) UC realizes the \cses \ ideal functionality \Fcses \ against static adversaries assuming \ases \ $\pi_{Ases}$ UC realizes \Fases. 
\end{lemma}

\subsubsection{Proof.} We consider three cases: (i) all agents are honest, (ii) only \master \ is corrupted, and (iii) all $A_R$'s are corrupted. We give the simulations below:

{\bf Case (i).}  This case is straightforward; when all agents are honest, $\mathsf{Sim}$ simulates the ideal functionalities \Fases \ and $\mathcal{F}_{\mathsf{A\text{-}auth}}$, and the output of honest agents correctly; it receives the task $tsk$ from the dummy \master \ and it executes it internally and gets the subtasks that should be simulated in each \Fases\ (see case (iii)).

{\bf Case (ii).}   When \master \ is corrupted, $\mathsf{Sim}^{A^*_I}$ simulates the ideal functionalities \Fases \ and $\mathcal{F}_{\mathsf{A\text{-}auth}}$, and the view of the corrupted \masterw \ (that is same as the simulated ideal functionalities) and output of honest agents correctly (see below):

\begin{itemize}
\item {\bf Simulating each $\mathcal{F}_{\mathsf{A\text{-}auth}}$.} $\mathsf{Sim}^{A^*_I}$ waits to receive \\$(GetAuthorization, sid, A^*_I, attr_I, A_i, attr_i)$ for $A_i \in A_R$, and checks whether $A^*_I$ and $A_i$ are registered  and their identities are in $CP_{A^*_I} \cap CP_{A_i}$. If so, it retrieves the counter and checks wether they have enough \ases \ budget. If so, it sends $(Allowed, sid,A^*_I, attr_I, A_i, attr_i)$ to $A^*_I$ and stores it in $L_{Auth}$.

\item {\bf Simulating each \Fases.} For each \ases \, $\mathsf{Sim}^{A^*_I}$ waits to receive $(AgentPolicy, sid, APolicy)$ for $A^*_I$ and $A_i$, and stores $APolicy$'s.  Upon receiving $(Task, sid, subtsk)$ from the corrupted $A^*_I$, it internally runs the task execution on $subtsk$, 
and returns the result and terminates when the task execution completes. Also, note that whenever \master \ aborts, $\mathsf{Sim}^{A^*_I}$ aborts and terminates.

\end{itemize}

{\bf Case (iii).}   When $A_i$ is corrupted, $\mathsf{Sim}^{A_i}$ simulates the ideal functionalities \Fases \ and $\mathcal{F}_{\mathsf{A\text{-}auth}}$, and the view of the corrupted $A_i$ (that is same as the simulated \Fases) and output of \master \ correctly (see below):

\begin{itemize}
\item {\bf Receiving the policy and task from $A^*_I$.} $\mathsf{Sim}^{A_i}$ waits to receive $(AgentPolicy, sid, APolicy)$ for $A^*_I$ and $A_i \in A_R$, and stores $APolicy$'s.   Upon receiving $(Task, sid, tsk)$ from dummy $A^*_I$, it internally runs the task execution on $tsk$ and gets the schedule and the $subtask$'s for each responding agent $A_i$.

\item {\bf Simulating each $\mathcal{F}_{\mathsf{A\text{-}auth}}$.} $\mathsf{Sim}^{A_i}$ simulates the $\mathcal{F}_{\mathsf{A\text{-}auth}}$ for each $A_i$; it checks whether $A^*_I$ and $A_i$ are registered  and their identities are in $CP_{A^*_I} \cap CP_{A_i}$. If so, it retrieves the counter and checks wether they have enough \ases \ budget. If the check passes, it sends $(Allowed, sid,A^*_I, attr_I, A_i, attr_i)$ to $A^*_I$ and stores it in $L_{Auth}$.

\item {\bf Simulating each \Fases.} $\mathsf{Sim}^{A_i}$ simulates \Fases \ for each $subtask$ with $A_i$, and returns the result and terminates when the task execution completes. Whenever $A_i$ aborts, $\mathsf{Sim}^{A_i}$ aborts and terminates.

\end{itemize}

The above simulation completely simulates the hybrid \cses \ protocol. We give the hybrid games  for the in-distinguishability of the hybrid world from the ideal world in the full version of the paper. 

Our analysis can be extended to the dynamic setting as well. We leave the complete proof for this setting to the full version of the paper.


\end{document}